\documentclass{article}
\PassOptionsToPackage{numbers, compress}{natbib}
\usepackage[final]{neurips_2023}

\usepackage{algorithm,algorithmic}
\usepackage{algorithmic}

\usepackage{graphicx}
\usepackage{multirow,hhline}% http://ctan.org/pkg/multirow
\usepackage{booktabs}% to use \midrule for drawing a horizontal line in align

\newcommand{\splitFthree}{ H }
\newcommand{\defSplitFthree}{ 
    \frac{1}{ (1-\mu_1+\varepsilon_2)(\mu_1-\varepsilon_2) h^2(\mu_1,\varepsilon_2) } 
    }

\newcommand{\defU}{
    \lceil \frac{ \ln\del{ T \kl(\mu_a + \varepsilon_1, \mu_1 - \varepsilon_2) \vee e^2} }
                { \kl(\mu_a+\varepsilon_1,\mu_1-\varepsilon_2) } 
    \rceil}
\newcommand{\BoundFOne}{\frac{1}{\kl(\mu_a + \varepsilon_1, \mu_1 - \varepsilon_2)}}
\newcommand{\BoundFTwo}{\frac{1}{\kl(\mu_a + \varepsilon_1, \mu_a)}}

\newcommand{\BoundFThree}{
    6 H \ln\del{ \del{\fr{T}{H} \wedge H} \vee e^2 }
    +
    \frac{4}{\kl(\mu_1-\varepsilon_2,\mu_1)}
}

\newcommand{\BoundFThreeOneFirst}{ 
    6 H \ln\del{ \fr{T}{H} \vee e^2 } +
    \frac{1}{\kl(\mu_1-\varepsilon_2,\mu_1)}
}

\newcommand{\BoundFThreeOneSecond}{
    6 H \ln (H \vee e^2) + 
    \fr{1}{\kl(\mu_1-\varepsilon_2,\mu_1)}
}

\newcommand{\BoundFThreeOne}{ 
    6 H \ln\del{ \del { \fr{T}{H} \wedge H } \vee e^2 } +
    \frac{1}{\kl(\mu_1-\varepsilon_2,\mu_1)}
}

\newcommand{\BoundFThreeTwo}{ 
    \fr{3}{\kl(\mu_1-\varepsilon_2,\mu_1)}
}

\def\dmu{{\dot\mu}}
\def\cd{\cdot}
\usepackage{def}
\usepackage{commath}
\usepackage{graphicx}
\usepackage{float}
\usepackage{comment}
\usepackage{xspace}

\graphicspath{ {./pics/} }

\newcommand{\kl}{\mathsf{kl}}
\newcommand{\Ber}{\mathrm{Bernoulli}}
\newcommand{\Reg}{\mathrm{Reg}}

\def\sig{\sigma}
\newcommand{\UCB}{\mathrm{UCB}}

\usepackage[normalem]{ulem}

\renewcommand{\paragraph}[1]{\noindent\textbf{#1}}

\usepackage[utf8]{inputenc} % allow utf-8 input
\usepackage[T1]{fontenc}    % use 8-bit T1 fonts
\usepackage{hyperref}       % hyperlinks
\usepackage{url}            % simple URL typesetting
\usepackage{booktabs}       % professional-quality tables
\usepackage{amsfonts}       % blackboard math symbols
\usepackage{nicefrac}       % compact symbols for 1/2, etc.
\usepackage{microtype}      % microtypography
\usepackage{xcolor}         % colors

\title{Kullback-Leibler Maillard Sampling for Multi-armed Bandits with Bounded Rewards}

\author{
  Hao Qin\\
  University of Arizona\\
  \texttt{hqin@arizona.edu} \\
  \And
  Kwang-Sung Jun \\
  University of Arizona\\
  \texttt{kjun@cs.arizona.edu} \\
  \And
  Chicheng Zhang \\
  University of Arizona\\
  \texttt{chichengz@cs.arizona.edu} \\
}

\allowdisplaybreaks

\begin{document}

\maketitle

\begin{abstract}
  We study $K$-armed bandit problems where the reward distributions of the arms are all supported on the $[0,1]$ interval. 
  Maillard sampling~\cite{maillard13apprentissage}, an attractive alternative to Thompson sampling, has recently been shown to achieve competitive regret guarantees in the sub-Gaussian reward setting~\cite{bian2022maillard} while maintaining closed-form action probabilities, which is useful for offline policy evaluation. In this work, we analyze the Kullback-Leibler Maillard Sampling (KL-MS) algorithm, a natural extension of Maillard sampling {and a special case of Minimum Empirical Divergence (MED)~\cite{honda2011asymptotically}} for achieving a KL-style finite-time gap-dependent regret bound. 
  We show that KL-MS enjoys the asymptotic optimality when the rewards are Bernoulli and has an {adaptive} worst-case regret bound of the form $O(\sqrt{\mu^*(1-\mu^*) K T \ln K} + K \ln T)$, where $\mu^*$ is the expected reward of the optimal arm, and $T$ is the time horizon length; {this is the first time such adaptivity is reported in the literature for an algorithm with asymptotic optimality guarantees.}
\end{abstract}

\setlength{\abovedisplayskip}{3pt}
\setlength{\belowdisplayskip}{4pt}
\setlength{\abovedisplayshortskip}{3pt}
\setlength{\belowdisplayshortskip}{4pt}
\textfloatsep=.5em
\vspace{-.5em}
\section{Introduction}
\vspace{-.5em}

The multi-armed bandit (abbrev. MAB) problem~\cite{thompson33onthelikelihood,lai85asymptotically,lattimore20bandit}, a stateless version of the reinforcement learning problem, has received much attention by the research community, due to its relevance in may applications such as online advertising, recommendation, and clinical trials.
In a multi-armed bandit problem, a  learning agent has access to a set of $K$ arms (also known as actions), where for each $i \in [K] := \cbr{1,\ldots,K}$, arm $i$ is associated with a distribution $\nu_i$ with mean $\mu_i$; at each time step $t$, the agent adaptively chooses an arm $I_t \in [K]$ by sampling from a probability distribution $p_t\in\Delta^{K-1}$ 
and receives reward $y_t \sim \nu_{I_t}$, based on the information the agent has so far. 
The goal of the agent is to minimize its pseudo-regret over $T$ time steps:
$
\Reg(T) = T \mu^* - \EE\sum_{t=1}^T y_t
$,
where $\mu^* = \max_i \mu_i$ is the optimal expected reward.

In this paper, we study the multi-armed bandit setting where reward distributions of all arms are supported on $[0,1]$.
\footnote{
  All of our results can be extended to distributions supported in $\sbr{L,U}$ for any known $L\le U$ by shifting and scaling the rewards to lie in $[0,1]$.
}
An important special case is Bernoulli bandits, where for each arm $i$, $\nu_i = \Ber(\mu_i)$ for some $\mu_i \in [0,1]$. It has practical relevance in settings such as computational advertising, where the reward feedback is oftentimes binary (click vs. not-click, buy vs. not-buy). 

%environment's
 
Broadly speaking, there are two popular families of provably regret-efficient algorithms for bounded-reward bandit problems: deterministic exploration algorithms (such as KL-UCB~\citep{garivier2011kl, cappe2013kullback,maillard11finite}) and randomized exploration algorithms (such as Thompson sampling (TS)~\citep{thompson33onthelikelihood}). 
Randomized exploration algorithms such as TS have been very popular, perhaps due to its excellent empirical performance and the ability to cope with delayed rewards better than deterministic counterparts~\citep{chapelle11anempirical}.  
In addition, the logged data collected from randomized exploration, of the form $(I_t, p_{t,I_t}, y_t)_{t=1}^T$, where $p_{t,I_t}$ is the probability with which arm $I_t$ was chosen, are useful for offline evaluation purposes by employing the inverse propensity weighting (IPW) estimator~\citep{horvitz52generalization} or the doubly robust estimator~\citep{robins95semiparametric}.
However, calculating the arm sampling probability distribution $p_t$ for Thompson sampling is nontrivial.
Specifically, there is no known closed-form \footnote{Suppose the arms' mean reward posterior distributions' PMFs and PDFs are $(p_1, \ldots, p_K)$ and $(F_1, \ldots, F_K)$ respectively; for example, they are Beta distributions with different parameters, i.e. $p_i(x) = x^{a_i-1} (1-x)^{b_i-1} I(x \in [0,1])$ for some $a_i, b_i$. To the best of our knowledge, the action probabilities have the following integral expression: 
$
\PP(I_t = a) 
= 
\int_{\RR} p_a(x) \prod_{i \neq a} F_i(x) \dif x
$ and cannot be further simplified.}, and generic numerical integration methods and Monte-Carlo approximations 
suffer from instability issues: the time complexity for obtaining a numerical precision of $\eps$ is $\Omega(\mathrm{poly}(1/\eps))$~\cite{novak2016some}.
This is too slow to be useful especially for web-scale deployments; e.g., Google AdWords receives $\sim$237M clicks per day.
Furthermore, the computed probability will be used after taking the inversion, which means that even moderate amount of errors are intolerable.
Indeed, Figure~\ref{fig:offline} shows that the offline evaluation with Thompson sampling as the behavioral policy will be largely biased and inaccurate due to the errors from the Monte Carlo approximation.

\begin{figure}
  \centering
  \begin{minipage}{0.4\textwidth}
  \includegraphics[width=\linewidth]{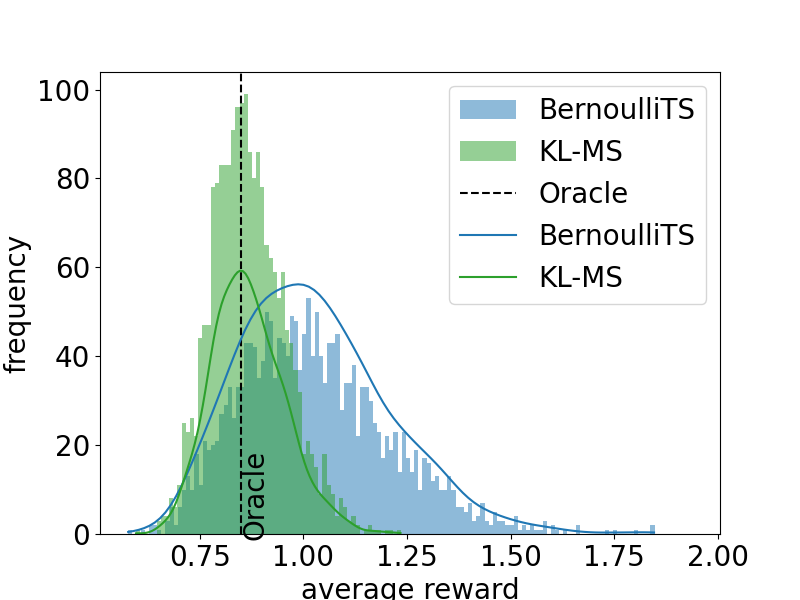}
  \end{minipage}
  \hfill
  \begin{minipage}{0.58 \textwidth}
   Figure 1: Histogram of the average rewards computed from the offline evaluation where
   the logged data is collected from Bernoulli TS and KL-MS (Algorithm~\ref{alg:KL-MS}) in a Bernoulli bandit environment with the mean reward (0.8, 0.9) with time horizon $T=10,000$.
For Bernoulli TS's log, we approximate the action probability by Monte Carlo Sampling with 1000 samples for each step. 
Here we estimate the expected reward of the uniform policy which has expected average reward of 0.85 (black dashed line). 
Across 2000 trials, the logged data of KL-MS induces an MSE of $0.00796$; however, for half of the trials, the IPW estimator induced by Bernoulli TS's log returns invalid values due to the action probability estimates being zero. 
Even excluding those invalid values, the Bernoulli TS's logged data induces an MSE of $0.02015$. See Appendix~\ref{sec:addl-experiments} for additional experiments.
  \end{minipage}
  \label{fig:offline}
\end{figure}

Recently, many studies have introduced alternative randomized algorithms that allow an efficient computation of $p_t$~\cite{honda2011asymptotically,maillard13apprentissage,cb17boltzmann,zimmert21tsallis}.
Of these, Maillard sampling (MS) ~\cite{maillard13apprentissage,bian2022maillard}, a Gaussian adaptation of the Minimum Empirical Divergence (MED) algorithm \cite{honda2011asymptotically} originally designed for finite-support reward distributions, provides a simple algorithm for the sub-Gaussian bandit setting that computes $p_t$ in a closed form: 
\begin{align}
  p_{t,a} \propto \exp\del{- N_{t-1,a} \fr{\hat\Delta_{t-1,a}^2}{2\sig^2}}
  \label{label:ms-rule}
\end{align}
where at time step $t$, $N_{t,a}$ is the number of pulling arm $a$. 
We define the estimator of $\mu_a$ as $\hmu_{t, a} := \frac{\sum_{s=1}^t \onec{I_t = a} y_t}{N_{t, a}}$ and the best performed mean value as $\hmu_{t, \max} := \max_{a\in[K]} \mu_{t, a}$.
$\hat\Delta_{t-1,a} = \max_{a'}\hmu_{t-1,a'} - \hmu_{t-1,a}$ is the empirical suboptimality gap of arm $a$, and $\sigma$ is the subgaussian parameter of the reward distribution of all arms. 
For sub-Gaussian reward distributions, MS enjoys the asymptotic optimality under the special case of Gaussian rewards
and a near-minimax optimality~\cite{bian2022maillard}, making it an attractive alternative to Thompson sampling.
{Also, MS satisfies the sub-UCB criterion (see Section~\ref{sec:prelims} for a precise definition) to help establish sharp finite-time instance-dependent regret guarantees.}
Can we adapt MS to the bounded reward setting and achieve the asymptotic, minimax optimality and sub-UCB criterion while computing the sampling probability in a closed-form?
In this paper, we make significant progress on this question.

\paragraph{Our contributions.} We focus on a Bernoulli adaptation of MS that we call Kullback-Leibler Maillard Sampling (abbrev. KL-MS) and perform a finite-time analysis of it in the bounded-reward bandit problem. 
KL-MS uses a sampling probability similar to MS but tailored to the $[0,1]$-bounded reward setting:

\[
    p_{t,a} \propto \exp\del{ -N_{t-1,a} \kl\del{\hmu_{t-1, a}, \hmu_{t-1, \max}} },
\]
where $\kl(\mu, \mu') := \mu \ln\frac{\mu}{\mu'} + (1-\mu) \ln\frac{1-\mu}{1-\mu'}$ is the binary Kullback-Leibler (KL) divergence. {We can also view KL-MS as an instantiation of MED~\cite{honda2011asymptotically} for Bernoulli rewards; See Section~\ref{sec:related} for a detailed comparison.}

KL-MS performs an efficient exploration for bounded rewards since one can use $\kl(a,b) \ge 2(a-b)^2$ to verify that the probability being assigned to each empirical non-best arm by KL-MS is never larger than that of MS with $\sigma^2 = 1/4$, the best sub-Gaussian parameter for the bounded rewards in $[0,1]$.
We show that KL-MS achieves a sharp finite-time regret guarantee (Theorem~\ref{thm:expected-regret-total}) that can be simultaneously converted to: 
\begin{itemize}\itemsep0em
    \item an asymptotic regret upper bound (Theorem~\ref{thm:asymptotic-optimality}), which is asymptotically optimal when specialized to the Bernoulli bandit setting;
    \item a $\sqrt{T}$-style regret guarantee of $O(\sqrt{\mu^*(1-\mu^*) K T \ln K} + K\ln(T))$ (Theorem~\ref{thm:minimax-regret-bound}) where $\mu^*$ is the mean reward of the best arm.
    This bound has two salient features. 
    First, in the worst case, it is at most a $\sqrt{\ln K}$ factor suboptimal than the minimax optimal regret of $\Theta(\sqrt{KT})$~\cite{audibert09minimax,auer03nonstochastic}. 
    Second, its $\tilde{O}(\sqrt{\mu^*(1-\mu^*)}$ coefficient adapts to the variance of the optimal arm reward; this is the first time such adaptivity is reported in the literature for an algorithm with asymptotical optimality guarantees. \footnote{As side results, we show in Appendix~\ref{sec:refined} that with some modifications of the analysis, 
existing algorithms~\cite{audibert09exploration,cappe2013kullback,menard17minimax} also achieve regret of the form $\tilde{O}(\sqrt{\mu^*(1-\mu^*) \mathrm{poly}(K) T})$ for $[0,1]$-bounded reward MABs.}
    {\item a sub-UCB regret guarantee, which many existing minimax optimal algorithms~\citep{menard17minimax, garivier2022kl} have not been proven to satisfy. 
    }

\end{itemize}

\begin{table}
\centering
\begin{tabular}{ccccl}
\hline 
\text{Algorithm\&} & \multicolumn{2}{c}{\text{Finite-Time Regret}} & \text {Closed-form} & Reference \\
\text{Analysis}    & \text{Minimax Ratio}         & \text{Sub-UCB} & \text {Probability}    &   \\
\hline 
\text { TS }     &$\sqrt{\ln K}$& \text { yes } & \text { no }  & See the caption \\
\text { ExpTS }  &$\sqrt{\ln K}$& \text { yes } & \text { no }  & \citet{jinfinite}    \\
$\text{ ExpTS}^+$& $1$          & $-^{\star\star}$  & \text { no }  & \citet{jinfinite}    \\
$\text{ kl-UCB }$&$\sqrt{\ln T}$& \text { yes } & \text { N/A } & \citet{cappe2013kullback}\\
$\text{kl-UCB++}$&    $ 1 $     & $-^{\star\star}$  & \text { N/A } & \citet{menard17minimax}\\
$\text{kl-UCB-switch}$& $1$     & $-^{\star\star}$  & \text { N/A } & \citet{garivier2022kl} \\
\text { MED }    &      $-$     &       $-$     & $\text {no}^\star$  & \citet{honda2011asymptotically}\\
\text { DMED }   &      $-$     &       $-$     & \text {N/A}   & \citet{honda2012finite} \\
\text { IMED }   &      $-$     &       $-$     & \text {N/A}   & \citet{honda15non} \\
\text { KL-MS}   &$\sqrt{\ln K}$& \text { yes } & \text { yes } & \text{this paper}\\
\hline
\end{tabular}
\vspace{.3em}
\caption{Comparison of regret bounds for bounded reward distributions; for space constraints we only include those that achieves the asymptotic optimality for the special case of Bernoulli distributions (this excludes, e.g., Maillard Sampling~\cite{maillard13apprentissage,bian2022maillard}, Tsallis-INF~\cite{zimmert21tsallis} and UCB-V~\cite{audibert2009exploration}).
\lq$-$\rq indicates that the corresponding analysis is not reported. 
\lq N/A\rq indicates that the algorithm does have closed-form, but it is deterministic.
\lq$\star$\rq indicates that its computational complexity for calculating the action probability is $\ln(1/\text{precision})$.
\lq$\star\star$\rq indicates that we conjecture that the algorithm is not sub-UCB. 
The results on TS are reported by \citet{agrawal2013further,agrawal2017near,korda2013thompson}.
}
\end{table}

We also conduct experiments that show that thanks to its  closed-form action probabilities, KL-MS generates much more reliable logged data than Bernoulli TS with Monte Carlo estimation of action probabilities; this is reflected in their offline evaluation performance 
using the IPW estimator; see Figure~\ref{fig:offline} and Appendix~\ref{sec:addl-experiments} for more details.

\section{Preliminaries}
\label{sec:prelims}

Let $N_{t,a}$ be the number of times arm $a$ has been pulled until time step $t$ (inclusively). Denote the suboptimality gap of arm $a$ by $\Delta_a := \mu^* - \mu_a$, where $\mu^* = \max_{i \in [K]} \mu_i$ is the optimal expected reward.
Denote the empirical suboptimality gap of arm $a$ by $\hat{\Delta}_{t,a}:= \hat{\mu}_{t,\max} - \hat{\mu}_{t,a}$;
here, $\hat{\mu}_{t,a}$ is the empirical estimation to $\mu_{a}$ up to time step $t$, i.e., $\hat{\mu}_{t,a}:= \frac{1}{N_{t,a}}\sum_{s=1}^{t} y_s \onec{I_s = a}$, and $\hat{\mu}_{t,\max} = \max_{a\in[K]} \hmu_{t,a}$ is the best empirical reward at time step $t$.
For arm $a$, define $\tau_a(s) := \min\cbr{t \geq 1: N_{t,a}=s}$ at the time step when arm $a$ is pulled for the $s$-th time, which is a stopping time; we also use $\hat{\mu}_{(s),a} := \hat{\mu}_{\tau_a(s),a}$ to denote empirical mean of the first $s$ reward values received from pulling arm $a$.

We define the Kullback-Leibler divergence between two distributions $\nu$ and $\rho$ as $\KL(\nu, \rho) = \EE_{X \sim \nu}\sbr{ \ln \frac{\dif \nu}{\dif \rho} (X) }$ 
if $\nu$ is absolutely continuous w.r.t. $\rho$, and $= +\infty$ otherwise.
Recall that we define the binary Kullback-Leibler divergence between two numbers $\mu,\mu'$ in $[0,1]$ as $\kl(\mu, \mu') := \mu \ln\frac{\mu}{\mu'} + (1-\mu) \ln\frac{1-\mu}{1-\mu'}$, which is also the KL divergence between two Bernoulli distributions with mean parameters $\mu$ and $\mu'$ respectively. We define $\dmu = \mu(1-\mu)$, which is the variance of $\Ber(\mu)$ but otherwise an upper bound on any distribution supported on $[0,1]$ with mean $\mu$; see Lemma \ref{lemma:control-variance} for a formal justification.

In the regret analysis, we will oftentimes use the following notation for comparison up to constant factors:
define $f \lesssim g$ (resp. $f \gtrsim g$) to denote that $f \leq C g$ (resp. $f \geq C g$) for some numerical constant $C > 0$. 
We define $a \vee b$ and $a \wedge b$ as $\max(a,b)$ and $\min(a,b)$, respectively.
For an event $E$, we use $E^c$ to denote its complement.

Below, we define some useful criteria for measuring the performance of bandit algorithms, specialized to the $[0,1]$ bounded reward setting.

\paragraph{Asymptotic optimality in the Bernoulli reward setting} An algorithm is asymptotically optimal in the Bernoulli reward setting~\cite{lai85asymptotically,burnetas96optimal} if for any Bernoulli bandit instance $(\nu_a = \Ber(\mu_a))_{a \in [K]}$, 
$
            \limsup_{T \to \infty}
            \frac{\Reg(T)}{\ln{T}} =
            \sum_{a:\Delta_a > 0} 
            \frac{\Delta_a}{\kl(\mu_a, \mu^*)}.
$
    
        \paragraph{Minimax ratio}
        The minimax optimal regret of the $[0,1]$ bounded reward bandit problem is $\Theta\del{\sqrt{KT}}$ \cite{audibert09minimax,auer03nonstochastic}. Given a $K$-armed bandit problem with time horizon $T$, an algorithm has a minimax ratio of $f(T,K)$ if its has a worst-case regret bound of $O(\sqrt{KT}f(T,K))$.

        \paragraph{Sub-UCB} Sub-UCB is originally defined in the context of sub-Gaussian bandits~\citep{lattimore20bandit}:
        given a bandit problem with $K$ arms whose reward distributions are all sub-Gaussian, 
        an algorithm is said to be sub-UCB
        if there exists some positive constants $C_1$ and $C_2$, such that for all $\sigma^2$-sub-Gaussian bandit instances, 
            $
                \Reg(T) \leq C_1 \sum_{a: \Delta_a >0} \Delta_a + C_2 \sum_{a:\Delta_a>0} \frac{\sigma^2}{\Delta_a}\ln{T}
            $.
       Specialized to our setting, as any distribution supported on $[0,1]$ is also $\frac14$-sub-Gaussian, and all suboptimal arm gaps $\Delta_a \in (0,1]$ are such that $\Delta_a < \frac{1}{\Delta_a}$, the above sub-UCB criterion simplifies to: there exists some positive constant $C$, such that for all $[0,1]$-bounded reward bandit instances, 
        $
        \Reg(T) \leq C \sum_{a:\Delta_a>0} \frac{\ln{T}}{\Delta_a}
        $.

\vspace{-.5em}
\section{Related Work}
\label{sec:related}
\vspace{-.5em}

\paragraph{Bandits with bounded rewards.} 
Early works of~\citet{lai85asymptotically,burnetas96optimal} show that in the bounded reward setting, for any consistent stochastic bandit algorithm, the regret is lower bounded by $(1 + o(1)) \sum_{a: \Delta_a > 0}\frac{\Delta_a \ln T}{\KL_{\inf}(\nu_a, \mu^*)}$ and $\KL_{\inf}(\nu_a, \mu^*)$ is defined as
\begin{align} \label{eqn:KL-inf}
    \KL_{\inf}(\nu, \mu^*) := \inf\cbr{\KL(\nu, \rho): \EE_{X\sim \rho} \sbr{X} > \mu^*, \mathrm{supp}(\rho) \subset [0,1]},
\end{align}
where the random variable follows a distribution $\rho$ bounded in $\sbr{0, 1}$.
Therefore, any algorithm whose regret upper bound matches the lower bound is said to achieve asymptotic optimality.
\citet{cappe2013kullback} propose the KL-UCB algorithm and provide a finite time regret analysis, which is further refined by~\citet[][Chapter 10]{lattimore20bandit}.
Another line of work establishes asymptotic and finite-time regret guarantees for
Thompson sampling algorithms and its variants~\cite{agrawal12analysis,agrawal2017near,kaufmann12thompson,jinfinite}, which, when specialized to the Bernoulli bandit setting, can be combined with Beta priors for the Bernoulli parameters to design efficient algorithms.

A number of studies even go beyond the Bernoulli-KL-type regret bound and adapt to the variance of each arm in the bounded reward setting.
UCB-V~\cite{audibert09exploration} achieves a regret bound that adapts to the variance.
Efficient-UCBV~\cite{mukherjee18efficient} achieves a variance-adaptive regret bound and also an optimal minimax regret bound $O(\sqrt{KT})$, but it is not sub-UCB.
\citet{honda2011asymptotically} propose the MED algorithm that is asymptotically optimal for bounded rewards, but it only works for rewards that with finite supports. 
\citet{honda15non} propose the Indexed MED (IMED) algorithm that can handle a more challenging case where the reward distributions are supported in $(-\infty, 1]$.

As with worst-case regret bounds, 
first, it is well-known that for Bernoulli bandits as well as bandits with $[0,1]$ bounded rewards, the minimax optimal regrets are of order $\Theta(\sqrt{K T})$~\cite{auer03nonstochastic,audibert09minimax}. 
Of the algorithms that enjoy asymptotic optimality under the Bernoulli reward setting described above, KL-UCB~\cite{cappe2013kullback} has a worst-case regret bound of $O(\sqrt{K T \ln T})$, which is refined by the KL-UCB++ algorithm~\cite{menard17minimax} that has a worst-case regret bound of $O(\sqrt{K T})$.
We also show in Appendix~\ref{sec:kl-ucb-refined} and \ref{sec:kl-ucb++-refined} that with some modifications of existing analysis, KL-UCB and KL-UCB++ enjoy a regret bound of $O(\sqrt{ \mu^*(1-\mu^*) K T \ln T})$ and $O(\sqrt{ \mu^*(1-\mu^*) K^3 T \ln T})$ respectively. 
Although the regret is worse in the order of $K$, it adapts to $\mu^*$ and will have a better regret when $\mu^*$ is small (say, $\mu^* \leq 1/K^2$).
{KL-UCB++\cite{menard2017minimax} and KL-UCB-Switch\cite{garivier2022kl} achieves $O(\sqrt{KT})$ regret in the finite-time regime and asymptotic optimality, while the sub-UCB criterion has not been satisfied.
However, \citet[\S3]{lattimore2018refining} shows that MOSS~\cite{audibert2009minimax} suffers a sub-optimal regret worse than UCB-like algorithms because of not satisfying sub-UCB criteria, and we suspect that KL-UCB-switch experience the same issue as MOSS.}
For Thompson Sampling style algorithms,~\citet{agrawal2013further} shows that the original Thompson Sampling algorithm has a worst-case regret of $O(\sqrt{K T \ln K})$, and the ExpTS+ algorithm~\cite{jinfinite} has a worst-case regret of $O(\sqrt{KT})$.

\paragraph{Randomized exploration for bandits.} Many randomized exploration methods have been proposed for multi-armed bandits. Perhaps the most well-known is Thompson sampling~\cite{thompson33onthelikelihood}, which is shown to achieve Bayesian and frequentist-style regret bounds in a broad range of settings~\cite{russo14learning,agrawal12analysis,kaufmann12thompson,korda2013thompson,jin2021mots,jinfinite}. 
A drawback of Thompson sampling, as mentioned above, is that the action probabilities cannot be obtained easily and robustly. To cope with this, a line of works design randomized exploration algorithms with action probabilities in closed forms. 
For sub-Gaussian bandits, \citet{cb17boltzmann} propose a variant of the Boltzmann exploration rule (that is, the action probabilities are proportional to exponential to empirical rewards, scaled by some positive numbers), and show that it has $O\del{\frac{K \ln^2 T}{\Delta}}$ instance-dependent and $O\del{\sqrt{KT} \ln K}$ worst-case regret bounds respectively, where $\Delta = \min_{a: \Delta_a > 0} \Delta_a$ is the minimum suboptimalty gap. 
Maillard sampling (MS; Eq.~\eqref{label:ms-rule}) is an algorithm proposed by the thesis of \citet{maillard13apprentissage} where the author reports  that MS achieves the asymptotic optimality and has a finite-time regret  of order 
$\sum_{a: \Delta_a > 0} \rbr{ \frac{\ln T}{\Delta_a} + \frac{1}{\Delta_a^3} }$ from which a worst-case regret bound of $O(\sqrt K T^{3/4})$ can be derived. 
MED~\cite{honda2011asymptotically}, albeit achieves asymptotic optimality for a broad family of bandits with finitely supported reward distributions, also has a high finite-time regret bound of at least $\sum_{a: \Delta_a > 0} \rbr{ \frac{\ln T}{\Delta_a} + \frac{1}{\Delta_a^{2|\mathrm{supp}(\nu_1)|-1}} }$.
\footnote{A close examination of~\cite{honda2011asymptotically}'s Lemma 9 (specifically, equation (20)) shows that for each suboptimal arm $a$, the authors bound $\EE\sbr{N_{T,a}}$ by a term at least $\sum_{t=1}^T K\del{t+1}^{|\mathrm{supp}(\nu_1)|}\cdot$ $\exp\del{-t C(\mu_1, \mu_1-\varepsilon)}$, where $C(\mu,\mu'):= \frac{(\mu-\mu')^2}{2\mu'(1+\mu)}$ and $\varepsilon \leq \Delta_a$; this is $\Omega\rbr{ \frac{1}{\Delta_a^{2|\mathrm{supp}(\nu_1)|}} }$ when $\mu_1$ is bounded away from $0$ and $1$.}
Recently,~\citet{bian2022maillard} report a refined analysis of~\citet{maillard13apprentissage}'s sampling rule, showing that it 
has a finite time regret of order $\sum_{a: \Delta_a > 0} \frac{\ln(T\Delta_a^2)}{\Delta_a} + O\del{\sum_{a: \Delta_a > 0} \frac{1}{\Delta_a} \ln(\frac{1}{\Delta_a}) }$, 
and additionally enjoys a $O\del{ \sqrt{KT \ln T} }$ worst-case regret, and by inflating the exploration slightly (called MS$^+$), the bound can be improved and enjoy the minimax regret of $O\del{ \sqrt{KT \ln K} }$, which matches the best-known regret bound among those that satisfy sub-UCB criterion, except for AdaUCB.
In fact, it is easy to adapt our proof technique in this paper to show that MS, without any further modification, achieves a $O\del{ \sqrt{KT \ln K} }$ worst-case regret.

Randomized exploration has also been studied from a nonstochastic bandit perspective~\cite{auer03nonstochastic,audibert09minimax}, where randomization serves both as a tool for exploration and a way to hedge bets against the nonstationarity of the arm rewards. Many recent efforts focus on designing randomized exploration bandit algorithms that achieve ``best of both worlds'' adaptive guarantees, i.e., achieving logarithmic regret for stochastic environments while achieving $\sqrt{T}$ regret for adversarial environments~\cite[e.g.][]{zimmert21tsallis,wei18more}.

{\paragraph{Binarization trick.}
It is a folklore result that bandits with $[0,1]$ bounded reward distributions can be reduced to Bernoulli bandits via a simple binarization trick: at each time step $t$, the learner sees reward $r_t \in [0,1]$, draws $\tilde{r}_t \sim \Ber(r_t)$ and feeds it to a Bernoulli bandit algorithm. However, this reduction does not result in asymptotic optimality for the general bounded reward setting, where the asymptotic optimal regret is of the form 
$(1 + o(1)) \sum_{a: \Delta_a > 0}\frac{\Delta_a \ln T}{\KL_{\inf}(\nu_a, \mu^*)}$ with $\KL_{\inf}(\nu_a, \mu^*)$ defined in the Eq~\eqref{eqn:KL-inf}.
If we combine the binarization trick and the MED algorithm in the bounded reward setting, the size of the support set is viewed as $2$, the finite-time regret bound is at best as $O(K^{1/4}T^{3/4})$ (ignoring  logarithmic factors), which is much higher than $O(\sqrt{KT})$.}

\paragraph{Bandit algorithms with worst-case regrets that depend on the optimal reward.}
Recent linear logistic bandit works have shown worst-case regret bounds that depend on the variance of the best arm~\cite{mason2022experimental,abeille2021instance}.
When the arms are standard basis vectors, logistic bandits are equivalent to Bernoulli bandits, and the bounds of \citet{abeille2021instance} become $\tilde O\del{K\sqrt{\dmu^* T} + \fr{K^2}{\dmu_{\min}} \wedge (K^2 + A) } $ where $\dmu_{\min} = \min_{i\in[K]}\dmu_i$ and $A$ is an instance dependent quantity that can be as large as $T$.
This bound, compared to ours, has an extra factor of $\sqrt{K}$ in the leading term and the lower order term has an extra factor of $K$.
Even worse, it has the term $\dmu_{\min}^{-1}$ in the lower order term, which can be arbitrarily large. 
The bound in \citet{mason2022experimental} becomes $\tilde O\del{\sqrt{\dmu^* KT }+ \dmu_{\min}^{-1} K^2}$, which matches our bound in the leading term up to logarithmic factors yet still have extra factors of $K$ and $\dmu_{\min}^{-1}$ in the lower order term.

\vspace{-.5em}
\section{Main Result}
\label{sec:main}
\vspace{-.5em}

\paragraph{The KL Maillard Sampling Algorithm.} We propose an algorithm called KL Maillard sampling (KL-MS) for bounded reward distributions (Algorithm~\ref{alg:KL-MS}). 
For the first $K$ times steps, the algorithm pulls each arm once (steps~\ref{step:t-leq-k-1} to~\ref{step:t-leq-k-2}); this ensures that starting from time step $K+1$, the estimates of the reward distribution of all arms are well-defined. 
From time step $t = K+1$ on, the learner computes the empirical mean $\hat{\mu}_{t-1,a}$ of all arms $a$. 
For each arm $a$, the learner computes the binary KL divergence between $\hat{\mu}_{t-1,a}$ and $\hat{\mu}_{t-1, \max}$, $\kl(\hat{\mu}_{t-1,a},\hat{\mu}_{t-1,\max})$, as a measure of empirical suboptimality of that arm. The sampling probability of arm $a$, denoted by $p_{t,a}$, is proportional to the exponential of negative product between $N_{t-1,a}$ and $\kl(\hat{\mu}_{t-1,a},\hat{\mu}_{t-1,\max})$ (Eq.~\eqref{eqn:kl-ms-rule} of step~\ref{step:sampling-prob}). This policy naturally trades off between exploration and exploitation: arm $a$ is sampled with higher probability, if either it has not been pulled many times ($N_{t-1,a}$ is small) or it appears to be close to optimal empirically ($\kl(\hat{\mu}_{t-1,a},\hat{\mu}_{t-1,\max}))$ is small). 
The algorithm samples an arm $I_t$ from $p_t$, and observe a reward $y_t$ of the arm chosen. 

We remark that if the reward distributions $\nu_i$'s are Bernoulli, KL-MS is equivalent to the MED algorithm~\citep{honda2011asymptotically} since in this case, all reward distributions have a binary support of $\{0,1\}$.
However, 
KL-MS is different from MED in general: MED computes the empirical distributions of arm rewards $\hat{F}_{t-1,a}$, and chooses action according to probabilities $p_{t,a} \propto \exp(-N_{t-1,a} D_{t-1,a} )$; here, $D_{t-1,a} := \KL(\hat{F}_{t-1,a}, \hat{\mu}_{t-1,\max})$ (recall its definition in Section~\ref{sec:related}) is the ``minimum empirical divergence'' between arm $a$ and the highest empirical mean reward, which is different from the binary KL divergence of the mean rewards used in KL-MS.

    \begin{algorithm}[t]
    \begin{algorithmic}[1]
    \STATE \textbf{Input:} $K\geq 2$
    \FOR{$t=1,2,\cdots,T$}
        \IF{$t\leq K$} 
        \label{step:t-leq-k-1}
            \STATE Pull the arm $I_t=t$ and observe reward $y_t \sim \nu_i$.
            \label{step:t-leq-k-2}
        \ELSE 
            \STATE For every $a \in [K]$, compute 
            \begin{align}
                p_{t,a} = \frac{1}{M_t}\exp\del{-N_{t-1,a} \cdot \kl(\hat{\mu}_{t-1,a},\hat{\mu}_{t-1,\max})}
                \label{eqn:kl-ms-rule}
            \end{align}
            where $M_t = \sum_{a=1}^K \exp(-N_{t-1,a} \kl(\hat{\mu}_{t-1,a},\hat{\mu}_{t-1,\max}))$ 
            is the normalizer. 
            \label{step:sampling-prob}
            \STATE Pull the arm $I_t \sim p_t$.
            \STATE Observe reward $y_t \sim \nu_{I_t}$.
        \ENDIF
    \ENDFOR
    \end{algorithmic}
    \caption{KL Maillard Sampling (KL-MS)}
    \label{alg:KL-MS}
    \end{algorithm}

    \subsection{Main Regret Theorem}

    Our main result of this paper is the following theorem on the regret guarantee of KL-MS (Algorithm~\ref{alg:KL-MS}). Without loss of generality, throughout the rest of the paper, we assume $\mu_1 \ge \mu_2 \ge \cdots \ge \mu_K$.

    \begin{theorem} \label{thm:expected-regret-total}
    For any $K$-arm bandit problem with reward distribution supported on $[0,1]$, 
    KL-MS has regret bounded as follows. For any $\Delta \geq 0$ and $c \in (0, \frac14]$: 
    \begin{align}
    \Reg(T)
    &\leq
    T\Delta 
    +
    \sum_{a: \Delta_a > \Delta}  
    \frac{\Delta_a \ln(T \kl(\mu_a + c \Delta_a, \mu_1 - c \Delta_a) \vee e^2 )}{\kl(\mu_a + c \Delta_a, \mu_1 - c \Delta_a)}
    \nonumber \\
    &~~ + 
    392 \sum_{a: \Delta_a > \Delta} \rbr{\frac{\dmu_1 + \Delta_a}{c^2 \Delta_a}} \ln\rbr{ \rbr{  \frac{\dmu_1 + \Delta_a}{c^2 \Delta_a^2} \wedge \frac{c^2 T \Delta_a^2}{\dmu_1 + \Delta_a }} \vee e^2 } 
    \label{eqn:kl-ms-main-regret}
    \end{align}
    \end{theorem}

    The regret bound of Theorem~\ref{thm:expected-regret-total} is composed of three terms. The first term is
    $T \Delta$, which controls the contribution of regret from all $\Delta$-near-optimal arms. 
    The second term is 
    asymptotically $(1+o(1))\sum_{a: \Delta_a > 0}
    \frac{\Delta_a}{\kl(\mu_a, \mu_1)} \ln(T)$ with an appropriate choice of $c$, which is a term that grows in $T$ in a logarithmic rate. The third term is simultaneously upper bounded by 
    two expressions.
    One is $\sum_{a: \Delta_a > 0} \rbr{\frac{\dmu_1 + \Delta_a}{c^2 \Delta_a}} \ln\rbr{ \frac{c^2 T \Delta_a^2}{\dmu_1 + \Delta_a} \vee e^2 }$, which is of order $\ln T$ and helps establish a tight worst-case regret bound (Theorem~\ref{thm:minimax-regret-bound});
    the other is 
    $\sum_{a: \Delta_a > 0} \rbr{\frac{\dmu_1 + \Delta_a}{c^2 \Delta_a}} \ln\rbr{ \rbr{  \frac{\dmu_1 + \Delta_a}{c^2 \Delta_a^2} } \vee e^2 }$, which does not grow in $T$ and helps establish a tight asymptotic upper bound on the regret (Theorem~\ref{thm:asymptotic-optimality}).
     
    To the best of our knowledge, existing regret analysis on Bernoulli bandits or bandits with bounded support have regret bounds of the form 
    \[
    \Reg(T)
    \leq 
    T\Delta 
    +
    \sum_{a: \Delta_a > \Delta}  
    \frac{\Delta_a \ln(T)}{\kl(\mu_a + c \Delta_a, \mu_1 - c \Delta_a)}
    + 
    O\del[4]{ \sum_{a: \Delta_a > \Delta} \frac{1}{c^2 \Delta_a} },
    \]
    for some $c >0$, where the third term is much larger than its counterpart given by Theorem~\ref{thm:expected-regret-total} when $\Delta_a$ and $\dmu_1$ are small. 
    As we will see shortly, as a consequence of its tighter bounds, our regret theorem yields a superior worst-case regret guarantee over previous works.

    \begin{theorem}[Sub-UCB]
    \label{thm:sub-UCB}
    KL-MS's regret is bounded by $\Reg(T) \lesssim \sum_{a: \Delta_a > 0} \frac{\ln T}{\Delta_a}$. Therefore, KL-MS is sub-UCB.   
    \end{theorem}
    
    Sub-UCB criterion is important for measuring a bandit algorithm's finite-time instance-dependent performance. 
    Indeed,  \citet[\S3]{lattimore2018refining} 
    points out that MOSS \cite{audibert2009minimax} does not satisfy sub-UCB and that it leads to a strictly suboptimal regret in a specific instance compared to the standard UCB algorithm \cite{auer02using}. 
    A close inspection of the finite-time regret bounds of existing asymptotically optimal and minimax optimal algorithms for the $[0,1]$-reward setting, such as KL-UCB++~\cite{menard17minimax} and KL-UCB-switch~\cite{garivier2022kl}, reveals that
    they are not sub-UCB.
    Thus, we speculate that they would also have a suboptimal performance in the aforementioned instance.

    In light of Theorem~\ref{thm:expected-regret-total}, our first corollary is that KL Maillard sampling achieves the following adaptive worst-case regret guarantee. 
   
    \begin{theorem}[Adaptive worst-case regret]\label{thm:minimax-regret-bound}
        For any $K$-arm bandit problem with reward distribution supported on $[0,1]$, 
        KL-MS has regret bounded as:
        $
            \Reg(T) \lesssim  
            \sqrt{\dmu_1 KT\ln{K}} + K \ln T
        $.
    \end{theorem}

    An immediate corollary is that KL Maillard sampling has a regret of order $O(\sqrt{KT \ln K})$, which is a factor of $O(\sqrt{\ln K})$ within the minimax optimal regret $\Theta(\sqrt{KT})$~\cite{menard17minimax,audibert09minimax}.
    This also matches the worst-case regret bound $O(\sqrt{VKT\ln(K)})$ of~\citet{jinfinite} where $V=\fr14$ is the worst-case variance for Bernoulli bandits using a Thompson sampling-style algorithm.
    Another main feature of this regret bound is its adaptivity to $\dmu_1$, the variance of the reward of the optimal arm for the Bernoulli bandit setting, or its upper bound in the general bounded reward setting (see Lemma~\ref{lemma:control-variance}).
    Specifically, 
    if $\mu_1$ is close to 0 or 1, $\dmu_1$ is very small, which results in the regret being much smaller than $O(\sqrt{K T \ln K})$. 
    
    Note that UCB-V~\cite{audibert09exploration} and KL-UCB/KL-UCB++, while not reported, enjoy a worst-case regret bound of $O(\sqrt{\dmu_1 KT\ln T})$, which is worse than our bound in its logarithmic factor; see Appendix~\ref{sec:ucbv} and~\ref{sec:kl-ucb-refined} for the proofs.
    Among these, UCB-V does not achieve the asymptotic optimality for the Bernoulli case.
    While logistic linear bandits~\cite{abeille2021instance,mason2022experimental} can be applied to Bernoulli $K$-armed bandits and achieve similar worst-case regret bounds involving $\dmu_1$, their lower order term can be much worse as discussed in Section~\ref{sec:related}.

    Our second corollary is that KL Maillard sampling achieves a tight asymptotic regret guarantee for the special case of Bernoulli rewards: 

    \begin{theorem}(Asymptotic Optimality) \label{thm:asymptotic-optimality}
        For any $K$-arm bandit problem with reward distribution supported on $[0,1]$, 
        KL-MS satisfies the following asymptotic regret upper bound:
        \begin{align}
            \limsup_{T \rightarrow \infty} \fr{ \Reg(T) }{\ln(T)}
            =
            \sum_{a\in[K]:\Delta_a>0}
            \frac{\Delta_a}{\kl(\mu_a, \mu_1)}
            \label{eqn:kl-ms-asympt}
        \end{align}
    \end{theorem}

    Specialized to the Bernoulli bandit setting, 
    in light of the asymptotic lower bounds~\cite{lai85asymptotically,burnetas96optimal},
    the above asymptotic regret upper bound implies that KL-MS is asymptotically optimal.
    
    While the regret guarantee of KL-MS is not asymptotically optimal for the general $[0,1]$ bounded reward setting, it nevertheless is a better regret guarantee than naively viewing this problem as a sub-Gaussian bandit problem and applying sub-Gaussian bandit algorithms on it. To see this, note that any reward distribution supported on $[0,1]$ is $\frac14$-sub-Gaussian; therefore, standard sub-Gaussian bandit algorithms will yield an asymptotic regret $(1+o(1))\sum_{a\in[K]:\Delta_a>0} \frac{\ln T}{2\Delta_a}$. This is always no better than the asymptotic regret provided by Eq.~\eqref{eqn:kl-ms-asympt}, in view of Pinsker's inequality that $\kl(\mu_a, \mu_1) \geq 2\Delta_a^2$.

\vspace{-.5em}
\section{Proof Sketch of Theorem~\ref{thm:expected-regret-total}}
\label{sec:proof-sketch}
\vspace{-.5em}

We provide an outline of our proof of Theorem~\ref{thm:expected-regret-total}, with full proof details deferred to Appendix~\ref{sec:full-proof}.
Our approach is akin to the recent analysis of the sub-Gaussian Maillard Sampling algorithm in \citet{bian2022maillard} with several refinements tailored to the bounded reward setting and achieving $\sqrt{\ln K}$ minimax ratio. 
First, for any time horizon length $T$, $\Reg(T)$ can be bounded by:
\begin{align} 
    \Reg(T) 
    =
    \sum_{a\in[K]:\Delta_a > 0} \Delta_a \EE\sbr{N_{T,a}}
    \leq
    \Delta T + \sum_{a\in[K]:\Delta_a > \Delta} \Delta_a \EE\sbr{N_{T,a}}, 
    \label{eqn:reg-decomp}
\end{align}
i.e., the total regret can be decomposed to a $T\Delta$ term and 
the sum of regret $\Delta_a \EE\sbr{N_{T,a}}$ from pulling $\Delta$-suboptimal arms $a$. Therefore, in subsequent analysis, we focus on bounding $\EE\sbr{N_{T,a}}$. To this end, we show the following lemma.

\begin{lemma} \label{lemma:expected-sub-optimal-arm-pull-main}
    For any suboptimal arm $a$, let $\varepsilon_1, \varepsilon_2 > 0$ be such that $\varepsilon_1 + \varepsilon_2 < \Delta_a$. 
    Then its expected number of pulls is bounded as:
    \begin{align}
    \EE\sbr{ N_{T,a} }
    \leq & 1 + \frac{ \ln\del{ T \kl(\mu_a + \varepsilon_1, \mu_1 - \varepsilon_2) \vee e^2} }
                { \kl(\mu_a+\varepsilon_1,\mu_1-\varepsilon_2) } 
    + \BoundFOne + \BoundFTwo \nonumber
    \\
    & + \BoundFThree,
    \label{eqn:arm-pull-bound-main}
    \end{align}
    where $\splitFthree := \defSplitFthree \lesssim \frac{2\dmu_1 + \varepsilon_2}{\varepsilon_2^2}$ and $h(\mu_1,\varepsilon_2):= \ln \del{\fr{(1-\mu_1+\varepsilon_2)\mu_1}{(1-\mu_1)(\mu_1-\varepsilon_2)}}$.
\end{lemma}

Theorem~\ref{thm:expected-regret-total} follows immediately from Lemma~\ref{lemma:expected-sub-optimal-arm-pull-main}. See section~\ref{sec:full-proof} for details; we show a sketch here. 
\begin{proof}[Proof sketch of Theorem~\ref{thm:expected-regret-total}]
Fix any $c \in (0, \frac14]$. Let $\varepsilon_1 = \varepsilon_2 = c\Delta_a$; by the choice of $c$, $\varepsilon_1 + \varepsilon_2 < \Delta_a$. 
From Lemma~\ref{lemma:expected-sub-optimal-arm-pull-main}, $\EE\sbr{N_{T,a}}$ is bounded by Eq.~\eqref{eqn:arm-pull-bound-main}.
Plugging in the values of $\varepsilon_1 = \varepsilon_2$, and using Lemma~\ref{lemma:KL-lower-bound} that lower bounds the binary KL divergence, along with Lemma~\ref{lemma:H-concavity-ineq} that gives $\splitFthree \lesssim \frac{2\dmu_1 + \varepsilon_2}{\varepsilon_2^2}$, and algebra, 
all terms except the second term on the right hand side of Eq.~\eqref{eqn:arm-pull-bound-main} are bounded by  
\[
\left( \frac{34}{c^2} + \frac{4}{(1-2c)^2} \right) \rbr{\frac{\dmu_1 + \Delta_a}{c^2 \Delta_a^2}} \ln\del[3]{ \del[3]{  \frac{\dmu_1 + \Delta_a}{c^2\Delta_a^2} \wedge \frac{c^2 T \Delta_a^2}{\dmu_1 + \Delta_a} } \vee e^2 }.
\]
As a result,  KL-MS satisfies that, for any arm $a$, for any $c \in (0, \frac14]$: 
    \begin{align*}
    \EE\sbr{N_{T,a}}
    \leq& 
    \frac{\ln(T \kl(\mu_a + c \Delta_a, \mu_1 - c \Delta_a) \vee e^2 )}{\kl(\mu_a + c \Delta_a, \mu_1 - c \Delta_a)} \\
    &+
    \left( \frac{34}{c^2} + \frac{4}{(1-2c)^2} \right)\rbr{\frac{\dmu_1 + \Delta_a}{c^2 \Delta_a^2}} \ln\rbr{ \rbr{  \frac{\dmu_1 + \Delta_a}{ c^2\Delta_a^2} \wedge \frac{c^2 T \Delta_a^2}{\dmu_1 + \Delta_a }} \vee e^2 }  .
    \end{align*}
Theorem~\ref{thm:expected-regret-total}  follows by plugging the above bound to Eq.~\eqref{eqn:reg-decomp} for arms $a$ s.t. $\Delta_a > \Delta$ with $c=\frac{1}{4}$.
\end{proof}

\vspace{-.5em}
\subsection{Proof sketch of Lemma~\ref{lemma:expected-sub-optimal-arm-pull-main}}
\vspace{-.5em}

We sketch the proof of Lemma~\ref{lemma:expected-sub-optimal-arm-pull-main} in this subsection.
For full details of the proof, please refer to Appendix~\ref{sec:pf-expected-sub-optimal-arm-pull}.
We first set up some useful notations that will be used throughout the proof. 
Let $u := \defU$. 
We define the following events
    \[
        A_t := \cbr{I_t = a}, \quad 
        B_t := \cbr{N_{t,a} < u}, \quad
        C_t := \cbr{\hat{\mu}_{t,\max} \geq \mu_1-{\varepsilon_2}}, \quad
        D_t := \cbr{\hat{\mu}_{t,a} \leq \mu_a+{\varepsilon_1}}, \quad
    \]

By algebra, one has the following elementary upper bound on $\EE\sbr{N_{T,a}}$:
$\EE\sbr{N_{T,a}} \leq  u + \EE\sbr{\sum_{t=K+1}^T\one\cbr{ A_t, B_{t-1}^c }}
$. Intuitively, the $u$ term serves to control the length of a "burn-in" phase when the number of pulls to arm $a$ is at most $u$. 
It now remains to control the second term, the number of pulls to arm $a$ after it is large enough, i.e., $N_{t-1,a} \geq u$. We decompose it to $F1$, $F2$, and $F3$, resulting in the following inequality: 
\begin{align*}
    \EE\sbr{N_{T,a}} \leq &  u + \underbrace{\EE\sbr[4]{\sum_{t=K+1}^T 
    \one\cbr{A_t, B_{t-1}^c, 
        C_{t-1}, D_{t-1}
        }
    }}_{=:F1}
    \\
    &+ \underbrace{\EE\sbr[4]{\sum_{t=K+1}^T 
    \one\cbr{A_t, B_{t-1}^c, 
        C_{t-1}, 
        D_{t-1}^c,
        }
    }}_{=:F2}
     + \underbrace{\EE\sbr[4]{\sum_{t=K+1}^T 
        \one\cbr{A_t, B_{t-1}^c, C_{t-1}^c
        }
    }}_{=:F3}
\end{align*}

Here:
\vspace{-.5em}
\begin{itemize}
\item $F1$ corresponds to the ``steady state'' when the empirical means of arm $a$ and the optimal arm are both estimated accurately, i.e., $\hat{\mu}_{t-1,\max} \geq \mu_1 - \varepsilon_2$ and $\hat{\mu}_{t-1,a}  \leq  \mu_a + \varepsilon_1$. It can be straightforwardly bounded by $\BoundFOne$, as we show in Lemma~\ref{lemma:F1-upper-bound} (section~\ref{sec:f1}).  
\item $F2$ corresponds to the case when the empirical mean of arm $a$ is abnormally high, i.e., $\hat{\mu}_{t-1,a} > \mu_a + \varepsilon_1$. It can be straightforwardly bounded by $\BoundFTwo$, as we show in Lemma~\ref{lemma:F2-upper-bound} 
(section~\ref{sec:f2}).  
\item $F3$ corresponds to the case when the empirical mean of the optimal arm is abnormally low, i.e., $\hat{\mu}_{t-1,\max} \leq \mu_1 - \varepsilon_2$; it is the most challenging term and we discuss our techniques in bounding it in detail below (section~\ref{sec:f3}). 
\end{itemize}
\vspace{-.5em}

We provide an outline of our analysis of $F3$ in Appendix~\ref{sec:analysis-roadmap} and sketch its main ideas and technical challenges here.

We follow the derivation from \citet{bian2022maillard} by first using a probability transferring argument (Lemma~\ref{lemma:prob-transfer}) to bound the expected counts of pulling suboptimal arm $a$ by the expectation of indicators of pulling the optimal arm with a multiplicative factor and then change the counting from global time step $t$ to local count of pulling the optimal arm. Then, $F3$ is bounded by,
    \[
    \sum_{k=1}^\infty
    \underbrace{\EE\sbr{ 
     \one\cbr{ \hat{\mu}_{(k),1} \leq \mu_1-\varepsilon_2 } \exp( k \cdot \kl(\hat{\mu}_{(k),1}, \mu_1-\varepsilon_2)}
    }_{M_k}.
    \]
Intuitively, each $M_k$ should be controlled: when $\exp( k \cdot \kl(\hat{\mu}_{(k),1}, \mu_1-\varepsilon_2))$ is large, $\hat{\mu}_{(k),1}$ must significantly negatively deviate from $\mu_1 - \varepsilon_2$, which happens with low probability by Chernoff bound (Lemma~\ref{lemma:maximal-inequality}).
Using a double integration argument, we can bound each $M_k$ by 
\[ 
M_k \leq 
\del{ \frac{ 2 \splitFthree }
{ k } 
+ 
1} \exp(-k\kl(\mu_1 - \varepsilon_2, \mu_1)).
\]
Summing over all $k$, we can bound $F3_1$ by $O\del{\splitFthree \ln\del{\splitFthree \vee e^2} + \frac{1}{\kl\del{\mu_1-\varepsilon_2, \mu_1}}}$.
Combining the bounds on $F1$ and $F2$, we can show a bound on $\EE\sbr{N_{T,a}}$ similar to Eq.~\eqref{eqn:arm-pull-bound-main} without the ``$\frac{T}{H} \wedge$'' term in the logarithmic factor. This yields a regret bound of KL-MS, in the form of Eq.~\eqref{eqn:kl-ms-main-regret} without the ``$\wedge \frac{c^2 T \Delta_a^2}{\dmu_1 + \Delta_a}$'' term in the logarithmic factor. Such a regret bound can be readily used to show KL-MS's Bernoulli asympototic optimality and sub-UCB property. An adaptive worst-case regret bound of $\sqrt{ \dmu_1 K T \ln\del{T}}$ also follows immediately.

To show that MS has a tighter adaptive worst-case regret bound of $\sqrt{ \dmu_1 K T \ln\del{K}}$, we adopt a technique in~\cite{menard17minimax,jinfinite}. First, we observe that the looseness of the above bound on $F3$ comes from small $k$ (denoted as $F3_1:= \sum_{k \leq \splitFthree} M_k$), as the summation of $M_k$ for large $k$ (denoted as $F3_2 := \sum_{k > \splitFthree} M_k$) is well-controlled.
The key challenge in a better control of $F3_1$ comes from the difficulty in bounding the tail probability of $\hat{\mu}_{(k),1}$ for $k < \splitFthree$ beyond Chernoff bound.
To cope with this, we observe that a modified version of $F3_1$ that contains an extra favorable indicator of $\kl( \hat{\mu}_{(k),1}, \mu_1 ) \leq \frac{2\ln\del{T/k}}{k}$, 
denoted as:
\[
\sum_{k \leq \splitFthree}
\EE\sbr{ 
     \one\cbr{ \hat{\mu}_{(k),1} \leq \mu_1-\varepsilon_2, \kl( \hat{\mu}_{(k),1}, \mu_1 ) \leq \frac{2\ln\del{T/k}}{k} } \exp( k \cdot \kl(\hat{\mu}_{(k),1}, \mu_1-\varepsilon_2)}
\]
can be well-controlled. 
Utilizing this introduces another term in the regret analysis, $T \cdot \PP(\Ecal^C)$, where $\Ecal = \cbr[1]{\forall k \in [1,H], \kl( \hat{\mu}_{(k),1}, \mu_1 ) \leq \frac{2\ln\del{T/k}}{k}}$, which we bound by $O(H)$ via a time-uniform version of Chernoff bound. Putting everything together, we prove a bound of $F3$ of $O\del[1]{ \splitFthree \ln\del[1]{  (\frac{T}{\splitFthree} \wedge \splitFthree) \vee e^2 } + \frac{1}{\kl\del{\mu_1-\varepsilon_2, \mu_1}}}$, which yields our final regret bound of KL-MS in Theorem~\ref{thm:expected-regret-total} and the refined minimax ratio.

\begin{remark}
Although our technique is inspired by~\cite{jinfinite,menard17minimax}, 
we carefully set the case splitting threshold for $N_{t-1,1}$ (to obtain $F3_1$ and $F3_2$) to be $\splitFthree = O(\tfrac{\dmu_1 + \epsilon_2}{\epsilon_2^2})$, which is significantly different from prior works ($\tilde{O}(\tfrac{1}{\epsilon_2^2})$). 
\end{remark}

\begin{remark}
    One can port our proof strategy back to sub-Gaussian MS and show that it achieves a minimax ratio of $\sqrt{\ln K}$ as opposed to $\sqrt{\ln T}$ reported in~\citet{bian2022maillard}; a sketch of the proof is in Appendix~\ref{sec:sgms}.
    Recall that \citet{bian2022maillard} proposed another algorithm MS$^+$ that achieved the minimax ratio of $\sqrt{\ln K}$ at the price of extra exploration.
    Our result makes MS$^+$ obsolete; MS should be preferred over MS$^+$ at all times.
\end{remark}

\vspace{-.5em}
\section{Conclusion}
\vspace{-.5em}

We have proposed KL-MS, a KL version of Maillard sampling for stochastic multi-armed bandits in the $[0,1]$-bounded reward setting, with a closed-form probability computation, which is highly amenable to off-policy evaluation.
Our algorithm requires constant time complexity with respect to the target numerical precision in computing the action probabilities, and our regret analysis shows that KL-MS achieves the best regret bound among those in the literature that allows computing the action probabilities with $O(\text{polylog}(1/\text{precision}))$ time complexity, for example, Tsallis-INF~\cite{zimmert21tsallis}, EXP3++~\cite{seldin2014one}, in the stochastic setting.

Our study opens up numerous open problems.
One immediate open problem is to generalize KL-MS to handle exponential family reward distributions.
Another exciting direction is to design randomized and off-policy-amenable algorithms that achieve the asymptotic optimality for bounded rewards (i.e., as good as IMED~\cite{honda15non}).

One possible avenue is to extend MED~\cite{honda2011asymptotically} and remove the restriction that the reward distribution must have bounded support.
Furthermore, it would be interesting to extend MS to structured bandits 
and find connections to the Decision-Estimation Coefficient~\cite{foster21thestatistical}, which have recently been reported to characterize the optimal minimax regret rate for structured bandits.
Finally, we believe MS is practical by incorporating the booster hyperparameter introduced in~\citet{bian2022maillard}.
Extensive empirical evaluations on real-world problems would be an interesting future research direction.

\bibliographystyle{abbrvnat}
\bibliography{references,library-shared}

\clearpage

\appendix

%%%%%%%%%%%%%%%%%%%%%%%%%%%%%%%%%%%%%%%%%%%%%%%%%%%%%%%%%%%%
\paragraph{Acknowledgments.} We thank Kyoungseok Jang for helpful discussions on refinements of binary Pinsker's inequality (Lemma~\ref{lemma:KL-lower-bound}). 
Hao Qin and Chicheng Zhang gratefully acknowledge funding support from the University of Arizona FY23 Eighteenth Mile TRIF Funding. 
%%%%%%%%%%%%%%%%%%%%%%%%%%%%%%%%%%%%%%%%%%%%%%%%%%%%%%%%%%%%

\section{Proof of Worst-case Regret Bounds (Theorem \ref{thm:minimax-regret-bound}) and Sub-UCB Property (Theorem~\ref{thm:sub-UCB})}
\label{sec:proof-worst-case}

Before proving Theorem~\ref{thm:minimax-regret-bound}, we first state and prove a useful lemma that gives us an upper bound to the regret, which is useful for subsequent minimax ratio analysis and asymptotic analysis. This regret bound consists of two components, which correspond to arms with suboptimality gaps at most or greater than a predetermined threshold $\Delta$ respectively
. The former is bounded by $T\Delta$, while the latter is upper bounded by a finer $\tilde{O}(\sum_{a: \Delta_a > \Delta} \frac{\dmu_1}{\Delta_a} + K)$ term.
\begin{lemma}
\label{lem:kl-ms-pre-sub-ucb}
For KL-MS, its regret is bounded by: for any $\Delta \geq 0$, 
\[
\Reg(T)
\leq 
T\Delta 
+
O\del{
    \sum_{a: \Delta_a > \Delta} 
        \rbr{\frac{\dmu_1 + \Delta_a}{\Delta_a}} 
        \ln\rbr{ \frac{T\Delta_a^2}{\dmu_1 + \Delta_a} \vee e^2 }  
}
\]
\end{lemma}
\begin{proof}
Applying Theorem \ref{thm:expected-regret-total} with $c = \frac14$, we have:
\begin{align*}
 & \Reg(T)
        \\
        \leq& 
            T\Delta 
            +
            \sum_{a: \Delta_a > \Delta}  
            \fr {\Delta_a \ln(T \kl(\mu_a + c \Delta_a, \mu_1 - c \Delta_a) \vee e^2 )}
                {\kl(\mu_a + c \Delta_a, \mu_1 - c \Delta_a)}
        \\
        &+ 
        392 \del{
            \sum_{a: \Delta_a > \Delta} 
                \rbr{\frac{\dmu_1 + \Delta_a}{c^4 \Delta_a}} 
                \ln\rbr{ \rbr{  \frac{\dmu_1 + \Delta_a}{\Delta_a^2} \wedge \frac{T\Delta_a^2}{\dmu_1 + \Delta_a} } \vee e^2 } 
        }
                \tag{Theorem \ref{thm:expected-regret-total}}
        \\
        \leq& 
            T\Delta 
            +
            O\del{
                \sum_{a: \Delta_a > \Delta} 
                    \rbr{\frac{\dmu_1 + \Delta_a}{\Delta_a}} 
                    \ln\rbr{ \frac{T\Delta_a^2}{\dmu_1 + \Delta_a} \vee e^2 }  
            },
                \tag{Lemma \ref{lemma:monoticity} and Lemma \ref{lemma:KL-lower-bound} }
\end{align*}
here, the second inequality is because we choose $\fr{T\Delta_a^2}{\dmu_1+\Delta_a}$ as the upper bound in the lower order term then we use Lemma \ref{lemma:KL-lower-bound} to lower bound $\kl(\mu_a+c\Delta_a,\mu_1-c\Delta_a) \gtrsim \frac{\Delta_a^2}{\dmu_a + \Delta_a}$ and Lemma \ref{lemma:monoticity} that $x \mapsto \fr{\ln(T x \vee e^2)}{x}$ is monotonically decreasing when $x \geq 0$. Also by 1-Lipshitzness of $z \mapsto z (1-z)$, we have $(\mu_1-c\Delta_a)(1-(\mu_1-c\Delta_a)) \leq \dmu_1 + c\Delta_a$ and all terms except $T\Delta$ will be merged into the $O( \cdot )$ term.
\end{proof}

    \begin{proof}[Proof of Theorem~\ref{thm:minimax-regret-bound}]
    Let $\Delta = \sqrt{\frac{ \dmu_1 K \ln{K}}{T}}$, from {Lemma~\ref{lem:kl-ms-pre-sub-ucb}} we have
    
    \begin{align*}
       \Reg(T)
       \leq & 
       T\Delta 
            +
            O\del{
                \sum_{a: \Delta_a > \Delta} 
                    \rbr{\frac{\dmu_1 + \Delta_a}{\Delta_a}} 
                    \ln\rbr{ \frac{T\Delta_a^2}{\dmu_1 + \Delta_a} \vee e^2 }  
            }
        \\
        \leq& 
            T\Delta 
            +
            O\del{
                \sum_{a: \Delta_a > \Delta} 
                    \frac{\dmu_1}{\Delta_a} 
                    \ln\rbr{ \frac{T\Delta_a^2}{\dmu_1} \vee e^2 }  
            }
            +
            O\del{
                \sum_{a: \Delta_a > \Delta} 
                    \ln\rbr{ \rbr{ T\Delta_a } \vee e^2 }  
            }
        \\
        \leq& 
            T\Delta 
            +
            O\del{
                \frac{K \dmu_1}{\Delta} 
                \ln\rbr{ \frac{T\Delta^2}{\dmu_1} \vee e^2 }  
            }
            +
            O\del{
                K  \ln\rbr{ T }  
            }
            \tag{Lemma \ref{lemma:monoticity}}
        \\
        \leq&
            O\del{ \sqrt{\dmu_1 KT\ln{K}} } + 
            O\del{ K \ln\del{ T } },
    \end{align*}
    where in the second inequality, we split fraction $\fr{\dmu_1+\Delta_a}{\Delta_a}$ into $\fr{\dmu_1}{\Delta_a}$ and $1$, then bound each term separately. 
    The second-to-last inequality is 
    due to the monotonicity of $x \mapsto \fr{\ln{(b x^2 \vee e^2})}{x}$ proven in Lemma \ref{lemma:monoticity}; 
    The last inequality is
    by algebra. 
    \end{proof}

\begin{proof}[Proof of Theorem~\ref{thm:sub-UCB}]
This is an immediate consequence of Lemma~\ref{lem:kl-ms-pre-sub-ucb} with $\Delta = 0$, along with the observations that 
$\frac{\dmu_1 + \Delta_a}{\Delta_a} \leq \frac{2}{\Delta_a}$,
and 
$\frac{T\Delta_a^2}{\dmu_1 + \Delta_a} \leq T$.
\end{proof}

\section{Proof of Asymptotic Optimality (Theorem \ref{thm:asymptotic-optimality})}
We establish asymptotic optimality of KL-MS by analyzing the ratio between the expected regret to $\ln T$ and letting $T \to \infty$.
    \begin{proof}

    Starting from Theorem \ref{thm:expected-regret-total} and letting $\Delta = 0$ and $c = \fr{1}{\ln\ln{T}}$:
    \begin{align*}
        & \limsup_{T \to \infty} \fr{ \Reg(T) }{\ln(T)}
        \\
        \leq&
        \lim_{T \to \infty}
        \sum_{a\in[K]: \Delta_a > 0}  
        \fr {\Delta_a \ln(T \kl(\mu_a + c \Delta_a, \mu_1 - c \Delta_a) \vee e^2 )}
            {\ln{T} \kl(\mu_a + c \Delta_a, \mu_1 - c \Delta_a)}
        \\
        &+ 
        \lim_{T \rightarrow \infty}
        392 \del{  
            \sum_{a\in[K]: \Delta_a > 0} 
            \rbr{\frac{\dmu_1 + \Delta_a}{c^4 \ln{T} \Delta_a}} \ln\rbr{ \rbr{  \frac{\dmu_1 + \Delta_a}{\Delta_a^2} } \vee e^2 }  
        } 
                \tag{Theorem \ref{thm:expected-regret-total}}
        \\
        \leq&
            \lim_{T \rightarrow \infty} \sum_{a\in[K]:\Delta_a>0} 
            \fr {\Delta_a \ln(T \kl(\mu_a + c \Delta_a, \mu_1 - c \Delta_a) \vee e^2 )}
            {\kl(\mu_a + c \Delta_a, \mu_1 - c \Delta_a) \ln{T} }
        \\
        =&
            \lim_{T \rightarrow \infty} \sum_{a\in[K]:\Delta_a>0} 
            \fr {\Delta_a \ln(T \kl(\mu_a + c \Delta_a, \mu_1 - c \Delta_a) \vee e^2 )}
            { \kl(\mu_a, \mu_1) \ln{T} }
            \cd
            \fr {\kl(\mu_a, \mu_1)}{\kl(\mu_a + c \Delta_a, \mu_1 - c \Delta_a)}
        \\
        =&
            \sum_{a\in[K]:\Delta_a>0}
            \frac{\Delta_a}{\kl(\mu_a, \mu_1)},
                \tag{By the continuity of $\kl(\cdot, \cdot)$}
    \end{align*}
    where the first inequality is because 
    of the fact that $\ln\rbr{ \rbr{  \frac{\dmu_1 + \Delta_a}{c^2 \Delta_a^2} \wedge \frac{c^2 T}{\frac{\dmu_1 + \Delta_a}{\Delta_a^2}} } \vee e^2 } \leq \frac{1}{c^2} \ln\rbr{ \frac{\dmu_1 + \Delta_a}{ \Delta_a^2} \vee e^2 } $ due to Lemma~\ref{lem:log-c-move}, and,
     the second inequality is due to that when $T \rightarrow \infty$, $c^4 \ln{T} = \fr{\ln{T}}{(\ln\ln{T})^4} \to \infty$. 
     \end{proof}

\section{Full Proof of Theorem~\ref{thm:expected-regret-total}}
\label{sec:full-proof}

\subsection{A general lemma on the expected arm pulls and its implication to Theorem~\ref{thm:expected-regret-total}}
\label{sec:kl-lemma-thm}

We first present a general lemma that bounds the number of pulls to arm $a$ by KL-MS; due to its technical nature, we defer its proof to Section~\ref{sec:pf-expected-sub-optimal-arm-pull} and focus on its implication to Theorem~\ref{thm:expected-regret-total} in this section.

\begin{lemma}[Lemma~\ref{lemma:expected-sub-optimal-arm-pull-main} restated]\label{lemma:expected-sub-optimal-arm-pull}
    For any suboptimal arm $a$, let $\varepsilon_1, \varepsilon_2 > 0$ be such that $\varepsilon_1 + \varepsilon_2 < \Delta_a$. 
    Then its expected number of pulls is bounded as:
    \begin{align}
    \EE\sbr{ N_{T,a} }
    \leq & 1 + \frac{ \ln\del{ T \kl(\mu_a + \varepsilon_1, \mu_1 - \varepsilon_2) \vee e^2} }
                { \kl(\mu_a+\varepsilon_1,\mu_1-\varepsilon_2) } 
    + \BoundFOne + \BoundFTwo \\
    & + \BoundFThree,
    \label{eqn:arm-pull-bound}
    \end{align}
    where $\splitFthree := \defSplitFthree$ and $h(\mu_1,\varepsilon_2):= \ln \del{\fr{(1-\mu_1+\varepsilon_2)\mu_1}{(1-\mu_1)(\mu_1-\varepsilon_2)}}$.
\end{lemma}

We now use Lemma~\ref{lemma:expected-sub-optimal-arm-pull} to conclude Theorem~\ref{thm:expected-regret-total}.

\begin{proof}[Proof of Theorem~\ref{thm:expected-regret-total}]
Fix any $c \in (0, \frac14]$. Let $\varepsilon_1 = \varepsilon_2 = c\Delta_a$; note that by the choice of $c$, $\varepsilon_1 + \varepsilon_2 < \Delta_a$. 
From Lemma~\ref{lemma:expected-sub-optimal-arm-pull}, $\EE\sbr{N_{T,a}}$ is bounded by Eq.~\eqref{eqn:arm-pull-bound}.
We now plug in the value of $\varepsilon_1, \varepsilon_2$, and 
further upper bound the third to the sixth terms of the right hand side of Eq.~\eqref{eqn:arm-pull-bound}: 
\begin{itemize}
\item 
\[ 
\BoundFOne
\leq 
\frac{
2(\mu_1-\varepsilon_2)(1-\mu_1+\varepsilon_2) + 2(\Delta_a-\varepsilon_1-\varepsilon_2)}
{(\Delta_a-\varepsilon_1-\varepsilon_2)^2}
\leq
\frac{2}{(1-2c)^2} \cdot
\frac{\dot{\mu}_1 + \Delta_a}{\Delta_a^2}
\]
\item 
\[
\BoundFTwo
\leq 
\fr {2\dmu_a + 2\varepsilon_1}
                    {\varepsilon_1^2}
\leq
\frac{4}{c^2} \cdot
\frac{\dot{\mu}_1 + \Delta_a}{\Delta_a^2}
\]
\item 
\[
\frac{4}{\kl(\mu_1-\varepsilon_2,\mu_1)}
\leq 
\fr {8\dmu_1 + 8\varepsilon_2}
                    {\varepsilon_2^2}
\leq
\frac{8}{c^2} \cdot
\frac{\dot{\mu}_1 + \Delta_a}{\Delta_a^2}
\]
\item By Lemma~\ref{lemma:H-concavity-ineq}, $H \leq \frac{ 2 \dmu_1 + 2 \varepsilon_2}{\varepsilon_2^2} \leq \frac{2}{c^2} \cdot \frac{\dmu_1 + \Delta_a}{\Delta_a^2}$, and by Lemma~\ref{lem:ln-x-x-decr}, the function $H \mapsto 6 H \ln\del{ (\fr{T}{H} \wedge H) \vee e^2 }$ is monotonically increasing, we have that
\begin{align*}
    6 H \ln\del{ (\fr{T}{H} \wedge H) \vee e^2 }
    \leq & 
    \frac{12}{c^2}\cdot
    \rbr{\frac{\dmu_1 + \Delta_a}{\Delta_a^2}} 
        \ln\rbr{ \rbr{  \frac{\dmu_1 + \Delta_a}{c^2 \Delta_a^2} \wedge \frac{c^2 T  \Delta_a^2}{\dmu_1 + \Delta_a }} \vee e^2 }  
\end{align*}
\end{itemize}

Combining all the above bounds and Eq.~\eqref{eqn:arm-pull-bound}, 
KL-MS satisfies that, for any arm $a$, 
for any $c \in (0, \frac14]$: 
    \begin{align}
    \EE\sbr{N_{T,a}}
    \leq& 
    \frac{\ln(T \kl(\mu_a + c \Delta_a, \mu_1 - c \Delta_a) \vee e^2 )}{\kl(\mu_a + c \Delta_a, \mu_1 - c \Delta_a)}
    \nonumber \\
    &+ 
    \del{ \frac{24}{c^2} + \frac{4}{(1-2c)^2} } \rbr{\frac{\dmu_1 + \Delta_a}{\Delta_a^2}} \ln\rbr{ \rbr{  \frac{\dmu_1 + \Delta_a}{c^2 \Delta_a^2} \wedge \frac{c^2 T  \Delta_a^2}{\dmu_1 + \Delta_a} } \vee e^2 }
    \label{eqn:arm-pull-bound-simp}
    \end{align}

For any $\Delta \geq 0$, we now bound the pseudo-regret of KL-MS as follows:
\begin{align*}
& \Reg(T) \\
= &  \sum_{a: \Delta_a > 0} \Delta_a \EE\sbr{N_{T,a}} \\
= & \sum_{a: \Delta_a \in (0,\Delta]} \Delta_a \EE\sbr{N_{T,a}} + \sum_{a: \Delta_a > \Delta} \Delta_a \EE\sbr{N_{T,a}} \\
\leq & \Delta T + \sum_{a: \Delta_a > \Delta} \Delta_a \frac{\ln(T \kl(\mu_a + c \Delta_a, \mu_1 - c \Delta_a) \vee e^2 )}{\kl(\mu_a + c \Delta_a, \mu_1 - c \Delta_a)} \\
    & + 
    \del{ \frac{24}{c^2} + \frac{4}{(1-2c)^2} } \sum_{a: \Delta_a > \Delta} \rbr{\frac{\dmu_1 + \Delta_a}{ \Delta_a}} \ln\rbr{ \rbr{  \frac{\dmu_1 + \Delta_a}{c^2 \Delta_a^2} \wedge \frac{c^2 T \Delta_a^2}{\dmu_1 + \Delta_a} } \vee e^2 }, 
\end{align*}
where the last inequality is from Eq.~\eqref{eqn:arm-pull-bound-simp}.
Then we pick $c=\frac{1}{4}$ and conclude the proof of the theorem.
\end{proof}

\subsection{Proof of Lemma~\ref{lemma:expected-sub-optimal-arm-pull}: arm pull count decomposition and additional notations}
\label{sec:pf-expected-sub-optimal-arm-pull}

In this subsection, we prove Lemma~\ref{lemma:expected-sub-optimal-arm-pull}. 
We first 
recall the following set of useful notations defined in Section~\ref{sec:proof-sketch}:

Recall that $u = \defU$, and  
we have defined the following events
    \begin{align*}
        & A_t := \cbr{I_t = a} \\
        & B_t := \cbr{N_{t,a} < u} \\
        & C_t := \cbr{\hat{\mu}_{t,\max} \geq \mu_1-{\varepsilon_2}} \\
        & D_t := \cbr{\hat{\mu}_{t,a} \leq \mu_a+{\varepsilon_1}} 
    \end{align*}

\paragraph{A useful decomposition of the expected number of pulls to arm $a$.} With the notations above, we bound the expected number of pulling any suboptimal $a$ by  
decomposing the arm pull indicator $\one\cbr{I_t = a}$ according to events $B_{t-1}, C_{t-1}^c$ and $D_{t-1}$ in a cascading manner:
    \begin{align}
    \EE[N_{T,a}] 
    &= 
        \EE\sbr{ \sum_{t=1}^T \one\cbr{I_t = a} }
    \\
    &= 
        1 + \EE\sbr{ \sum_{t=K+1}^T \one\cbr{A_t} }
            \tag{Definition of Algorithm~\ref{alg:KL-MS}}
    \\
    &=
        1
        +
        \EE\sbr{ \sum_{t=K+1}^T \one\cbr{A_t, B_{t-1}} } 
        + 
        \EE\sbr{ \sum_{t=K+1}^T \one\cbr{A_t, B_{t-1}^c} }
    \\
    &\leq 
        1 + (u - 1) + \EE\sbr{\sum_{t=K+1}^T \one\cbr{A_t, B_{t-1}^c}}
            \tag{Lemma \ref{lemma:sum-prob-cumul-sum}
            }
    \\
    &= u + \underbrace{\EE\sbr{\sum_{t=K+1}^T 
    \one\cbr{A_t, B_{t-1}^c, 
        C_{t-1}, D_{t-1}
        }
    }}_{F1}
        \label{eqn:F_1}
    \\
    &\quad + \underbrace{\EE\sbr{\sum_{t=K+1}^T 
    \one\cbr{A_t, B_{t-1}^c, 
        C_{t-1}, 
        D_{t-1}^c
        }
    }}_{F2}
        \label{eqn:F_2}
    \\
    &\quad + \underbrace{\EE\sbr{\sum_{t=K+1}^T 
        \one\cbr{A_t, B_{t-1}^c, C_{t-1}^c
        }
    }}_{F3}
        \label{eqn:F_3}
    \end{align}

Given the above decomposition, the lemma is now an immediate consequence of the definition of $u$, Lemmas~\ref{lemma:F1-upper-bound},~\ref{lemma:F2-upper-bound} and~\ref{lemma:F3-upper-bound} (that bounds $F1, F2, F3$ respectively), which we state and prove in  Appendix~\ref{sec:bounding-f-i}.

\section{Bounding the number of arm pulls in each case}
\label{sec:bounding-f-i}

    \subsection{F1}
    \label{sec:f1}
    In this section we bound $F1$. This is the case that $\hat{\mu}_{t,a}$ is small and $\hat{\mu}_{t,\max}$ is large, so that $\kl(\hat{\mu}_{t,a}, \hat{\mu}_{t,\max})$ do not significantly underestimate
    $\kl(\mu_a, \mu_1)$, which will imply that suboptimal arm $a$ will be only pulled a small number of times due to the arm selection rule (Eq.~\eqref{eqn:kl-ms-rule}). 
    Note that $u$ is set carefully so that $F1$ is bounded just enough to be lower than the $\frac{\ln T}{\kl(\mu_a, \mu_1)}$ Bernoulli asymptotic lower bound. 

    \begin{lemma} \label{lemma:F1-upper-bound}
        \[ 
            F1 \leq \BoundFOne
        \]
    \end{lemma}

    \begin{proof}
    Recall the notations that $A_t = \cbr{I_t = a}$, $B_{t-1}^c = \cbr{N_{t-1,a} \geq u}$, $C_{t-1} = \cbr{\hat{\mu}_{t-1,\max}\geq\mu_1-{\varepsilon_2}}$, 
    $D_{t-1} = \cbr{ \hat{\mu}_{t-1,a}\leq\mu_a+{\varepsilon_1} }$.
    We have:
        
    \begin{align}
        F1=& \EE\sbr{
            \sum_{t=K+1}^T 
            \one\cbr{
                A_t, B_{t-1}^c, C_{t-1}, D_{t-1}
                }
            }
        \\
        =& \sum_{t=K+1}^{T} 
            \EE\sbr{
                \EE\sbr{
                \one\cbr{
                    A_t, B_{t-1}^c, C_{t-1}, D_{t-1}
                    }
                    \mid \cH_{t-1}
                } 
            }
                \tag{Law of total expectation }
        \\
        =&  \sum_{t=K+1}^{T} 
            \EE\sbr{
                \one\cbr{B_{t-1}^c, C_{t-1}, D_{t-1}}
                \EE\sbr{
                \one\cbr{
                    A_t  
                    }
                    \mid \cH_{t-1}
                } 
            }
                \tag{$B_{t-1}, C_{t-1}, D_{t-1}$ are $\cH_{t-1}$-measurable}
        \\
        \leq&  \sum_{t=K+1}^{T} 
            \EE\sbr{
                \one\cbr{B_{t-1}^c, C_{t-1}, D_{t-1}}
                \exp(-N_{t-1,a} \kl(\hat{\mu}_{t-1,a},\hat{\mu}_{t-1,\max}))
            }
                \tag{By Lemma \ref{lemma:prob-transfer}}
        \\
        \leq&  \sum_{t=K+1}^{T} 
            \EE\sbr{
                \one\cbr{B_{t-1}^c}
                \exp(-u \cdot \kl(\mu_a + \varepsilon_1, \mu_1 - \varepsilon_2 ))
            }
                \tag{Based on $B_{t-1}^c, C_{t-1}$ and $ D_{t-1}$, there is $N_{t-1,a} \geq u$ and $\kl\del{ \hat{\mu}_{t-1,a}, \hat{\mu}_{t-1,\max} } \geq \kl\del{ \mu_a+\varepsilon_1, \mu_1-\varepsilon_2 }$}
        \\
        \leq& T \cdot 
            \exp(-u \cdot \kl(\mu_a + \varepsilon_1, \mu_1 - \varepsilon_2 ))
                \tag{$\one\cbr{\cdot} \leq 1$}
        \\
        \leq& T \cdot 
            \frac{1}{T \kl(\mu_a + \varepsilon_1, \mu_1 - \varepsilon_2)}
                \tag{Recall definition of $u$}
        \\
        =& \BoundFOne \nonumber \qedhere
    \end{align} 
    \end{proof}
    
    \subsection{F2}
    \label{sec:f2}
    In this section we upper bound $F2$. This is the case when the suboptimal arm $a$'s mean reward is overestimated by at least $\varepsilon_2$. Intuitively this should not happen too many times, due to the concentration between the empirical mean reward and the population mean reward of arm $a$.
    
    \begin{lemma} \label{lemma:F2-upper-bound}
        \[
            F2 \leq \BoundFTwo
        \]
    \end{lemma}
    
    \begin{proof}
    Recall the notations that $A_t = \cbr{I_t = a}$, $B_{t-1}^c = \cbr{N_{t-1,a} \geq u}$, $C_{t-1} = \cbr{\hat{\mu}_{t-1,\max}\geq\mu_1-{\varepsilon_2}}$, 
    $D_{t-1}^c = \cbr{ \hat{\mu}_{t-1,a}>\mu_a+{\varepsilon_1} }$. We have:

    \begin{align}
    F2 =
    & 
    \EE\sbr{\sum_{t=K+1}^T 
        \one\cbr{
            A_t, 
            B_{t-1}^c, 
            C_{t-1},
            D_{t-1}^c
            }} 
    \\
    \leq & 
        \EE\sbr{
        \sum_{k=2}^\infty 
        \one\cbr{
            B_{\tau_a(k)-1}^c, 
            C_{\tau_a(k)-1},
            D_{\tau_a(k)-1}^c
            }}
            \tag{implies that only when $t =\tau_a(k)$ for some $k \geq 2$ the inner indicator is non-zero
            }
    \\
    \leq & 
        \EE\sbr{
        \sum_{k=2}^\infty 
        \one\cbr{ D_{\tau_a(k)-1}^c }}
            \tag{Drop unnecessary conditions}
    \\
    = & 
        \EE\sbr{
        \sum_{k=2}^\infty 
        \one\cbr{ D_{\tau_a(k-1)}^c }}
            \tag{$\hat{\mu}_{\tau_a(k)-1,a} = \hat{\mu}_{\tau_a(k-1),a} $}
    \\
    = & 
        \EE\sbr{
        \sum_{k=1}^\infty 
        \one\cbr{ D_{\tau_a(k)}^c }}
            \tag{shift time index $t$ }
    \\
    \leq &
        \sum_{k=1}^\infty
        \exp( - k\cdot \kl(\mu_a + \varepsilon_1, \mu_a) ) 
            \tag{By Lemma \ref{lemma:maximal-inequality}}
    \\
    \leq &
        \frac{\exp( - \kl(\mu_a + \varepsilon_1, \mu_a) )}
             {1 - \exp( - \kl(\mu_a + \varepsilon, \mu_a) )}
            \tag{Geometric sum}
    \\
    \leq & 
        \BoundFTwo
            \tag{Applying inequality $e^x\ge 1+x$ when $x\ge 0$}
    \\
            \label{eqn:F_2_upper_bound}
    \end{align}

    Note that in the first inequality, we use the observation that for every $t \geq K+1$ such that $A_t$ happens, there exists a unique $k \geq 2$ such that $t = \tau_a(k)$.
    The third inequality is due to the Chernoff's inequality (Lemma \ref{lemma:maximal-inequality}) on the random variable $\hat{\mu}_{\tau_a(k),a}-\mu_a$. Given any $\tau_a(k)$, $\hat{\mu}_{\tau_a(k),a}$ is the running average reward of the first $k$'s pulling of arm $a$. In each pulling of arm $a$ the reward follows a bounded distribution $\nu_a$ with mean $\mu_a$ independently.
    \end{proof}

    \subsection{F3}
    \label{sec:f3}
    In this section we upper bound $F3$, which counts the expected number of times steps when arm $a$ is pulled while $\hat{\mu}_{t-1,\max}$ underestimates $\mu_1$ by at least $\varepsilon_2$. Our main result of this section is the following lemma:

    \begin{lemma} \label{lemma:F3-upper-bound}
        \begin{align*}
            F3 \leq \BoundFThree,
        \end{align*}
        where we recall that $H = \defSplitFthree$.
    \end{lemma}
    
\subsubsection{Roadmap of analysis}
\label{sec:analysis-roadmap}
    
    Before proving the lemma, we sketch the key ideas underlying our proof. First, 
    note that by the KL-MS sampling rule (Eq.~\eqref{eqn:kl-ms-rule}), at any time step $t$, 
    $p_{t,1}$ should not be too small ($p_{t,1} = \exp(-N_{t-1,1} \kl(\hat{\mu}_{t-1,1}, \hat{\mu}_{t-1,\max})) / M_t$), and as a result,
    the conditional probability of pulling arm $a$, $p_{t,a}$ should be not much higher than that of arm 1, $p_{t,1}$;
    using this along with a ``probability transfer'' argument similar to~\cite{agrawal2017near,bian2022maillard} 
    (see Lemma~\ref{lemma:prob-transfer} for a formal statement)
    tailored to KL-MS sampling rule, 
    we have: 
    \begin{align*} 
    F3 \leq \EE\sbr{ \sum_{t=K+1}^T \one \cbr{ A_t, C_{t-1}^c } }
    \leq & \EE \sbr{ \sum_{t=K+1}^T \one \cbr{ I_t =1, C_{t-1}^c} \exp(N_{t-1,1} \cdot \kl(\hat{\mu}_{t-1,1}, \mu_1-\varepsilon_2) } \\
    \leq & \EE \sbr{ \sum_{t=K+1}^T \one \cbr{ I_t =1, \hat{\mu}_{t-1,1} \leq \mu_1 - \varepsilon_2} \exp(N_{t-1,1} \cdot \kl(\hat{\mu}_{t-1,1}, \mu_1-\varepsilon_2) }
    \end{align*}

    By filtering the time steps when $I_t = 1$, the above can be upper bounded by an expectation over the outcomes in arm 1:  
    \[
    \sum_{k=1}^\infty
    \EE\sbr{ 
     \one\cbr{ \hat{\mu}_{(k),1} \leq \mu_1-\varepsilon_2 } \exp( k \cdot \kl(\hat{\mu}_{(k),1}, \mu_1-\varepsilon_2)
    }
    \]
    Intuitively, this is well-controlled, as by Chernoff bound (Lemma~\ref{lemma:maximal-inequality}), the probability that $\one\cbr{ \hat{\mu}_{(k),1} \leq \mu_1-\varepsilon_2 }$ is nonzero is exponentially small in $k$;  
    therefore, the expectation of 
    $\one\cbr{ \hat{\mu}_{(k),1} \leq \mu_1-\varepsilon_2 } \exp( k \cdot \kl(\hat{\mu}_{(k),1}, \mu_1-\varepsilon_2)$
    can be controlled. After a careful calculation that utilizes
    a double-integral argument (that significantly simplifies similar arguments in~\cite{bian2022maillard,jinfinite}), we can show that it is at most 

    \[
        2 \splitFthree \sum_{k=1}^{\lfloor \splitFthree \rfloor}
        \frac{ 1 }
        { k } +
        \fr{1}{\kl(\mu_1-\varepsilon_2,\mu_1)}
    \]

    Summing this over all $k$, we can upper bound $F3$ by 
    \begin{equation}
    F3 
    \leq 
    O\del{ H \ln\del{ H \vee e^2 }
    +
    \frac{1}{\kl(\mu_1-\varepsilon_2,\mu_1)} }.
    \label{eqn:f3-bound-coarse}
    \end{equation}

    A slight generalization of the above argument yields the following useful lemma which further focuses on bounding the expected number of time steps when the number of pulls of arm 1 is in interval $(m, n]$; we defer its proof to Section~\ref{sec:useful-lemma-f3}: 
    \begin{lemma} 
    \label{lemma:bound-F3}
        Recall the notations $A_t = \cbr{I_t = a}$, $C_{t-1} = \cbr{\hat{\mu}_{t-1,\max} \geq \mu_1-{\varepsilon_2}}$. 
        Define event $S_{t} = \cbr{N_{t,1} > m }$ and $T_{t} = \cbr{N_{t,1} \leq n}$ where $m \leq n$ and $m,n \in \NN \cup \cbr{\infty}$. Then we have the following inequality: 
    
        \begin{align*}
            \EE \sbr{ \sum_{t=K+1}^T
            \one \cbr{ A_t, C^c_{t-1}, S_{t-1}, T_{t-1} } }
            \leq
            \sum_{k=m+1}^{n}  
                \del{ \frac{ 2 \splitFthree }
                     { k } +
                     1}
                     \exp(-k\kl(\mu_1 - \varepsilon_2, \mu_1))
        \end{align*}
    \end{lemma}

    Naively, the bound of $F3$ given by Eq.~\eqref{eqn:f3-bound-coarse}, when combined with previous bounds on $F1$, $F2$, suffice to bound $\EE[N_{T,a}]$ by 
    \[
    \frac{\ln(T \kl(\mu_a + c \Delta_a, \mu_1 - c \Delta_a) \vee e^2 )}{\kl(\mu_a + c \Delta_a, \mu_1 - c \Delta_a)}
    +
    O\left( \rbr{\frac{\dmu_1 + \Delta_a}{c^2 \Delta_a^2}} \ln\rbr{ \frac{\dmu_1 + \Delta_a}{c^2 \Delta_a^2}  \vee e^2 }  \right)
    \]
    which establishes KL-MS's asymptotic optimality in the Bernoulli setting and a $O(\sqrt{ \dmu_1 K T \ln T } + K \ln T)$ regret bound. 
    To show a refined $O(\sqrt{ \dmu_1 K T \ln K } + K \ln T)$ regret bound, we prove another bound of $F3$: 
    \begin{equation}
    F3 
    \leq 
    O\del{ H \ln\del{ \frac T H \vee e^2 }
    +
    \frac{1}{\kl(\mu_1-\varepsilon_2,\mu_1)} }.
    \label{eqn:f3-bound-refine}
    \end{equation}
    This bound is sometimes stronger than bound~\eqref{eqn:f3-bound-coarse}, since its logarithmic factor depends on $\frac{T}{H}$, which can be substantially smaller than $H$. This alternative bound is crucial to achieve to achieve the $\sqrt{\ln K}$ minimax ratio; see 
    Appendix~\ref{sec:proof-worst-case} and the proof of Theorem~\ref{thm:minimax-regret-bound} therein for details.
    
    To this end, we decompose $F3$ according to whether the number of times arm 1 get pulled exceeds threshold $H$:
    \begin{align}
        F3 
        \leq & \EE\sbr{\sum_{t=K+1}^T 
                \one\cbr{
                    A_t, 
                    C_{t-1}^c}
                    }
        \nonumber \\
        = & 
            \underbrace{
            \EE\sbr{\sum_{t=K+1}^T 
                \one\cbr{A_t, C_{t-1}^c, E_{t-1}
                         }}
            }_{=: F3_1}
        + 
            \underbrace{
            \EE\sbr{\sum_{t=K+1}^T 
                \one\cbr{A_t, C_{t-1}^c, E_{t-1}^c
                         }}
            }_{=: F3_2},
            \label{eqn:f3-decomp}
    \end{align}
    where $E_{t} := \cbr{N_{t,1} \leq \splitFthree }$. 

    Intuitively, $F3_2$ is small as when number of time steps arm 1 is pulled is large, $\hat{\mu}_{t-1,1} \leq \mu_1 - \varepsilon_2$ is unlikely to happen.  
    Indeed, using Lemma~\ref{lemma:bound-F3} with $m = \lfloor H \rfloor, n = \infty$, we immediately have $F3_2 \leq O(\frac{1}{\kl(\mu_1 - \varepsilon_2, \mu_1)})$. 

    It remains to bound $F3_1$. 
    These terms are concerned with the time steps when arm 1 is pulled at most $\lfloor H \rfloor$ times. 
    Inspired by~\cite{menard17minimax,jin2021mots}, 
    we introduce an event $\mathcal{E}:= \cbr{\forall k\in \sbr{1, \lfloor \splitFthree \rfloor}, \hat{\mu}_{(k),1} \in L_{k,1} }$ (see the definition of $L_{k,1}$ in Eq.~\eqref{eqn:l-s}) and use it to induce a split: 
    
    \begin{align*}
    F3_1
    \leq & 
            \EE\sbr{\sum_{t=K+1}^T 
                \one\cbr{A_t, C_{t-1}^c, E_{t-1}, \Ecal
                         }}
            + 
            \EE\sbr{\sum_{t=K+1}^T 
                \one\cbr{\Ecal^c
                         }}
                    \\
    \leq & 
    \EE\sbr{\sum_{t=K+1}^T 
                \one\cbr{A_t, C_{t-1}^c, E_{t-1}, \hat{\mu}_{t-1,1} \geq \mu_1 - \alpha_{N_{t-1,1}}
                         }}
            + 
            \EE\sbr{\sum_{t=K+1}^T 
                \one\cbr{\Ecal^c
                         }}
    \end{align*}
    A probability transferring argument on the first term shows that it is bounded by $O\del{ H \ln\del{ \frac T H \vee e^2 }}$;
    the second term is at most $T \PP(\Ecal^c)$, which in turn is at most $H$ using a peeling device and maximal Chernoff  inequality (Lemma~\ref{lemma:seq-estimator-deviation-bernoulli}). Combining these two, we prove Eq.~\eqref{eqn:f3-bound-refine}, which concludes the proof of Lemma~\ref{lemma:F3-upper-bound}.

    \subsubsection{Proof of Lemma~\ref{lemma:F3-upper-bound}}

    \paragraph{Additional notations.} In the proof of Lemma~\ref{lemma:F3-upper-bound}, 
    we will use the following notations: we denote ramdom variable $X_k := \mu_1 - \hat{\mu}_{(k),1}$, and denote its probability density function by $p_{X_k}(x)$. We also define function $f_k(x) := \exp(k \cdot \kl(\mu_1 - x, \mu_1 - \varepsilon_2))$.

    \begin{proof}[Proof of Lemma~\ref{lemma:F3-upper-bound}]
        Recall that we introduce $E_{t} := \cbr{N_{t,1} \leq \splitFthree }$; and according to $E_{t-1}$ we obtain the decomposition Eq.~\eqref{eqn:f3-decomp} above that $F3 \leq F3_1 + F3_2$.

    As we will prove in  Lemmas~\ref{claim:F3-1-upper-bound} and~\ref{claim:F3-2-upper-bound}, 
    $F3_1$ and $F3_2$ are bounded by $\BoundFThreeOne$ and $\BoundFThreeTwo$, respectively. The lemma follows from combining these two bounds by algebra.
    \end{proof}

    \subsubsection{\mathinhead{F3_1}{F3-1}}

    \begin{lemma} \label{claim:F3-1-upper-bound}
        \begin{align*}
            F3_1 \leq \BoundFThreeOne
        \end{align*}
    \end{lemma}
    
    \begin{proof}
    We consider three cases.
    
    \paragraph{Case 1: $\splitFthree < 1$.} 
        In this case, $E_t$ cannot happen for $t \geq K+1$ since we have pulled each arm once in the first $K$ rounds and $N_{K,a}$ for any arm should be at least $1$. 
        Therefore 
        \[
            F3_1 = 0 \leq \BoundFThreeOne 
        \]

    \paragraph{Case 2: $\splitFthree > \fr{T}{e}$.}
        $T$ is relative small compared to $\frac{H}{e}$ and since the logarithmic term $\ln\del{\del{\frac{T}{h} \wedge H} \vee e^2} \geq \ln(e^2) = 2 $ is lower bounded by $2$, we have
        \begin{align*}
            F3_1 
            \leq& T < 4 H 
            \\
            \leq& 
            \BoundFThreeOne
        \end{align*}          
    \paragraph{Case 3: $1 \leq \splitFthree \leq \fr{T}{e}$.} It suffices to prove the following two inequalities:

    \begin{equation}
        F3_1 \leq \BoundFThreeOneFirst
        \label{eqn:f31-1}
    \end{equation}
    \begin{equation}
        F3_1 \leq \BoundFThreeOneSecond
        \label{eqn:f31-2}
    \end{equation}

    \paragraph{Case 3 -- Proof of Eq.~\eqref{eqn:f31-1}.} To show Eq.~\eqref{eqn:f31-1}, we first set up some useful notations.
    Recall from Section~\ref{sec:prelims} that we denote $\tau_1(s) = \min\{t\geq 1: N_{t,1}=s\}$ and $\hat{\mu}_{(s),1}:=\hat{\mu}_{\tau_1 (s),1}$.
    For $s \in \NN$, we first define 
    interval $L_{s,1}$ as:
    \begin{equation}
    L_{s,1} := \cbr{ \mu \in [0,1]: \kl( \mu, \mu_1 ) \leq \frac{2\ln (T/s)}{s} \text{ or } \mu \geq \mu_1}.
    \label{eqn:l-s}
    \end{equation}
    For notational convenience, we also define $\alpha_s = \mu_1 - \inf L_{s,1}$ and therefore $L_{s,1} = [\mu_1 - \alpha_s, 1]$.
    
    Define $\mathcal{E}$ as $\cbr{\forall k\in \sbr{1, \lfloor \splitFthree \rfloor}, \hat{\mu}_{(k),1} \in L_{k,1} }$. 
    We denote event $\mathcal{E}_{k}:= \cbr{\hat{\mu}_{(k),1} \in L_{k,1}}$; in this notation, $\mathcal{E} = \bigcap_{k=1}^{\lfloor \splitFthree \rfloor} \mathcal{E}_{k}$, that is, $\Ecal$ happens iff all $\mathcal{E}_k$ holds simultaneously for all $k$ less or equal to $\splitFthree$. 
    Note that Lemma~\ref{lemma:seq-estimator-deviation-bernoulli} implies that $\PP(\cE^c) \leq \frac{2H}{T}$.

    Therefore,
    \begin{align}
    F3_1
    \leq & 
            \EE\sbr{\sum_{t=K+1}^T 
                \one\cbr{A_t, C_{t-1}^c, E_{t-1}, \Ecal
                         }}
            + 
            \EE\sbr{\sum_{t=K+1}^T 
                \one\cbr{\Ecal^c
                         }}
                \nonumber    \\
    \leq & 
    \EE\sbr{\sum_{t=K+1}^T 
                \one\cbr{A_t, C_{t-1}^c, E_{t-1}, \cE_{N_{t-1,1}}
                         }}
            + 
            T \PP\rbr{\Ecal^c
            }
            \nonumber \\
            \leq & \EE\sbr{\sum_{t=K+1}^T 
                \one\cbr{A_t, C_{t-1}^c, E_{t-1}, \cE_{N_{t-1,1}}
                         }}
            + 2H,
    \label{eqn:f3-fix-arm-1-E}
    \end{align}
    where in the second inequality, we use the observation that if $\cE$ happens and $N_{t-1,1} \leq H$, $\cE_{N_{t-1,1}}$ also happens; in the third inequality, we recall that $\PP(\cE^c) \leq \frac{2H}{T}$.

    We continue upper bounding Eq.~\eqref{eqn:f3-fix-arm-1-E}. For the first term in Eq.~\eqref{eqn:f3-fix-arm-1-E}, we use a ``probability transfer'' argument (Lemma~\ref{lemma:prob-transfer}) to bound the probability of pulling the suboptimal arm by the probability of pulling optimal times an inflation term.

    \begin{align}
        & \EE \sbr{ \sum_{t=K+1}^T 
            \one\cbr{ A_t, C_{t-1}^c, E_{t-1}, \mathcal{E}_{N_{t-1,1}} }
            } 
        \\
        = & \EE\sbr{
            \sum_{t=K+1}^T 
            \one\cbr{ C_{t-1}^c, \mathcal{E}_{N_{t-1,1}}, E_{t-1} }
                    \cdot
                    \EE\sbr{
                        \one\cbr{A_t} \mid \mathcal{H}_{t-1}
                        }
                    }
                \tag{Law of total expectation}
        \\
        \leq & \EE\sbr{
            \sum_{t=K+1}^T 
            \one\cbr{ C_{t-1}^c, \mathcal{E}_{N_{t-1,1}}, E_{t-1} }
                    \cdot
                    \exp(N_{t-1,1} \cdot \kl(\hat{\mu}_{t-1,1}, \hat{\mu}_{t-1,\max}))
                    \EE\sbr{ \one\cbr{I_t=1} \mid \mathcal{H}_{t-1} }
                }
                \tag{By Lemma~\ref{lemma:prob-transfer}}
        \\
        = & \EE\sbr{
            \sum_{t=K+1}^T 
            \one\cbr{ I_t=1, C_{t-1}^c, \mathcal{E}_{N_{t-1,1}}, E_{t-1} }
                    \cdot
                    \exp(N_{t-1,1} \cdot \kl(\hat{\mu}_{t-1,1}, \hat{\mu}_{t-1,\max}))
                }
                \tag{Law of total expectation }
        \end{align}

        Then we make a series of manipulations to reduce the above to bounding the expectation of some function of the random  observations drawn from the optimal arm.
        First, note that for the summation inside the expectation above, each nonzero term corresponds to a time step $t$ such that $t = \tau_1(k)$ for some unique $k \geq 2$, therefore,

        \begin{align}
        &\leq \EE\sbr{
            \sum_{k=2}^\infty 
            \one\cbr{ C_{\tau_1(k)-1}^c, \mathcal{E}_{N_{\tau_1(k)-1,1}}, E_{\tau_1(k)-1} }
                    \cdot
                    \exp(N_{\tau_1(k)-1,1} \cdot \kl(\hat{\mu}_{\tau_1(k)-1,1}, \hat{\mu}_{\tau_1(k)-1,\max}))
                }
        \\
        &\leq \EE\sbr{
            \sum_{k=2}^\infty 
            \one\cbr{C_{\tau_1(k)-1}^c, \mathcal{E}_{N_{\tau_1(k)-1,1}}, E_{\tau_1(k)-1} }  
                    \cdot
                    \exp(N_{\tau_1(k)-1,1} \cdot \kl(\hat{\mu}_{\tau_1(k)-1,1}, \mu_1 - \varepsilon_2))
                }
                \tag{when the condition $C_{\tau_1(k)-1}^c$ holds, $\hat{\mu}_{\tau_1(k)-1,1 } \leq \hat{\mu}_{\tau_1(k)-1,\max} < \mu_1 - \varepsilon_2$
                }
        \\
         &\leq \EE\sbr{
            \sum_{k=2}^\infty 
            \one\cbr{\mathcal{E}_{N_{\tau_1(k)-1,1}}, E_{\tau_1(k)-1} }  
                    \cdot
                    \exp(N_{\tau_1(k)-1,1} \cdot \kl(\hat{\mu}_{\tau_1(k)-1,1}, \mu_1 - \varepsilon_2))
                }
                \tag{Dropping $C_{\tau_1(k)-1}^c$}
        \\
        &= \EE\sbr{
            \sum_{k=2}^\infty
                    \one\cbr{ \mathcal{E}_{k-1}, k-1 \leq H  } 
                    \exp((k-1) \cdot \kl(\hat{\mu}_{(k-1),1}, \mu_1 - \varepsilon_2))
                }
                \tag{$N_{\tau_1(k)-1} = k-1$ and $\hat{\mu}_{\tau_{1}(k)-1,1} = \hat{\mu}_{\tau(k-1),1}$}
        \\
        &= \EE\sbr{
            \sum_{k=1}^\infty
                    \one\cbr{ \mathcal{E}_{k}, k \leq H } 
                    \exp(k \cdot \kl(\hat{\mu}_{(k),1}, \mu_1 - \varepsilon_2))
                }
                \tag{shift index $k$ by $1$ }
        \\
        &= \EE\sbr{
            \sum_{k=1}^{\lfloor \splitFthree \rfloor} 
                    \one\cbr{\varepsilon_2 \leq \mu_1 - \hat{\mu}_{(k),1} \leq \alpha_{k}}
                    \cdot
                    \exp(k \cdot \kl(\hat{\mu}_{(k),1}, \mu_1 - \varepsilon_2))
                }
            +
            \sum_{k= \lfloor \splitFthree \rfloor + 1}^\infty
            0
                \tag{Under the conditions $ \mathcal{E}_{k}, E_{\tau_1(k+1)-1} $, when $k \geq \lfloor \splitFthree \rfloor + 1$, $E_{\tau_1(k+1)-1}$ is always false
                }
        \\
        &= \sum_{k=1}^{\lfloor \splitFthree \rfloor} 
            \EE\sbr{
                \one\cbr{\varepsilon_2 \leq X_k \leq \alpha_{k}}
                \cdot
                f_k(X_k)
                }
                \tag{Recall $X_k = \mu_1 - \hat{\mu}_{(k),1}$ 
                }   
        \\
                \label{eqn:F_3_1_after_prob_trans}
    \end{align}

        Here the Eq.~\eqref{eqn:F_3_1_after_prob_trans} is the sum of expectation of the function $f_k(X_k)$ over a bounded range $\cbr{\varepsilon_2 \leq X_k \leq \alpha_{k}}$ from $k=1$ to $\lfloor \splitFthree \rfloor$. Continuing Eq.~\eqref{eqn:F_3_1_after_prob_trans},

    \begin{align}
        &
            \EE \sbr{ 
            \sum_{t=K+1}^T
            \one\cbr{ A_t, C_{t-1}^c, E_{t-1}, \mathcal{E}_{N_{t-1,1}} }
            } \\
        \leq &
            \sum_{k=1}^{\lfloor \splitFthree \rfloor} 
            \EE[f_k(X_k) 
            \one[\cbr{\varepsilon_2 \leq X_k \leq \alpha_{k}}] 
        \\
        =& \sum_{k=1}^{\lfloor \splitFthree \rfloor}  
            \int_{\varepsilon_2}^{\alpha_{k}} 
            f_k(x) p_{X_k}(x) \diff x
                \tag{$p_{X_k}(\cdot)$ is the p.d.f. of $X_k$}
        \\
        =& \sum_{k=1}^{\lfloor \splitFthree \rfloor} 
            \int_{\varepsilon_2}^{\alpha_{k}}
            \left(f_k(\varepsilon_2) + \int_{\varepsilon_2}^{x}f'_k(y)\dy)\right) 
            p_{X_k}(x)\diff x 
                \tag{$f_k(x) = f_k(\varepsilon_2) + \int_{\varepsilon_2}^x f'_k(y)\dy$}
        \\
        =& \underbrace{
                \sum_{k=1}^{\lfloor \splitFthree \rfloor}             
                \int_{\varepsilon_2}^{\alpha_{k}}
                \int_{\varepsilon_2}^{x}
                f'_k(y)p_{X_k}(x)\dy\diff x
            }_{A} + 
        \underbrace{
                \sum_{k=1}^{\lfloor \splitFthree \rfloor}
                \int_{\varepsilon_2}^{\alpha_{k}} 
                p_{X_k}(x)\diff x
            }_{B}
                \label{eqn:F_3_1_decompose}
    \end{align}

    We denote the first term in Eq.~\eqref{eqn:F_3_1_decompose} as $A$ and the second one as $B$. Next we are going to handle $A$ and $B$ separately. Starting from the easier one,
    \begin{align}
        B
        =& 
            \sum_{k=1}^{\lfloor \splitFthree \rfloor}
            \int_{\varepsilon_2}^{\alpha_{k}} 
            p_{X_k}(x)\diff x
        \\
        \leq& 
            \sum_{k=1}^{\lfloor \splitFthree \rfloor}
            \PP(X_k \geq \varepsilon_2)
        \\
        =& 
            \sum_{k=1}^{\lfloor \splitFthree \rfloor}
            \PP(\hat{\mu}_{(k),1} \leq \mu_1 - \varepsilon_2)
        \\
        \leq& 
            \sum_{k=1}^{\lfloor \splitFthree \rfloor}
            \exp(-k \cdot \kl(\mu_1 - \varepsilon_2, \mu_1)) 
                \tag{Applying Lemma \ref{lemma:maximal-inequality}}
        \\
        \leq& 
            \sum_{k=1}^{\infty}
            \exp\del{-k \cdot \kl(\mu_1 - \varepsilon_2, \mu_1)} 
        \\
        \leq& 
           \frac{\exp\del{-\kl(\mu_1 - \varepsilon_2, \mu_1)}}
                {1 - \exp\del{-\kl(\mu_1 - \varepsilon_2, \mu_1)}}
                \tag{Geometric sum}
        \\
        =& 
           \frac{1}
                {\exp\del{\kl(\mu_1 - \varepsilon_2, \mu_1)} - 1}
        \\
        \leq& 
           \frac{1}{\kl(\mu_1 - \varepsilon_2, \mu_1)}
                \tag{$e^x \geq x + 1$ when $x \geq 0$}
        \\
                \label{eqn:F3_1_2_upper_bound}
    \end{align}

    On the other hand,

    \begin{align}
        A
        =& 
            \sum_{k=1}^{\lfloor \splitFthree \rfloor}     
            \int_{\varepsilon_2}^{\alpha_k}
            \int_{\varepsilon_2}^{x}
            f'_k(y)p_{X_k}(x)\dy \dx
        \\
        =&
            \sum_{k=1}^{\lfloor \splitFthree \rfloor} 
            \int_{\varepsilon_2}^{\alpha_k} 
            \int_{y}^{\alpha_k} 
            f'_k(y)p_{X_k}(x)\dx \dy
                \tag{Switching the order of integral}
        \\
        =&  
            \sum_{k=1}^{\lfloor \splitFthree \rfloor} 
            \int_{\varepsilon_2}^{\alpha_k}
            k \frac{\diff \kl(\mu_1 - y, \mu_1 - \varepsilon_2)}{\dy}
            f_k(y)
            \PP(y \leq X_k \leq \alpha_k)\dy
                \tag{Calculate inner integral}
        \\
        \leq& 
            \sum_{k=1}^{\lfloor \splitFthree \rfloor} 
            \int_{\varepsilon_2}^{\alpha_{k}} 
            k \frac{\diff \kl(\mu_1 - y, \mu_1 - \varepsilon_2)}{\dy}
            f_k(y) 
            \exp{(-k \cdot \kl(\mu_1 - y, \mu_1))} \dy
                \tag{Apply Lemma \ref{lemma:maximal-inequality}}
        \\
        \leq& 
            \sum_{k=1}^{\lfloor \splitFthree \rfloor} 
            \int_{\varepsilon_2}^{\alpha_{k}} 
            k \frac{\diff \kl(\mu_1 - y, \mu_1 - \varepsilon_2)}{\dy}
            \dy
                \tag{$f_k(y)\exp{(-k \cdot \kl(\mu_1 - y, \mu_1)} \leq 1$ when $y\in \sbr{\varepsilon_2, \alpha_k}$}
        \\
        =& 
            \sum_{k=1}^{\lfloor \splitFthree \rfloor} 
            k \kl(\mu_1 - \alpha_k, \mu_1 - \varepsilon_2)
                \tag{Fundamental Theorem of Calculus}
        \\
        \leq& 
            \sum_{k=1}^{\lfloor \splitFthree \rfloor} 
            2 \ln\frac{T}{k}
                \tag{Recall definition of $\alpha_k$}
                \label{eqn:apply-def-alpha}
        \\
        \leq& 
            2 \lfloor \splitFthree \rfloor \ln T
            - 
            2 \int_{1}^{\lfloor \splitFthree \rfloor} \ln{k} \diff k
                \tag{Integral inequality Lemma \ref{lemma:integral-inequality} }
        \\
        =& 
            2 \lfloor \splitFthree \rfloor \ln T 
            - 
            2 (k\ln{k}-k)|_{1}^{\lfloor \splitFthree \rfloor}
                \tag{the anti-derivative of $\ln x$ is $x\ln x - x$}
        \\
        =& 
            2 \lfloor \splitFthree \rfloor \ln T - 
            2 \lfloor \splitFthree \rfloor \ln\del{\lfloor \splitFthree \rfloor} +
            2 \lfloor \splitFthree \rfloor
        \\
        =& 
            2 \lfloor \splitFthree \rfloor \ln\del{ \fr{T}{\lfloor \splitFthree \rfloor} } +
            2 \lfloor \splitFthree \rfloor
        \\
        \leq&
            2 \splitFthree \ln\del{ \fr{T}{\splitFthree} \vee e^2 } +
            2 \splitFthree
                \tag{$x\ln\frac{T}{x}$ is monotonically increasing when $x \in (0, \fr{T}{e})$}
        \\
                \label{eqn:F3_1_1_upper_bound_I}
    \end{align}

    The fist inequality is due to the Lemma \ref{lemma:maximal-inequality}. In the second inequality, we use the fact that  when $y\in \sbr{\varepsilon_2, \alpha_k}$, $f_k(y) \exp{(-k \cdot \kl(\mu_1 - y, \mu_1))} \leq 1$. This is because
    \[
        f_k(y)\exp{(-k \cdot \kl(\mu_1 - y, \mu_1))}
        =
        \exp(k \cdot( \kl(\mu_1 - y, \mu_1 - \varepsilon_2) - \kl(\mu_1 - y, \mu_1))) 
        \leq 1
    \]
    In the third one we use the definition of $\alpha_k$ to bound $\kl(\mu_a-\alpha_k,\mu_1-\varepsilon_2)$ by $\ln\del{\fr{T}{k}}$.
    In the fourth inequality, we apply integral inequality Lemma \ref{lemma:integral-inequality} by letting $f(x):=\ln(x)$, $a=2$ and $b= \lfloor \splitFthree \rfloor$.
    For the last inequality, we use the fact that $x \mapsto x\ln\del{\fr{T}{x}}$ is monotonically increasing when $x \in (0, \fr{T}{e})$.

    We conclude that
    $F3_1$ is bounded by
    \begin{align}
        F3_1 
        \leq&
            A + B + 2 H
        \\
        \leq& 
            2 \splitFthree \ln\del{ \fr{T}{\splitFthree} \vee e^2 }
            + 2 \splitFthree
            + \fr{1}{\kl(\mu_1-\varepsilon_2, \mu_1)}
            + 2 H
        \\
        \leq& 
            \BoundFThreeOneFirst
                \label{eqn:F3_1_upper_bound_I}
    \end{align}

    \paragraph{Case 3 -- Proof of Eq.~\eqref{eqn:f31-2}.} Applying Lemma~\ref{lemma:bound-F3} by letting $m=0$ and $n=\lfloor \splitFthree \rfloor$, we have that 
    \begin{align}
        F3_1
        =& 
            \EE\sbr{ \sum_{t=K+1}^T 
                \one\cbr{ A_t, C_{t-1}^c, E_{t-1}
                }}
        \\
        \leq& 
        \sum_{k=1}^{\lfloor \splitFthree \rfloor}
        \frac{ 2\exp(-k\kl(\mu_1 - \varepsilon_2, \mu_1)) }
        { k (\mu_1-\varepsilon_2) (1-\mu_1+\varepsilon_2) h^2(\mu_1,\varepsilon_2) } +
        \sum_{k=1}^{\lfloor \splitFthree \rfloor} \exp(-k\kl(\mu_1 - \varepsilon_2, \mu_1))
        \\
        \leq & 
        2 \splitFthree \sum_{k=1}^{\lfloor \splitFthree \rfloor}
        \frac{ 1 }
        { k } +
        \fr{1}{\kl(\mu_1-\varepsilon_2,\mu_1)}
        \\
        \leq & 
        \BoundFThreeOneSecond
    \end{align}
    where in the second inequality, we use that $\exp(-k\kl(\mu_1 - \varepsilon_2, \mu_1)) \leq 1$ and the definition of $H$, as well as the fact that $\sum_{k=1}^{\lfloor \splitFthree \rfloor} \exp(-k t) \leq \sum_{k=1}^{\infty} \exp(-k t) = \frac{e^{-t}}{1 - e^{-t}} \leq \frac{1}{t}$; in the third inequality, we use the algebraic fact that for $t > 0$, $\sum_{k=1}^{\lfloor \splitFthree \rfloor} \frac1k \leq (1 + \ln( \lfloor \splitFthree \rfloor) \leq 2( \ln( \lfloor \splitFthree \rfloor ) \vee 1) \leq 2 \ln(H \vee e^2)$.
    
    Therefore, when $\splitFthree \in (1, \frac{T}{e})$, $F3_1$ can be bounded using Eq.~\eqref{eqn:f31-1} and Eq.~\eqref{eqn:f31-2} simultaneously, concluding the proof in Case 3.

    In summary, in all three cases, $F3_1$ is upper bounded by $\BoundFThreeOne$; this concludes the proof.
    \end{proof}
    
    \subsubsection{\mathinhead{F3_2}{F3-2}}
    
    As mentioned in the proof roadmap, 
    intuitively, $F3_2$ is small, since when number of times arm 1 is pulled is large, $\hat{\mu}_{t-1,1} \leq \mu_1 - \varepsilon_2$ is unlikely to happen.
    Here, we control $F3_2$ using Lemma~\ref{lemma:bound-F3}.

    \begin{claim} \label{claim:F3-2-upper-bound}
        \begin{align*}
            F3_2 \leq \BoundFThreeTwo
        \end{align*}
    \end{claim}
    
    \begin{proof}
        $F3_2$ is the case where the number of arm pulling of optimal arm $1$ is lower bounded by $\splitFthree$.

    \begin{align}
        F3_2
        =& 
            \EE\sbr{ \sum_{t=K+1}^T 
                \one\cbr{ A_t, C_{t-1}^c, E_{t-1}^c
                }}
        \\
        \leq&
            \sum_{k=\lfloor \splitFthree \rfloor+1}^{\infty}
            \frac{ 2 \exp(-k\kl(\mu_1 - \varepsilon_2, \mu_1)) }
                 { k (\mu_1-\varepsilon_2) (1-\mu_1+\varepsilon_2) h^2(\mu_1,\varepsilon_2) } +
            \fr{1}{\kl(\mu_1-\varepsilon_2,\mu_1)}
                    \tag{Lemma \ref{lemma:bound-F3}}
        \\
        \leq&
            \sum_{k=\lfloor \splitFthree \rfloor+1}^{\infty}
            \frac{ 2 \exp(-k\kl(\mu_1 - \varepsilon_2, \mu_1)) }
                 { H (\mu_1-\varepsilon_2) (1-\mu_1+\varepsilon_2) h^2(\mu_1,\varepsilon_2) } +
            \fr{1}{\kl(\mu_1-\varepsilon_2,\mu_1)}
                    \tag{$\lfloor H \rfloor + 1 \geq H$}
        \\
        \leq&
            \sum_{k=\lfloor \splitFthree \rfloor+1}^{\infty}
                2 \exp(-k\kl(\mu_1 - \varepsilon_2, \mu_1)) +
            \fr{1}{\kl(\mu_1-\varepsilon_2,\mu_1)}
                    \tag{By the definition of $H$}
        \\
        \leq&  
            \fr {2 \exp(-(\lfloor H \rfloor + 1)\kl(\mu_1 - \varepsilon_2, \mu_1))}
                {1 - \exp(-\kl(\mu_1 - \varepsilon_2, \mu_1))} +
            \fr{1}{\kl(\mu_1-\varepsilon_2,\mu_1)}
                    \tag{Geometric sum}
        \\
        \leq&  
            \fr {2 }
                {1 - \exp(-\kl(\mu_1 - \varepsilon_2, \mu_1))} +
            \fr{1}{\kl(\mu_1-\varepsilon_2,\mu_1)}
                    \tag{$\exp(-x) \leq 1$ when $x \leq 0$}
        \\
        \leq&
            \BoundFThreeTwo
                    \label{eqn:F3_2_upper_bound}
    \end{align}
    The first inequality is true because Lemma \ref{lemma:bound-F3} by letting $m=\lfloor \splitFthree \rfloor$ and $n=\infty$, as well as the fact that for $t > 0$, $\sum_{k=\lfloor \splitFthree \rfloor+1}^{\infty} \exp(-k t) \leq \sum_{k=1}^{\infty} \exp(-k t) = \frac{e^{-t}}{1 - e^{-t}} \leq \frac{1}{t}$.
    \end{proof}

\subsubsection{Proof of Lemma~\ref{lemma:bound-F3}}
\label{sec:useful-lemma-f3}
    
    \begin{proof}[Proof of Lemma~\ref{lemma:bound-F3}]
    For any fixed $k$, recall that we denoted $f_k(x) = \exp(k\cdot \kl(\mu_1 - x, \mu_1-\varepsilon_2))$, $X_k = \mu_1 - \hat{\mu}_{\tau_1(k),1}$ and the pdf of $X_k$ as $p_{X_k}(x)$. 

    \begin{align}
    	&
    	    \EE\sbr{\sum_{t=u+1}^T 
    	        \one\cbr{A_{t}, C_{t-1}^c, S_{t-1}, T_{t-1} }
    	        }
    	\\
    	=&
            \sum_{t=u+1}^T
    	    \EE\sbr{ 
    	            \one\cbr{C_{t-1}^c, S_{t-1}, T_{t-1} }
    	            \EE\sbr{A_{t} \mid \cH_{t-1}}
    	            }
    	                \tag{Law of total expectation
                     }
    	\\
    	\leq&
            \sum_{t=u+1}^T
    	    \EE\sbr{ 
    	            \one\cbr{C_{t-1}^c, S_{t-1}, T_{t-1} } 
    	            \cdot 
    	            \exp(N_{t-1,1} \cdot \kl(\hat{\mu}_{t-1,1}, \hat{\mu}_{t-1,\max})) \EE\sbr{I_t = 1 \mid \cH_{t-1}}
    	            }
    	                \tag{Lemma \ref{lemma:prob-transfer}}
    	\\
    	\leq& 
            \sum_{t=u+1}^T 
    	    \EE\sbr{
    	            \one\cbr{C_{t-1}^c, S_{t-1}, T_{t-1} } 
    	            \cdot 
    	            \exp(N_{t-1,1} \cdot \kl(\hat{\mu}_{t-1,1}, \mu_1 - \varepsilon_2)) \EE\sbr{I_t = 1 \mid \cH_{t-1}}
    	            }
    	                \tag{when $C_{t-1}^c$ happens, $\kl(\hat{\mu}_{t-1,1}, \hat{\mu}_{t-1, \max}) \leq \kl(\hat{\mu}_{t-1,1}, \mu_1 - \varepsilon_2)$}
    	\\
        =&
            \sum_{t=u+1}^T 
    	    \EE\sbr{
    	            \one\cbr{I_t = 1, C_{t-1}^c, S_{t-1}, T_{t-1} } 
    	            \cdot 
    	            \exp(N_{t-1,1} \cdot \kl(\hat{\mu}_{t-1,1}, \mu_1-\varepsilon_2)) 
    	            }
    	                \tag{Law of total expectation}
    	\\
    	\leq& 
    	    \EE\sbr{
                \sum_{k=2}^\infty 
    	            \one\cbr{C_{\tau_1(k)-1}^c, k-1 \in (m,n]} 
    	            \cdot 
    	            \exp(N_{\tau_1(k)-1,1} \cdot \kl(\hat{\mu}_{(k-1),1}, \mu_1-\varepsilon_2)) 
    	            }
    	                \tag{for any $t$ such that $\one\cbr{I_t = 1}$ is nonzero, $t = \tau_1(k)$ for some unique $k$; $N_{\tau(k)-1,1} = k-1$, and $\hat{\mu}_{\tau(k)-1,1} = \hat{\mu}_{(k-1),1}$}
    	\\
    	=& 
    	    \EE\sbr{
                \sum_{k=2}^\infty 
    	            \one\cbr{C_{\tau_1(k)-1}^c, k \in (m+1, n+1]} 
    	            \cdot 
    	            \exp((k-1) \cdot \kl(\hat{\mu}_{(k-1),1}, \mu_1-\varepsilon_2)) 
    	            }
    	                \tag{algebra}
    	\\
    	\leq& 
    	    \EE\sbr{\sum_{k=m+2}^{n+1}
    	            \one\cbr{\mu_1\geq\mu_1-\hat{\mu}_{(k-1),1} > \varepsilon_2} 
    	            \cdot 
    	            \exp((k-1)\cdot \kl(\hat{\mu}_{(k-1),1}, \mu_1-\varepsilon_2))
    	        }
    	\\  
    	=& 
    	    \EE\sbr{\sum_{k=m+1}^{n}
    	            \one\cbr{\mu_1\geq\mu_1-\hat{\mu}_{(k),1} > \varepsilon_2} 
    	            \cdot 
    	            f_k(\mu_1 - \hat{\mu}_{(k),1})
    	        },
                        \tag{shift $k$ by $1$}
        \\
    	        \label{eqn:after-prob-transfer}
    \end{align}

    here, for the second to last inequality, we use the fact that when $S_{\tau_1(k)-1}$ happens, $k-1 > m$, and when $T_{\tau_1(k)-1}$ happens, $k-1 \leq n$.
    In the last inequality, we use the fact that when $C^c_{\tau_1(k)-1}$ happens, $\hat{\mu}_{\tau_1(k)-1,\max} < \mu_1 - \varepsilon_2$. Combining this with the fact that $\hat{\mu}_{(k-1),1} = \hat{\mu}_{\tau_1(k)-1,1} \leq \hat{\mu}_{\tau_1(k)-1,\max} $, we have $ \mu_1 - \hat{\mu}_{(k-1), 1} > \varepsilon_2$.

    Hence Eq.~\eqref{eqn:after-prob-transfer} becomes

    \begin{align*}
    	 \eqref{eqn:after-prob-transfer} 
    	 & =
    	    \EE\sbr{ 
    	        \sum_{k=m+1}^{n}
    	        \one\cbr{\mu_1 \geq X_k > \varepsilon_2} 
    	        \cdot 
    	        f_k(X_k)
    	        } 
    	 \\
    	 & = 
    	    \sum_{k=m+1}^{n}
    	    \int_{\varepsilon_2}^{\mu_1} f_k(x) p_{X_k}(x) \dx
    	 \\
    	 & = 
    	    \sum_{k=m+1}^{n}
    	        \int_{\varepsilon_2}^{\mu_1} p_{X_k}(x) \rbr{ f_k(\varepsilon_2) +
    	    \sum_{k=m+1}^{n}
    	        \int_{\varepsilon_2}^x f_k'(y) \dy } \dx
    	            \tag{$f_k(x) = f_k(\varepsilon_2) + \int_{\varepsilon_2}^x f_k'(y) \dy)$}
    	 \\
    	 & = 
    	    \sum_{k=m+1}^{n}
    	        \int_{\varepsilon_2}^{\mu_1} 
    	        p_{X_k}(x) 
    	        f_k(\varepsilon_2) \dx +
    	    \sum_{k=m+1}^{n}         
    	        \int_{\varepsilon_2}^{\mu_1}
    	        \int_{\varepsilon_2}^x 
    	            p_{X_k}(x)f_k'(y) \dy \dx
    	  \\
    	 & = 
    	    \underbrace{
    	    \sum_{k=m+1}^{n}
    	        \int_{\varepsilon_2}^{\mu_1} 
    	        p_{X_k}(x)f_k(\varepsilon_2) \dx
    	        }_{A} +
    	    \underbrace{
    	    \sum_{k=m+1}^{n}
    	        \int_{\varepsilon_2}^{\mu_1}\int_{y}^{\mu_1} 
    	        p_{X_k}(x)f_k'(y) \dx \dy
    	        }_{B}
    	            \tag{Exchange the order of integral}
    \end{align*}

    \paragraph{For $A$:}
        
    \begin{align}
        A
        &=
            \sum_{k=m+1}^{n}
	            \int_{\varepsilon_2}^{\mu_1} 
	            p_{X_k}(x)f_k(\varepsilon_2) \dx
	    \\
	    &\leq
            \sum_{k=m+1}^{n}
	            \int_{\varepsilon_2}^{\mu_1} 
	            p_{X_k}(x)f_k(\varepsilon_2) \dx
	    \\
	    &\leq
	        \sum_{k=m+1}^n
	            \exp\del{ -k \cdot \kl\del{ \mu_1 - \varepsilon_2, \mu_1 } }
	            f_k(\varepsilon_2)
	                \tag{By Lemma \ref{lemma:maximal-inequality}}
    	\\
    	&=
    	    \sum_{k=m+1}^n
	            \exp\del{ -k \cdot \kl\del{ \mu_1 - \varepsilon_2, \mu_1 } }
	                \label{eqn:A-upper-bound}
    \end{align}
    where the last equality is because $f_k(\varepsilon_2) = 1$.
    
    \paragraph{For $B$:}
    \begin{align}
        &
        B
        \\
        =& 
            \sum_{k=m+1}^{n}
            \int_{\varepsilon_2}^{\mu_1}
            \int_{\varepsilon_2}^{x}
            f'_k(y)p_{X_k}(x)\dy \dx
        \\
        =&
            \sum_{k=m+1}^{n}
            \int_{\varepsilon_2}^{\mu_1} 
            \int_{y}^{\mu_1} 
            f'_k(y)p_{X_k}(x)\dx \dy
                \tag{Switching the order of integral}
        \\
        =&  
            \sum_{k=m+1}^{n}
            \int_{\varepsilon_2}^{\mu_1}
            k \frac{\diff \kl(\mu_1 - y, \mu_1 - \varepsilon_2)}{\dy}
            f_k(y)
            \PP(y \leq x \leq \mu_1)\dy
                \tag{Calculate inner integral}
        \\
        =&
            \sum_{k=m+1}^{n}
            \int_{\varepsilon_2}^{\mu_1}
                f_k(y) \cdot 
                k \frac{\diff \kl(\mu_1-y,\mu_1-\varepsilon_2)}{\dy} \cdot 
                \exp(-k\cdot \kl(\mu_1-y,\mu_1)) \dy
                    \tag{Apply Lemma \ref{lemma:maximal-inequality}}
        \\
        =&
            \sum_{k=m+1}^{n}
            \int_{\varepsilon_2}^{\mu_1}
                \exp\del{
                    k(\kl(\mu_1-y,\mu_1-\varepsilon_2)-
                             \kl(\mu_1-y,\mu_1)) 
                    } \cdot 
                k \frac{\diff \kl(\mu_1-y,\mu_1-\varepsilon_2)}{\dy} \dy
        \\
        =& 
            \sum_{k=m+1}^{n}
            \int_{\varepsilon_2}^{\mu_1} 
                k\exp{\del{
                        -k \kl(\mu_1-\varepsilon_2,\mu_1)
                        }} \cd
        \\
        &
                \exp{\del{
                        k \del{y-\varepsilon_2}
                        \ln\del{\frac{(1 - \mu_1)(\mu_1 - \varepsilon_2)}
                                     {(1 - \mu_1 + \varepsilon_2)\mu_1}}
                    }}
            \frac{\diff \kl(\mu_1-y,\mu_1-\varepsilon_2)}{\dy} \dy 
                \tag{By Lemma~\ref{lemma:Bregman-equation} with $\phi(x) = x\ln(x)+(1-x)\ln(1-x)$, which induces $B_\phi(z,x) = \kl(z,x)$
                }
        \\
        =& 
            \sum_{k=m+1}^{n}
            \int_{\varepsilon_2}^{\mu_1}
            k\exp{\del{
                    -k \kl(\mu_1-\varepsilon_2,\mu_1)
                    -k \del{y - \varepsilon_2} h(\mu_1,\varepsilon_2)
                }}
            \frac{\diff \kl(\mu_1-y,\mu_1-\varepsilon_2)}{\dy} \dy
                \tag{Recall $\ln(\frac{(1-\mu_1+\varepsilon_2)\mu_1}{(1-\mu_1)(\mu_1-\varepsilon_2)}) = h(\mu_1,\varepsilon_2)$}
        \\
        =& 
            \sum_{k=m+1}^{n}
            \exp(-k \kl(\mu_1-\varepsilon_2,\mu_1)) 
            \cdot 
        \\
        &
            \del{
                    \underbrace{
                        \int_{\varepsilon_2}^{\mu_1} 
                            k\exp\del{- k \del{y-\varepsilon_2} h(\mu_1,\varepsilon_2)}
                            \frac{\dif  \kl(\mu_1 - y,\mu_1-\varepsilon_2)}{\dy} \dy
                    }_{ \mathrm{INT} }
                  }
            \label{eqn:B-after-change-variable}
    \end{align}

    Here, in the third to the last equation we have applied Lemma~\ref{lemma:Bregman-equation} and $\phi(x) = x\ln(x)+(1-x)\ln(1-x)$,$B_\phi(z,x)$ becomes $\kl(z,x)$. We set $z:=(\mu_1-y, 1-\mu_1+y)$, $x:=(\mu_1-\varepsilon_2, 1-\mu_1+\varepsilon_2)$ and $y:=(\mu, 1-\mu_1)$. Under this setting, according to Lemma~\ref{lemma:Bregman-equation}, we have $\kl(\mu_1-y, \mu_1-\varepsilon_2) - \kl(\mu_1-y, \mu_1) = -\kl(\mu_1-\varepsilon_2,\mu_1) + (y-\varepsilon_2)\ln\del{\frac{(1 - \mu_1)(\mu_1 - \varepsilon_2)}{(1 - \mu_1 + \varepsilon_2)\mu_1}}$.

    Next, we need to give an upper bound to the integral part $\mathrm{INT}$ carefully.
    By applying the observation below, the integral will become
    \begin{align}
    	\mathrm{INT} 
        = &  
            \int_{\varepsilon_2}^{\mu_1} 
            k\exp\del{- k \del{y-\varepsilon_2} h(\mu_1,\varepsilon_2) }
            \frac{\diff \kl(\mu_1-y,\mu_1-\varepsilon_2)}{\dy} \dy
        \\
	    = & 
	        \int_{\varepsilon_2}^{\mu_1} 
	        k\exp\del{- k \del{y-\varepsilon_2} h(\mu_1,\varepsilon_2) }
	        \frac{\diff \kl(\mu_1-y,\mu_1-\varepsilon_2)}{\diff (\mu_1-y)}
	        \frac{\diff (\mu_1-y)}{\dy} \dy
        \\
	    = & 
	        -\int_{\varepsilon_2}^{\mu_1} 
	        k\exp\del{- k \del{y-\varepsilon_2} h(\mu_1,\varepsilon_2) }
	        \left(
	            \ln{(\frac{\mu_1-y}{1-\mu_1+y})} - \ln{(\frac{\mu_1-\varepsilon_2}{1-\mu_1+\varepsilon_2})}
            \right) 
                \dy  
        \\
        = &  
            \int_{\varepsilon_2}^{\mu_1} 
                k
                \exp\del{- k \del{y-\varepsilon_2} h(\mu_1,\varepsilon_2) }
                \int_{\mu_1-y}^{\mu_1-\varepsilon_2} (\frac{1}{x}+\frac{1}{1-x}) \dx \dy
                    \tag{$\ln\frac{a}{b} = \int_a^b \frac{1}{x} \dx$}
        \\
        = &  
            \int_{0}^{\mu_1-\varepsilon_2} 
            \int_{\mu_1-x}^{\mu_1} 
                k
                \exp\del{- k \del{y-\varepsilon_2} h(\mu_1,\varepsilon_2) }
                (\frac{1}{x}+\frac{1}{1-x}) \dy \dx
    	            \tag{Change the order of integral}
	    \\
        = & 
            \int_{0}^{\mu_1-\varepsilon_2}
            \frac{\exp\del{ 
                        k \varepsilon_2 h(\mu_1,\varepsilon_2) 
                    }}
                    {h(\mu_1,\varepsilon_2)}
            \del{
                \exp\del{ -k \del{\mu_1-x} h(\mu_1,\varepsilon_2) } - 
                \exp\del{ -k \mu_1 h(\mu_1,\varepsilon_2) }
            } \cd
        \\
        &
                (\frac{1}{x}+\frac{1}{1-x}) \dx 
                    \tag{Calculate inner integral}
        \\
        = &  
            \frac{\exp\del{ 
                -k \del{\mu_1-\varepsilon_2} h(\mu_1,\varepsilon_2)
                }}
                {h(\mu_1,\varepsilon_2)} \cd
        \\
        &
                \left(
                    \underbrace{
    		            \int_{0}^{\mu_1-\varepsilon_2}
    		            \frac{\exp{(k x h(\mu_1,\varepsilon_2))}-1}{x} \dx
		            }_{\textit{part I}} 
		            + 
		            \underbrace{
		                \int_{0}^{\mu_1-\varepsilon_2}
		                \frac{\exp{(k x h(\mu_1,\varepsilon_2))}-1}{1-x} \dx
	                }_{\textit{part II}}
                \right)
                    \label{eqn:B-split}
    \end{align}

    \paragraph{Part I}
    For part I, we can bound it by
    \begin{align}
        \textit{Part I}
        & = 
        \int_{0}^{\mu_1-\varepsilon_2} 
            \frac{\exp\del{ k x h(\mu_1,\varepsilon_2) }-1}{x} \dx
        \\
        & = 
            \int_{0}^{(\mu_1-\varepsilon_2)k h(\mu_1,\varepsilon_2)}
                k h(\mu_1,\varepsilon_2)
                \frac{\exp{(y)}-1}{y} 
                \frac{1}{k h(\mu_1,\varepsilon_2)} \dy 
                    \tag{change variable $y=k x h(\mu_1,\varepsilon_2)$}
        \\
        & = 
            \int_{0}^{(\mu_1-\varepsilon_2)k h(\mu_1,\varepsilon_2)}
                \frac{\exp{(y)}-1}{y} \dy
        \\
        & \leq
            2 \frac{\exp\del{(\mu_1-\varepsilon_2) k h(\mu_1,\varepsilon_2)}}
                       {(\mu_1-\varepsilon_2) k h(\mu_1,\varepsilon_2)}
                    \tag{Using Lemma \ref{lemma:int-exp-bound} by letting $t = (\mu_1-\varepsilon_2) k h(\mu_1,\varepsilon_2)$}
        \\ 
                    \label{eqn:B-INT-PART-I}
    \end{align}

    \paragraph{Part II}
    For part II,
    \begin{align}
    	\textit{part II}
    	=& 
    	    \int_{0}^{\mu_1-\varepsilon_2}
    	        \frac{\exp\del{k x h(\mu_1,\varepsilon_2)}-1}{1-x} \dx                     
        \\
        \leq&
            \int_{0}^{\mu_1-\varepsilon_2}
                \frac{\exp\del{k x h(\mu_1,\varepsilon_2)}-1}{1-\mu_1+\varepsilon_2} \dx
                    \tag{Bound denominator by $1 - \mu_1 + \varepsilon_2$}
        \\
    	=&
    	    \frac{1}{1-\mu_1+\varepsilon_2}
    	    \del{
    	        \frac{1}{k h(\mu_1,\varepsilon_2)}
    	        \exp\del{ k x h(\mu_1,\varepsilon_2) }-x
    	        } \vert_0^{\mu_1-\varepsilon_2} 
    	            \tag{calculate integral}
        \\
        \leq&  
            \frac{\exp\del{ k \del{\mu_1-\varepsilon_2} h(\mu_1,\varepsilon_2) }}
                 {(1-\mu_1+\varepsilon_2)k h(\mu_1,\varepsilon_2)}
                    \label{eqn:B-INT-PART-II}
    \end{align}
    
    Hence from Eq.\eqref{eqn:B-INT-PART-I} and Eq.\eqref{eqn:B-INT-PART-II}, by multiplying the first factor in the Eq.~\eqref{eqn:B-split}, we can bound $\mathrm{INT}$ by
    \begin{align}
        \mathrm{INT}
        \leq&
        \frac{\exp\del{-k (\mu_1-\varepsilon_2) h(\mu_1,\varepsilon_2)}}
             {h(\mu_1,\varepsilon_2)}
            \del{ \text{part I} + \text{part II} }
        \\
        \leq&
        \frac{\exp\del{-k (\mu_1-\varepsilon_2) h(\mu_1,\varepsilon_2)}}
             {h(\mu_1,\varepsilon_2)}
            \del{ 
                2 \frac{\exp\del {(\mu_1-\varepsilon_2) k h(\mu_1,\varepsilon_2)}}{(\mu_1-\varepsilon_2) k h(\mu_1,\varepsilon_2)}
                +
                \frac{ \exp\del{ k \del{\mu_1-\varepsilon_2} h(\mu_1,\varepsilon_2) }}
                     {(1-\mu_1+\varepsilon_2)k h(\mu_1,\varepsilon_2)}
                }
        \\
        \leq &
            \frac{ 2 }{k h^2(\mu_1,\varepsilon_2)
            }
            \cdot 
            \rbr{
            \frac{1}{
            \mu_1-\varepsilon_2}
            +
            \frac{1}{1-\mu_1+\varepsilon_2
            }
            }
            = 
            \frac{ 2 }{k h^2(\mu_1,\varepsilon_2)
            }
            \cdot 
            \rbr{
            \frac{1}{
            (\mu_1-\varepsilon_2)(1-\mu_1+\varepsilon_2)}
            }
            \label{eqn:B-INT}
    \end{align}

    Therefore, we can upper bound $B$ by
    \begin{align}
        B 
        &\leq
            \sum_{k=m+1}^{n}
            \exp(-k \kl(\mu_1-\varepsilon_2,\mu_1)) 
            \cdot 
            \mathrm{INT}
        \\
        &= 
            \sum_{k=m+1}^{n}  
                \frac{ 2 \exp(-k\kl(\mu_1 - \varepsilon_2, \mu_1)) }
                     { k (\mu_1-\varepsilon_2) (1-\mu_1+\varepsilon_2) h^2(\mu_1,\varepsilon_2) } 
                    \label{eqn:B_upper_bound}
    \end{align}
    
    In a summary, by combining Eq.~\eqref{eqn:A-upper-bound} and Eq.~\eqref{eqn:B_upper_bound}, we have
    
    \begin{align*}
        &
        \sum_{t=K+1}^T
            \PP\del{ A_t, B^c_{t-1}, C^c_{t-1}, S_t, T_t }
        \\
        \leq& A + B
        \\
        \leq&
            \sum_{k=m+1}^{n}
            \del{ \frac{ 2 }
                     { k (\mu_1-\varepsilon_2) (1-\mu_1+\varepsilon_2) h^2(\mu_1,\varepsilon_2) } +
            1}  \exp(-k\kl(\mu_1 - \varepsilon_2, \mu_1)).
    \qedhere
    \end{align*}
    \end{proof}

\section{Auxiliary Lemmas}
    \subsection{Control Variance over Bounded Distribution}

    \begin{lemma} \label{lemma:control-variance}
        Let $\nu$ be a distribution supported on $[0,1]$ with mean $\mu$. Then, the variance of $\nu$ is no larger than $\dmu$.
    \end{lemma}

    \begin{proof}
        For a random variable $X \sim \nu$,

        \begin{align*}
            \Var_{X \sim \nu}[X]
            =& 
                \EE\sbr{X^2} - \del{\EE[X]}^2
            \\
            \leq&
                \EE\sbr{X} - \del{\EE[X]}^2
                    \tag{$X \geq X^2$ when $X \in \sbr{0,1}$}
            \\
            =&
                \mu - \mu^2
            \\
            =&
                \dmu
                    \tag{Recall that $\dmu = \mu(1-\mu)$}
        \end{align*}
    \end{proof}

    \subsection{Controlling the Moment Generating Function}

    \begin{lemma} \label{lemma:mgf-bounded-support}
        Let $\nu$ be a distribution with mean $\mu$ and support set $\mathcal{S} = \sbr{0,1}$. Then, moment generating function of $X \sim \nu$ is smaller than $1-\mu+e^{\lambda}$. More specifically,
        \begin{align}
            \EE_{X \sim \nu}\sbr{ e^{\lambda X} } \leq \mu e^\lambda + (1-\mu)
            \label{eqn:mgf-bounded-dist}
        \end{align}
    \end{lemma}
    \begin{proof}
        Since $e^y$ is a convex function, we apply Jensen's inequality on two point $y=0$ and $y=\lambda$ with weights $1-x$ and $x$ respectively.
        \begin{align*}
            \exp\del{(1-x) \cd 0 + x \cd \lambda }
            \leq&
            (1-x) \cd e^{0} + x \cd e^{\lambda}
            \\
            \Rightarrow
            \EE\sbr{\exp\del{\lambda x} }
            \leq&
            (1-\mu) + \mu \cd e^{\lambda}
            \qedhere
        \end{align*}
    \end{proof}
    
    \subsection{Upper Bounding the Sum of Probability of Cumulative Arm Pulling}

    \begin{lemma} \label{lemma:sum-prob-cumul-sum}
        Let $\cbr{E_t}_{t=1}^T$ be a sequence of events determined at the time step $t$ and $N := B_{t_1-1}$. $M$ is an integer such that $1\leq N \leq M \leq T$. Let $t_1, t_2$ be time indices in $\NN$ such that $t_1 < t_2$ and $F_t := \cbr{\sum_{i=1}^t \one\cbr{ E_i } < M}$ which is the event of upper bounding cumulative count 
        Then, it holds deterministically that
        \begin{align}
            \sum_{t=t_1}^{t_2} \one\cbr{ E_t, F_{t-1} }
            \leq 
            M - N
            \label{eqn:sum-prob-cumul-sum}
        \end{align}
    \end{lemma}
    
    \subsection{Useful Integral Bound}

    \begin{lemma}\label{lemma:int-exp-bound}
        Let $f(t)=\int_0^t \frac{\exp{(x)}-1}{x} \dx$. We have the inequality $f(t)\leq 2\cdot \frac{\exp{(t)}}{t}$.
    \end{lemma}
    \begin{proof}
        According to the Taylor expansion of $\exp(x)$ at $x=0$, we have
        \begin{align*}
            \frac{\exp(x) - 1}{x} 
            = \fr{\sum_{i=0}^\infty \fr{x^i}{i!} - 1}
                 {x} 
            = \sum_{i=0}^\infty \frac{x^{i}}{(i+1)!}
        \end{align*}
        Then for $f(t)$,
        \begin{align*}
            f(t) 
            =& \int_0^t \sum_{i=0}^\infty \frac{x^i}{(i+1)!} \dx
            \\
            =& 
                \sum_{i=0}^\infty \int_0^t \frac{x^i}{(i+1)!} \dx
            \\
            =& 
                \sum_{i=0}^\infty \frac{t^{i+1}}{ (i+1) \cdot (i+1)!}
            \\
            \leq& 
                2 \cd \sum_{i=0}^\infty \frac{t^{i+1}}{ (i+2)!}
            \\
            =& 
                2 \cd \sum_{i=2}^\infty \frac{t^{i-1}}{ i!}
            \\
            \leq&
                2 \cd \fr{\exp(t)}{t}
            \qedhere
        \end{align*}
    \end{proof}
    
	\begin{lemma} \label{lemma:integral-inequality}
         Given an integrable function $f(x)$ which is monotonically increasing in the range $\RR^+$. For two integers $1 \leq a < b$, we have the following inequality
        \begin{align*}
            \sum_{i=a}^b f(I) \geq f(a) + \int_{a}^{b} f(x) \dx             
        \end{align*}
    \end{lemma}
    
    \begin{proof}
        Since $f(x)$ is monotonically increasing,
        \begin{align*}
            \sum_{i=a}^b f(i)
            &= f(a) + \sum_{i=a+1}^b f(i) \cdot (i+1 - i)
            \geq
                \sum_{i=a}^b \int_{i-1}^{i} f(x) \dx
            =
                \int_{a-1}^{b} f(x) \dx
        \end{align*}
    \end{proof}
    
	\subsection{Bounding \mathinhead{H}{H}}

    \begin{lemma} \label{lemma: h-concavity-ineq}
        Given $h(\mu_1,\varepsilon_2) = \ln\del{\frac{(1-\mu_1+\varepsilon_2)\mu_1}{(1-\mu_1)(\mu_1-\varepsilon_2)}}$ with $0 < \varepsilon_2 < \mu_1$, there exists an inequality
        \[
            h(\mu_1,\varepsilon_2) \geq \frac{\varepsilon_2(1+\varepsilon_2)}{\mu_1(1-\mu_1+\varepsilon_2)}
        \]
    \end{lemma}
    
    \begin{proof}
        Using concavity of logarithm function which is for two nonnegative point $x$, $y$
        \[
        \forall x, y > 0,  \ln{y} \leq \ln{x} + \frac{y-x}{x}
        \]
        We apply this property to get the lower bound $h(\mu_a,\varepsilon_2)$ by
        \begin{align*}
            h(\mu_1, \varepsilon_2) 
            &= 
                \ln\del{\frac{(1-\mu_1+\varepsilon_2)\mu_1}{(1-\mu_1)(\mu_1-\varepsilon_2)}} 
            \\
            &= 
                \ln{ \mu_1} -
                \ln( \mu_1-\varepsilon_2 ) +
                \ln{(1-\mu_1+\varepsilon_2)} - 
                \ln{(1-\mu_1)}
            \\
            &\geq 
                \frac{\varepsilon_2}{\mu_1} +
                \frac{\varepsilon_2}{1-\mu_1+\varepsilon_2}
                    \tag{concavity property of logarithm}
            \\
            &= 
                \frac{\varepsilon_2(1+\varepsilon_2)}{\mu_1(1-\mu_1+\varepsilon_2)}
                \qedhere
        \end{align*}
    \end{proof}

    \begin{lemma} \label{lemma:H-concavity-ineq}
         Given $H:=\defSplitFthree$, $h(\mu_1,\varepsilon_2):= \ln\del{\frac{(1-\mu_1+\varepsilon_2)\mu_1}{(1-\mu_1)(\mu_1-\varepsilon_2)}}$, $0 \leq \varepsilon_2 \leq \fr{1}{2} \mu_1$ and $0 < \mu_1 \leq 1$. $H$  is bounded by the following inequality
         \[
            H \leq \frac{2\dmu_1}{\varepsilon_2^2} + \frac{2}{\varepsilon_2}
         \]
         where $\dmu_1 := (1-\mu_1)\mu_1$.
    \end{lemma}

    \begin{proof}
        According to Lemma \ref{lemma: h-concavity-ineq}, $h(\mu_1,\varepsilon_2)$ is lower bounded by
        \[
            h(\mu_1,\varepsilon_2) \geq \frac{\varepsilon_2(1+\varepsilon_2)}{\mu_1(1-\mu_1+\varepsilon_2)}
        \]
        To upper bound $H$,
        \begin{align*}
            H 
            &\leq 
                \frac{1}
                     {(\mu_1-\varepsilon_2)(1-\mu_1+\varepsilon_2) 
                        \del{ \frac{\varepsilon_2(1+\varepsilon_2)}{\mu_1(1-\mu_1+\varepsilon_2)}}^2
                     }
            \\
            &=
                \frac{ \mu_1^2 (1-\mu_1+\varepsilon_2) }
                     { (\mu_1-\varepsilon_2) \varepsilon_2^2 (1+\varepsilon_2)^2 }
            \\
            &=
                \frac{\mu_1}{\mu_1-\varepsilon_2} \cdot
                \del{ \frac{1}{1+\varepsilon_2} }^2 \cdot
                \frac{(1-\mu_1+\varepsilon_2) \mu_1}{\varepsilon_2^2}
            \\
            &\leq
                2 \cdot
                1 \cdot
                \frac{ (1-\mu_1+\varepsilon_2)\mu_1 }{ \varepsilon_2^2 }
                        \tag{$0 \leq \varepsilon_2 \leq \fr{\mu_1}{2}$}
            \\
            &\leq
                \frac{2\dot\mu}{\varepsilon_2^2} +
                \frac{2}{\varepsilon_2}
        \qedhere
        \end{align*}
    \end{proof}

    \subsection{Probability Transferring Inequality}
    \begin{lemma} \label{lemma:prob-transfer}
        Let $\cH_{t-1}$ be the $\sigma$-field generated by historical trajectory up to time (and including) $t-1$, which is defined as $\sigma\rbr{ \{I_i, r_i\}_{i=1}^{t-1} }$ ($I_i$ is the arm pulling at the time round $i$ and $r_i$ is its return reward).
        Given the algorithm~\ref{alg:KL-MS}, the probability of pulling a sub-optimal arm $a$ has the following relationship.
        \[
            \PP(I_t = a|\cH_{t-1}) 
            \leq 
            \exp(-N_{t-1,a} \kl(\hat{\mu}_{t-1,a},\hat{\mu}_{t-1,\max}))
        \]
        Also,
        \[
            \PP(I_t = a|\cH_{t-1}) 
            \leq 
            \exp(N_{t-1,1} \kl(\hat{\mu}_{t-1,1},\hat{\mu}_{t-1,\max}) ) \PP(I_t = 1 \mid \cH_{t-1})
        \]
    \end{lemma}
    
    \begin{proof}
        For the first item, recall the definition of $p_{t,a} = \exp\del{-N_{t-1,a} \kl(\hat{\mu}_{t-1,a},\hat{\mu}_{t-1,\max})} / M_t$.
        \begin{align*}
            & \PP(I_t = a|\cH_{t-1}) = p_{t,a}
            \\
            =&
                \frac{\exp(-N_{t-1,a} \kl(\hat{\mu}_{t-1,a},\hat{\mu}_{t-1,\max}))}{M_t}
            \\
            \leq& 
                \exp(-N_{t-1,a} \kl(\hat{\mu}_{t-1,a},\hat{\mu}_{t-1,\max}))
            \\
        \end{align*}
        Since $M_t \geq 1$ from the fact that $KL(\hat{\mu}_{t-1,a},\hat{\mu}_{t-1,\max}) = 0$ when $a = \arg\max_{i\in[K]}\hat{\mu}_{t-1,i}$, recall the definition of$M_t$, we have $M_t \geq 1$.
        
        For the second item, recall the algorithm setting, there exists the following relationship
        \begin{align*}
            & \PP(I_t = a|\cH_{t-1})
            \\
            =&
                \frac{\exp(-N_{t-1,a} \kl(\hat{\mu}_{t-1,a},\hat{\mu}_{t-1,\max})}{M_t}
            \\
            =& 
                \frac{\exp(-N_{t-1,a} \kl(\hat{\mu}_{t-1,a},\hat{\mu}_{t-1,\max}) )}
                     {\exp(-N_{t-1,1} \kl(\hat{\mu}_{t-1,1},\hat{\mu}_{t-1,\max}) )}
                \cdot
                \frac{\exp(-N_{t-1,1} \kl(\hat{\mu}_{t-1,1},\hat{\mu}_{t-1,\max}) )}{M_t}
            \\
            =&
                \frac{\exp(-N_{t-1,a} \kl(\hat{\mu}_{t-1,a},\hat{\mu}_{t-1,\max}) )}
                     {\exp(-N_{t-1,1} \kl(\hat{\mu}_{t-1,1},\hat{\mu}_{t-1,\max}) )}
                \cdot
                \PP(I_t = 1 \mid \cH_{t-1})
            \\
            \leq&
                \frac{ 1 }
                     {\exp(-N_{t-1,1} \kl(\hat{\mu}_{t-1,1},\hat{\mu}_{t-1,\max}) )}
                \cdot
                \PP(I_t = 1 \mid \cH_{t-1})
            \\
            =&
                \exp(N_{t-1,1} \kl(\hat{\mu}_{t-1,1},\hat{\mu}_{t-1,\max}) )
                \PP(I_t = 1 \mid \cH_{t-1})
        \end{align*}
        The first inequality is due to $ \kl(\hat{\mu}_{t-1,a}, \hat{\mu}_{t-1,\max}) \geq 0 $ and $ \exp(-N_{t-1,a} \kl(\hat{\mu}_{t-1,a}, \hat{\mu}_{t-1,\max})) \leq 1$. 
    \end{proof}
    
    \subsection{Bounding the Deviation of Running Averages from the Population Mean}
    
    \begin{lemma} \label{lemma:seq-estimator-deviation-bernoulli} 
        The distribution of random variable $X$ is $\nu_i$ which is a distribution with bounded support $\sbr{0,1}$ and mean $\mu$. Suppose that there is a sequence of sample $\cbr{X_i}_{i=1}^k$ draw i.i.d. from $\nu_i$. Denote $\sum_{i=1}^s X_{i}/s$ as $\hat{\mu}_{s}$. 

        Let $\epsilon > 0$, assume $T \geq k \geq 1$. 
        Then,
        \begin{align*}
            \mathbb{P}\del{
                \exists 1 \leq s \leq k: 
                    \kl(\hat{\mu}_s, \mu) \geq \frac{2\ln(T/s)}{s}
                } 
            \leq 
            \frac{2k}{T}
        \end{align*}
    \end{lemma}
    
    \begin{proof}
        We apply the peeling device $\frac{k}{2^{n+1}} < s \leq \frac{k}{2^n}$ to upper bound the upper left term
        \begin{align}
            & \mathbb{P}\del{
                \exists s \leq k: 
                \kl(\hat{\mu}_s, \mu) \geq \frac{2\ln( T/s )}{s}
                }
            \\
            \leq& \sum_{n=0}^\infty
                \mathbb{P}\del{
                \exists s:
                    s \in [k] \cap (\frac{k}{2^{n+1}}, \frac{k}{2^n}],
                    \kl(\hat{\mu}_s, \mu) \geq \frac{2\ln( T/s )}{s}
                }
            \\
            \leq& \sum_{n=0}^\infty
                \mathbb{P}\del{
                \exists s: 
                    s \in [k] \cap (\frac{k}{2^{n+1}}, \frac{k}{2^n}],
                    \kl(\hat{\mu}_s, \mu) \geq \frac{2^{n+1}\ln( 2^{n} T/k )}{k}
                }
                        \tag{Relax $s$ to the maximum in each subcase}
            \\
                    \label{eqn:before-subcases}
            \end{align}
            
            For $n \geq \lfloor \log_2 k\rfloor + 1$, $[k] \cap (\frac{k}{2^{n+1}}, \frac{k}{2^n}] = \emptyset$ , which means that the event
            $\cbr{\exists s: 
                    s \in [k] \cap (\frac{k}{2^{n+1}}, \frac{k}{2^n}],
                    \kl(\hat{\mu}_s, \mu) \geq \frac{2^{n+1}\ln( 2^{n} T/k )}{k}}$
            cannot happen and its probability is 0 trivially. Therefore,

            \begin{align*}
            \eqref{eqn:before-subcases} 
            =&
            \sum_{n=0}^{\lfloor \log_2 k\rfloor}
                \mathbb{P}\del{
                \exists s: 
                    s \in [k] \cap (\frac{k}{2^{n+1}}, \frac{k}{2^n}],
                    \kl(\hat{\mu}_s, \mu) \geq \frac{2^{n+1}\ln\del{ 2^{n} T/k }}{k}
                }
            +
            \sum_{n=\lfloor \log_2 k\rfloor+1}^\infty
                0
            \\
            \leq&
            \sum_{n=0}^{\lfloor \log_2 k\rfloor}
                \mathbb{P}\del{
                    \exists s \geq \fr{k}{2^{n+1}},
                    \kl(\hat{\mu}_s, \mu) \geq \frac{2^{n+1}\ln\del{ 2^{n} T/k }}{k}
                }
            \\
            =&
            \sum_{n=0}^{\lfloor \log_2 k\rfloor}
                \mathbb{P}\del{
                    \exists s \geq \lceil \fr{k}{2^{n+1}} \rceil,
                    \kl(\hat{\mu}_s, \mu) \geq \frac{2^{n+1}\ln\del{ 2^{n} T/k }}{k}
                }
            \\
            \leq& \sum_{n=0}^{\lfloor \log_2 k\rfloor}
                \exp\del{ 
                    - \lceil \frac{k}{2^{n+1}} \rceil
                    \cdot 
                    \fr{2^{n+1} \ln\del{2^{n} T / k}}{k}
                }
                    \tag{Maximal Inequality  Lemma \ref{lemma:maximal-inequality}}
            \\
            =& \sum_{n=0}^\infty
                \exp\del{ -\ln\frac{2^{n} T}{k} }
            \\
            =& \sum_{n=0}^\infty
                \frac{k}{2^{n} T}
            \\
            =& \frac{2k}{T}
        \end{align*}

        The first inequality relies on the Lemma~\ref{lemma:maximal-inequality}, for each choice of $n$, we set $y$ to be $\frac{2^{n+1}\ln\del{ 2^{n+1} T/k }}{k}$.
    \end{proof}

    The following lemma is standard in the literature, see e.g.~\cite{menard17minimax}; we include a proof for completeness.
    
    \begin{lemma} \label{lemma:maximal-inequality}
         Given a natural number $N$ in $\NN^+$, and a sequence of R.V.s $\cbr{X_i}_{i=1}^\infty$ is drawn from a distribution $\nu$ with bounded support $[0,1]$ and mean $\mu$. Let $\hat{\mu}_n = \frac{1}{n}\sum_{i=1}^n X_i, n \in \NN$, which is the empirical mean of the first $n$ samples.
         
         Then, for $y \geq 0$ 
        \begin{align}
            \PP(\exists n \geq N, \kl(\hat{\mu}_n, \mu) \geq y, \hat{\mu}_n < \mu ) 
            \leq& 
            \exp(-N y) \label{eqn:maximal-inequality-lower}
            \\
            \PP(\exists n \geq N, \kl(\hat{\mu}_n, \mu) \geq y, \hat{\mu}_n > \mu ) 
            \leq& 
            \exp(-N y) \label{eqn:maximal-inequality-upper}
        \end{align}
        Consequently, the following inequalities are also true:
        \begin{align}
            \PP(\hat{\mu}_N < \mu - \varepsilon ) 
            \leq 
            \exp(-N\cdot \kl(\mu - \varepsilon, \mu))
            \label{eqn:chernoff-lower-tail-bound}
            \\
            \PP(\hat{\mu}_N > \mu + \varepsilon ) 
            \leq 
            \exp(-N\cdot \kl(\mu + \varepsilon, \mu))
            \label{eqn:chernoff-upper-tail-bound}
        \end{align}

    \end{lemma}
    
    \begin{proof}
    First, we prove a useful fact that for any $\lambda \in \RR$, $S_n(\lambda) := \exp\del{ n \hat{\mu}_n  \lambda - n g_\mu(\lambda) }$ (abbrev. $S_n$) is a super-martingale sequence when $n \in \NN^+$ and $n \geq N$, where $g_\mu(\lambda) := \ln\del{1-\mu+\mu e^\lambda}$ is the log moment generating function of $\Ber(\mu)$. 

    Then, we have the following inequalities to finish the proof of the above fact:

        \begin{align*}
            \EE\sbr{ S_{n+1} \mid S_n, \dots, S_1}
            =&
            \EE\sbr{ S_{n+1} \mid S_n}
            \\
            =&
            \EE\sbr{ \exp\del{ (n+1) \hat{\mu}_{n+1}  \lambda - (n+1) g_\mu(\lambda) } \mid S_n}
            \\
            =&
            \EE\sbr{ S_n \cd \exp\del{  X_{n+1}  \lambda - g_\mu(\lambda) } \mid S_n}
            \\
            =&
            S_n \cd \fr { \EE\sbr{ \exp\del{  X_{n+1}  \lambda }}}
                        { \exp\del{g_\mu(\lambda)}}
            \\
            \leq&
            S_n \cd \fr{ 1-\mu+\mu e^{\lambda} }{ 1-\mu+\mu e^{\lambda} }
            = S_n
                \tag{Lemma \ref{lemma:mgf-bounded-support}}
        \end{align*}
        here, for the first equality, note that $S_{n+1}$, which is determined by $\hat{\mu}_{n+1}$ and $\hat{\mu}_{n+1}$ is conditionally independent of the trajectory up to time step $n-1$ given the condition $S_n$. 
        The second and third equalities are due to the definitions of $S_{n+1}$ and $S_n$ respectively. 
        In the first inequality, we apply Lemma \ref{lemma:mgf-bounded-support} to upper bound the numerator $\EE[\exp\del{X_{n+1} \lambda}]$ by $1-\mu+\mu e^\lambda$.
        
        We now prove Eq.~\eqref{eqn:maximal-inequality-lower} and Eq.~\eqref{eqn:maximal-inequality-upper} respectively.
        
        For Eq.~\eqref{eqn:maximal-inequality-lower}, we consider two cases: 

        \paragraph{Case 1: $y > \kl(0,\mu) = \ln\frac{1}{1-\mu}$.} In this case, event
        $\kl(\hat{\mu}_n, \mu) \geq y$ can never happen. Therefore, $\mathrm{LHS} = 0 \leq \mathrm{RHS}$.

        \paragraph{Case 2: $y \leq \kl(0,\mu)$.}
         In this case, there exists a unique $z_0 \in [0, \mu)$ such that $\kl(z_0, \mu) = y$. 
        We denote $\lambda_0 := \ln\frac{z_0(1-\mu)}{(1-z_0)\mu} < 0$.

        Observe that 
        \[
            y = \kl(z_0, \mu) = z_0 \lambda_0 - g_\mu(\lambda_0)
        \]
        
        Therefore, $\mathrm{LHS}$ of Eq.~\eqref{eqn:maximal-inequality-lower} is equal to
        \begin{align}
        &
        \PP(\exists n \geq N, \kl(\hat{\mu}_n, \mu) \geq y, \hat{\mu}_n < \mu ) 
        \\
        =& 
            \PP\del{\exists n \geq N, \hat{\mu}_n \leq z_0}
        \\
        \leq& 
            \PP\del{\exists n \geq N, n \hat{\mu}_n \lambda_0 - n g_{\mu}(\lambda_0) \geq n z_0 \lambda_0 - n g_{\mu}(\lambda_0) }
                \tag{$\lambda_0 <0$ and $\hat{\mu}_n \leq z_0$}
        \\
        \leq&
            \PP\del{\exists n \geq N, 
                n \hat{\mu}_n \lambda_0 - n g_{\mu}(\lambda_0)  \geq n y }
                    \tag{By the definition of $z_0$}
        \\
        \leq&
            \PP\del{ \exists n \geq N, 
                \exp(n \lambda_0 \hat{\mu}_n - n g_\mu(z_0))  \geq \exp(N y) }
        \label{eqn:s-n-event-prob}
        \\
        = & 
        \PP(\exists n \geq N, S_n(\lambda_0) \geq \exp(Ny)) \\
        \leq &
        \frac{\EE[S_N(\lambda_0)]}{\exp(Ny)}
        \leq 
        \tag{Ville's maximal inequality}
        \exp(-Ny)
        \end{align}

      For Eq.~\eqref{eqn:maximal-inequality-upper}, we consider two cases: 

        \paragraph{Case 1: $y > \kl(1,\mu) = \ln\frac{1}{\mu}$.} In this case, event
        $\kl(\hat{\mu}_n, \mu) \geq y$ can never happen. Therefore, $\mathrm{LHS} = 0 \leq \mathrm{RHS}$.

        \paragraph{Case 2: $y \leq \kl(1,\mu)$.}
        In this case, there exists a unique $z_1 \in (\mu, 1]$ such that $\kl(z_1,\mu)=y$. Let $\lambda_1 := \ln\del{\frac{z_1(1-\mu)}{(1-z_1)\mu}} > 0$. 
        Observe that 
        \[
            y = \kl(z_1, \mu) = z_1 \lambda_1 - g_\mu(\lambda_1)
        \]
        Then we have
        \begin{align*}
            &
            \PP(\exists n \geq N, \kl(\hat{\mu}_n, \mu) \geq y, \hat{\mu}_n > \mu )
            \\
            =& 
                \PP\del{\exists n \geq N, \hat{\mu}_n \geq z_1}
            \\
            \leq& 
                \PP\del{\exists n \geq N, n \hat{\mu}_n \lambda_1 - n g_{\mu}(\lambda_1) \geq n z_1 \lambda_1 - n g_{\mu}(\lambda_1) }
                    \tag{$\lambda_1 > 0$ and $\hat{\mu}_n \geq z_1$}
            \\
            \leq&
                \PP\del{\exists n \geq N, 
                    n \hat{\mu}_n \lambda_1 - n g_{\mu}(\lambda_1)  \geq n y }
                        \tag{By the definition of $z_1$}
            \\
            \leq&
                \PP\del{ \exists n \geq N, 
                    \exp(n \lambda_1 \hat{\mu}_n - n g_\mu(\lambda_1))  \geq \exp(N y) }
        \\
        = & 
        \PP(\exists n \geq N, S_n(\lambda_1) \geq \exp(Ny)) \\
        \leq &
        \frac{\EE[S_N(\lambda_1)]}{\exp(Ny)}
        \leq 
        \tag{Ville's maximal inequality}
        \exp(-Ny)
        \end{align*}

    where the first inequality is due to the fact that $\lambda_1 > 0$ and the condition $\hat\mu_n \geq z_1$ which is equivalent to the event $\cbr{ \kl\del{\hat\mu_n, \mu} \geq \kl\del{z_1,\mu} ,\hat{\mu}_n > \mu }$.

    Finally we derive Eq.~\eqref{eqn:chernoff-lower-tail-bound} and~\eqref{eqn:chernoff-upper-tail-bound} from Eq.~\eqref{eqn:maximal-inequality-lower} and Eq.~\eqref{eqn:maximal-inequality-upper} respectively.
    
    For Eq.~\eqref{eqn:chernoff-lower-tail-bound}, by letting $y= \kl(\mu-\varepsilon, \mu)$ we have that
    \begin{align*}
        \PP\del{ \hat{\mu}_N < \mu - \varepsilon }  
        =&
        \PP\del{ \kl(\hat{\mu}_N, \mu) > \kl(\mu - \varepsilon, \mu), 
                 \hat{\mu}_n < \mu }
        =
        \PP\del{ \kl(\hat{\mu}_N, \mu) > y,
                 \hat{\mu}_n < \mu}
        \\
        \leq&
        \PP\del{ \exists n \geq N, \kl(\hat{\mu}_n, \mu) \geq y,
                 \hat{\mu}_n < \mu) }
        \\
        \leq& \exp(-N y) 
        =
        \exp(-N \cdot \kl(\mu-\varepsilon, \mu))
    \end{align*}
    
    For Eq.~\eqref{eqn:chernoff-upper-tail-bound}, by letting $y= \kl(\mu+\varepsilon, \mu)$ we have that
    \begin{align*}
        \PP\del{ \hat{\mu}_N > \mu + \varepsilon }  
        =&
        \PP\del{ \kl(\hat{\mu}_N, \mu) > \kl(\mu + \varepsilon, \mu), 
                 \hat{\mu}_n > \mu }
        =
        \PP\del{ \kl(\hat{\mu}_N, \mu) > y,
                 \hat{\mu}_n > \mu}
        \\
        \leq&
        \PP\del{ \exists n \geq N, \kl(\hat{\mu}_n, \mu) \geq y,
                 \hat{\mu}_n > \mu) }
        \\
        \leq& \exp(-N y) 
        =
        \exp(-N \cdot \kl(\mu+\varepsilon, \mu)
    \qedhere
    \end{align*}
    \end{proof}

    \subsection{Lower Bound of KL}
    \begin{lemma} \label{lemma:KL-lower-bound}
        Given a KL-divergence $\kl(\mu_i, \mu_j)$ between two Bernoulli distribution $\nu(\mu_i)$ and $\nu(\mu_j)$ where $\mu_i,\mu_j \in \sbr{0,1}$. Denote $\dmu_i := \mu_i(1-\mu_i)$, $\dmu_j := \mu_j(1-\mu_j)$ and $\Delta := \abs{\mu_j - \mu_i}$, we have a lower bound to $\kl(\mu_i, \mu_j)$. 
        
        \[
            \kl(\mu_i,\mu_j) 
            \geq
            \fr{1}{2}\del{ 
                \fr{\Delta^2}{\dmu_i+\Delta} 
                \vee 
                \fr{\Delta^2}{\dmu_j+\Delta} } 
        \]
    \end{lemma}

    \begin{proof}
        Let $V(x):=x(1-x)$ to the variance of a Bernoulli distribution with mean $x$. According to \cite{Harremo_s_2017}, the KL divergence $\kl(\mu_i, \mu_k)$ can be computed from the following formula
            \[
                \kl(\mu_i, \mu_j) = \int_{\mu_i}^{\mu_j} \frac{x-\mu_i}{V(x)} \dx.
            \]
        Since by 1-Lipshizness of $z \mapsto z (1-z)$ we have $V(x):=x(1-x) \leq \dmu_j + 1 \cd \abs{x - \mu_j} \leq \dmu_j + \Delta$ and $V(x) \leq \dmu_i + \Delta$.
        Then we have,

        \begin{align*}
            \kl(\mu_i, \mu_j)
            =&
            \int_{\mu_i}^{\mu_j} \frac{x-\mu_i}{V(x)} \dx
            \geq
            \int_{\mu_i}^{\mu_j}
            \del{ \frac{x-\mu_i}{\dmu_i + \Delta} \vee \frac{x-\mu_i}{\dmu_j + \Delta} }\dx
            =
            \fr12 \cd \del{
            \frac{\Delta^2}{\dmu_i + \Delta} \vee
            \frac{\Delta^2}{\dmu_j + \Delta}
            }.
        \end{align*}
    \end{proof}

    \subsection{Algebraic Lemmas}

    \begin{lemma} \label{lemma:monoticity}
    Let $q \geq p > 0$ and $b 
    > 0$, and define $f_{p,q}(x) := \frac{\ln( b x^p \vee e^q )}{x}$. Then $f(x)$ is monotonically decreasing in $\RR_+$.
    Specifically, 
    both $f_{1,2}(x) := \frac{\ln( b x \vee e^2 )}{x}$ and $f_{2,2}(x) := \frac{\ln( b x^2 \vee e^2 )}{x}$ are monotonically decreasing.
    \label{lem:ln-x-x-decr}
    \end{lemma}
    \begin{proof}
    Note that 
    \[
    f_{p,q}(x) 
    =
    \begin{cases}
    \frac{q}{x} & b x^p \leq e^q \\
    \frac{\ln(b x^p)}{x} & b x^p > e^q
    \end{cases}
    \]
    \begin{itemize}
    \item When $x \in (0, \frac{e^\frac{q}{p}}{b^{\frac1p}})$, $bx^p < e^q$. In this case, $f_{p,q}$ is monotonically decreasing as $f_{p,q}(x)$ is inverse proportional to $x$.
    \item When $x \in [\frac{e^\frac{q}{p}}{b^{\frac1p}}, +\infty)$, $bx^p \geq e^q$. In this case, 
    \[
    f_{p,q}'(x) = \frac{ p - \ln(b x^p)  }{x^2} \leq \frac{p - q}{x^2} \leq 0,
    \]
    which implies that $f_{p,q}$ is also monotonically decreasing in this region. 
    \qedhere
    \end{itemize}
    \end{proof}

    \begin{lemma}
    \label{lem:log-c-move}
    For $C \geq 1$ and $a > 0$,
    \[
    \ln(C a \vee e^2) 
    \leq 
    C \ln(a \vee e^2)
    \]
    \end{lemma}
    \begin{proof}
    From Lemma~\ref{lem:ln-x-x-decr}, 
    $f_{1,2}(x) := \frac{\ln(bx \vee e^2)}{x}$ is monotonically decreasing. Therefore, we have 
    \[
    \frac{\ln(Ca \vee e^2) }{Ca}
    \leq 
    \frac{\ln(a \vee e^2) }{a},
    \]
    this yields the lemma.
    \end{proof}

    \subsection{Bregman divergence identity}
    \begin{lemma}[Lemma 6.6 in \citet{orabona2019modern}]
    \label{lemma:Bregman-equation}
        Let $B_{\phi}$ the Bregman divergence w.r.t. $\phi: X \to \RR$. Then, for any three points $x, y \in \mathrm{interior}(X)$ and $z \in X$, the following equality holds:
        \[
            B_{\phi} (z,x) + B_{\phi} (x, y) - B_{\phi}(z, y) = \inner{\nabla{\phi} (y) - \nabla{\phi} (x)}{z - x},
        \]
        where $B_{\phi}(z,x):= \phi(z) - \phi(x) - \inner{\nabla \phi(x)}{z-x}$.
    \end{lemma}

\section{Refined worst-case guarantees for existing algorithms}
\label{sec:refined}
    
\subsection{KL-UCB's refined regret guarantee}
\label{sec:kl-ucb-refined}

In this section, we show that KL-UCB~\cite{cappe2013kullback} also can enjoy a worst-case regret bound of the form $\sqrt{ \dmu_1 T K \ln T }$ in the bandits with $[0,1]$ bounded reward setting. We first recall the KL-UCB algorithm, Algorithm~\ref{alg:kl-ucb}, and we take the version of~\cite[][Section 10.2]{lattimore20bandit}.

\begin{algorithm}[t]
\begin{algorithmic}[1]
\STATE \textbf{Input:} $K\geq 2$
\FOR{$t=1,2,\cdots,n$}
    \IF{$t\leq K$} 
        \STATE Pull the arm $I_t=t$ and observe reward $y_t \sim \nu_i$.
    \ELSE 
        \STATE For every $a \in [K]$, compute 
        \[ 
        \UCB_t(a) = \max\cbr{\mu \in [0,1]: \kl(\hat{\mu}_{t-1,a}, \mu) \leq \frac{\ln f(t)}{ N_{t-1,a} } },
        \]
        where $f(t) = 1 + t\ln^2 t$.
        \STATE Choose arm $I_t = \argmax_{a \in [K]} \UCB_t(a)$
        \STATE Receive reward $y_t \sim \nu_{I_t}$
    \ENDIF
\ENDFOR
\end{algorithmic}
\caption{The KL-UCB algorithm (taken from~\citet[][Section 10.2]{lattimore20bandit})}
\label{alg:kl-ucb}
\end{algorithm}

The following theorem is a refinement of the guarantee of KL-UCB in~\cite[][Theorem 10.6]{lattimore20bandit}. 

\begin{theorem}[KL-UCB: refined guarantee]
\label{thm:kl-ucb-refined}
For any $K$-arm bandit problem with reward distributions supported on $[0,1]$, KL-UCB (Algorithm~\ref{alg:kl-ucb}) has regret bounded as follows. For any $\Delta > 0$ and $c \in (0,\frac14]$, 
\begin{equation}
\Reg(T)
\leq 
T\Delta + \sum_{a: \Delta_a > \Delta} \frac{\Delta_a \ln(1+T\ln^2T)}{\KL(\mu_a + c\Delta_a, \mu_1 - c\Delta_a)}
+ 
O\rbr{ \sum_{a: \Delta_a > \Delta} \frac{\dmu_1 + \Delta_a}{c^2 \Delta_a} }.
\label{eqn:kl-ucb-refined}
\end{equation}
and consequently, 
\begin{equation}
\Reg(T)
\leq 
O\rbr{ \sqrt{ \dmu_1 T K \ln T } + K \ln T}.
\label{eqn:kl-ucb-worstcase}
\end{equation}
\end{theorem}
\begin{proof}[Proof sketch]
To show Eq.~\eqref{eqn:kl-ucb-refined}, 
fix any suboptimal arm $a$; it suffices to show that 
\begin{equation}
\EE\sbr{ N_{T,a} }
\leq
\frac{\Delta_a \ln(1+T\ln^2T)}{\KL(\mu_a + c\Delta_a, \mu_1 - c\Delta_a)}
+ 
O\rbr{ \sum_{a: \Delta_a > \Delta} \frac{\dmu_1 + \Delta_a}{c^2 \Delta_a} }.
\label{eqn:kl-ucb-arm-pull-coarse}
\end{equation}

To this end, 
following~\citet[][proof of Theorem 10.6]{lattimore20bandit}, let $\varepsilon_1, \varepsilon_2 > 0$ be such that $\varepsilon_1 + \varepsilon_2 < \Delta_a$.

Define 
\[
\tau = \min\cbr{t: \max_{s \in \cbr{1,\ldots,T}} \kl( \hat{\mu}_{1,(s)}, \mu_1 - \varepsilon_2 ) - \frac{\ln f(t)}{s} \leq 0 },
\]
and
\[
\kappa = \sum_{s=1}^T \one\cbr{ \kl( \hat{\mu}_{a,(s)}, \mu_1 - \varepsilon_2 ) \leq \frac{\ln f(T)}{s} }.
\]
A close examination of~\citet[][proof of Lemma 10.7]{lattimore20bandit} reveals that a stronger bound on $\EE\sbr{\tau}$ holds, i.e.,
\[
\EE\sbr{\tau} \leq \frac{2}{\kl(\mu_1-\varepsilon_2, \mu_1) } 
\]
and similarly, a close examination of ~\citet[][proof of Lemma 10.8]{lattimore20bandit} reveals that a stronger bound on $\EE\sbr{\kappa}$ holds,
\[
\EE\sbr{\kappa} \leq \frac{\ln f(T)}{\kl(\mu_a + \varepsilon_1, \mu_1 - \varepsilon_2)} + \frac{1}{\kl(\mu_a + \varepsilon_1, \mu_a)}
\]
Therefore, by~\citet[][proof of Theorem 10.6]{lattimore20bandit}, we have 
\begin{equation}
\EE\sbr{N_{T,a}} 
\leq 
\EE\sbr{\tau} + \EE\sbr{\kappa}
\leq 
\frac{\ln f(T)}{\kl(\mu_a + \varepsilon_1, \mu_1 - \varepsilon_2)} + \frac{1}{\kl(\mu_a + \varepsilon_1, \mu_a)}
+
\frac{2}{\kl(\mu_1-\varepsilon_2, \mu_1) }
\label{eqn:kl-ucb-arm-pull-bound}
\end{equation}

We now set $\varepsilon_1 = \varepsilon_2 = c\Delta_a$. 
Observe that by Lemma~\ref{lemma:KL-lower-bound},
\[
\frac{1}{\kl(\mu_a + \varepsilon_1, \mu_a)}
\lesssim 
\frac{\dmu_a + \varepsilon_1}{ 
 \varepsilon_1^2}
\lesssim
\frac{\dmu_1 + \Delta_a}{ 
 c^2 \Delta_a^2},
\]
and 
\[
\frac{2}{\kl(\mu_1-\varepsilon_2, \mu_1) }
\lesssim 
\frac{\dmu_1 + \varepsilon_2}{ 
 \varepsilon_2^2}
\lesssim 
\frac{\dmu_1 + \Delta_a}{ 
 c^2 \Delta_a^2}
\]
Plugging these two inequalities into Eq.~\eqref{eqn:kl-ucb-arm-pull-bound} yields Eq.~\eqref{eqn:kl-ucb-arm-pull-coarse}.

As for Eq.~\eqref{eqn:kl-ucb-worstcase}, we note that $\frac{1}{\kl(\mu_a + c\Delta_a, \mu_a - c\Delta_a)} \lesssim \frac{\dmu_1 + \Delta_a}{c^2 \Delta_a^2}$, and therefore,
Eq.~\eqref{eqn:kl-ucb-refined} implies that
for any $\Delta > 0$, 
\begin{align*}
\Reg(T)
\leq & 
\Delta T + \sum_{a: \Delta_a > \Delta} \frac{\dmu_1 + \Delta_a}{c^2 \Delta_a} \ln f(T) \\
\leq & 
\Delta T + K \frac{\dmu_1 + \Delta}{c^2 \Delta} \ln f(T)
\end{align*}
Choosing $\Delta = \sqrt{ \dmu_1 \frac{K \ln f(T)}{T} }$ yields Eq.~\eqref{eqn:kl-ucb-worstcase}.
\end{proof}

\subsection{KL-UCB++'s refined regret guarantee}

\label{sec:kl-ucb++-refined}
In this section, we show a worst-case regret guarantee of KL-UCB++ of order $\tilde{O}\del{ \sqrt{\dmu_1 K^3 T \ln T} + K^2 \ln T }$ by adapting the original KL-UCB++ analysis (Theorem 2 of~\cite{menard17minimax}). 

First, we derive a refined bound of the number of suboptimal arm pulling, corresponding to Eq. (24) in \cite{menard17minimax}, which we state in the following theorem.
    \begin{theorem}[KL-UCB++: refined upper bound of suboptimal arm pulling]
    \label{thm:KL-UCB++-arm-number}
        For any suboptimal arm $a$, the expected number of its pulling up to time step $T$, namely $\EE\sbr{N_a(T)}$, is bounded by
        \begin{align}
            \EE\sbr{N_a(T)} 
            \leq 
            \frac{\ln(T)}{\kl(\mu_a+\delta, \mu_1-\delta)} + 
            O\del{\frac{(K + \ln\ln(T)) (\dmu_1+\Delta_a)}{\delta^2}},
            \label{eqn:kl-ucb++-arm-pulling}
        \end{align}
    for any $\delta \in [ \frac{88K}{T} + \sqrt{\frac{88 \dmu_1 K}{T}}, \frac{\Delta_a}{3}]$.
    \end{theorem}

    \begin{proof}
    
        First we decompose the expected number of arm pulling w.r.t. suboptimal arm $a$, $\EE\sbr{N_a(T)}$ as
        \[
            \EE\sbr{N_a(T)} \leq 1 + 
            \underbrace{\sum_{t=K}^{T-1} \PP\del{U_{1}(t) \leq \mu_1 - \delta}}_{A} 
            +
            \underbrace{\sum_{t=K}^{T-1} \PP\del{\mu_1 - \delta < U_{a}(t) \textit{ and } I_{t+1}=a }}_{B}
        \]
        Following~\cite{menard17minimax}, we can bound each term in $A$ as: 
        \begin{align*}
        & \PP\del{U_{1}(t) \leq \mu_1 - \delta} \\
        \leq & 
        \underbrace{ \PP\del{\exists 1 \leq n \leq f(\delta), \hat{\mu}_{1,n} \leq  \mu_1, \kl(\hat{\mu}_{1,n}, \mu_1) \geq \frac{g(n)}{n} } }_{A_1}
        + 
        \underbrace{ \PP\del{\exists f(\delta) \leq n \leq T, \hat{\mu}_{1,n} \leq  \mu_1 - \delta } }_{A_2},
        \end{align*}
        here, with foresight, we choose $f(\delta) = \frac{1}{\kl(\mu_1-\delta, \mu_1)} \ln\frac{ \kl(\mu_1-\delta, \mu_1) T }{K}$.
        
        Note that $A_2 \leq \exp(- f(\delta) \kl( \mu_1 - \delta, \mu_1 ))$ by the maximal inequality (Lemma \ref{lemma:maximal-inequality}). 
            \begin{align}
                A_2 \leq \frac{K}{T \kl(\mu_1-\delta, \mu_1)}.
                \label{eqn:kl-ucb++A_1}
            \end{align}
            
        For bounding $A_1$, we rely on the following inequality borrowed from \cite[][page $7$]{menard17minimax}:
        for any $N$ such that $\frac{T}{K N} \geq e^{3/2}$,\footnote{we only replaced their $f(u)$ with $N$. The proof still goes through since the proof has no assumption except $\frac{T}{K f(u)} \geq e^{3/2}$.}
        \begin{align*}
        & \PP\del{\exists 1 \leq n \leq N, \hat{\mu}_{1,n} \leq  \mu_1, \kl(\hat{\mu}_{1,n}, \mu_1) \geq \frac{g(n)}{n} } 
                \\
                \leq &
                    4 e^2 
                    \frac{ \ln( \frac{T}{K N} (1 + \ln^2(\frac{T}{K N})) ) }{ \ln(\frac{T}{K N}) }
                    \cdot 
                    \frac{N \kl(\mu_1-\delta, \mu_1) }{\ln(\frac{T}{K N})} 
                    \cdot 
                    \frac{K}{T \kl(\mu_1-\delta, \mu_1)}.
        \end{align*}

        Therefore, setting $N = f(\delta)$ we have the following inequality when $\frac{T}{K f(\delta)} \geq e^{3/2}$:
            \begin{align*}
                & 
                \PP\del{\exists 1 \leq n \leq f(\delta), \hat{\mu}_{1,n} \leq  \mu_1, \kl(\hat{\mu}_{1,n}, \mu_1) \geq \frac{g(n)}{n} } 
                \\
                \leq &
                    4 e^2 
                    \underbrace{
                    \frac{ \ln( \frac{T}{K f(\delta)} (1 + \ln^2(\frac{T}{K f(\delta)})) ) }{ \ln(\frac{T}{K f(\delta)}) }}_{C}
                    \cdot 
                    \underbrace{\frac{f(\delta) \kl(\mu_1-\delta, \mu_1) }{\ln(\frac{T}{K f(\delta)})} 
                    }_{D} 
                    \cdot 
                    \frac{K}{T \kl(\mu_1-\delta, \mu_1)}.
            \end{align*}

        Also, based on the assumption that $\delta \geq \frac{88K}{T} + \sqrt{\frac{88 \dmu_1 K}{T}}$, we have that $\frac{T \kl(\mu_1-\delta,\mu_1)}{K} \geq e^{3/2}$ and $\frac{T}{K f(\delta)} > 1$
        (we defer the justification at the end of this paragraph).
        Now:
        \begin{itemize}
        \item For $C$, we apply the elementary inequality that $\frac{\ln (x (1 + \ln^2 x) )}{ \ln x} \leq 2$ for $x > 1$ with $x = \frac{T}{K f(\delta)}$; therefore, $C \leq 2$.
        \item For $D =  \frac{ \ln\frac{\kl(\mu_1-\delta)T}{K} }{\ln \del{\frac{\kl(\mu_1-\delta)T}{K} / \ln\frac{\kl(\mu_1-\delta)T}{K}} }$, we apply the elementary inequality that $\frac{\ln (x) }{ \ln( x / \ln x) } \leq 2$ for $x \geq e^{3/2}$ with $x = \frac{T \kl(\mu_1-\delta,\mu_1)}{K}$; therefore, $D \leq 2$. 
        \end{itemize}

        Now we are going to justify the condition that $\delta \geq \frac{88K}{T} + \sqrt{\frac{88 \dmu_1 K}{T}}$ ensures these two elementary inequalities being true. In proving $\frac{T \kl(\mu_1-\delta,\mu_1)}{K} \geq e^{3/2}$, we use the KL lower bound lemma (lemma \ref{lemma:KL-lower-bound}). More specifically,
        \begin{align}
            & \frac{T \kl\del{\mu_1-\delta, \mu_1}}{K}
            \geq
            e^{3/2}
            \\
            \Leftarrow & \;
            \frac{T\delta^2}{4K(\dmu_1+\delta)}
            \geq
            e^{3/2}
                \tag{Lemma \ref{lemma:KL-lower-bound}}
            \\
            \Leftarrow & \;
            \delta^2
            \geq
            \frac{44K(\dmu_1+\delta)}{T}
            \\
            \Leftarrow & \;
            \delta^2
            \geq
            2 \cdot \max\cbr{\frac{44 \dmu_1 K}{T}, \frac{44 K \delta}{T}}
            \\
            \Leftarrow & \;
            \delta
            \geq
            \max\cbr{\frac{88K}{T}, \sqrt{\frac{88 \dmu_1 K}{T}}}
            \\
            \Leftarrow & \; 
            \delta
            \geq
            \frac{88K}{T} + \sqrt{\frac{88 \dmu_1 K}{T}}
        \end{align}
        In summary, from the above derivation,  
        $\delta \geq \frac{88K}{T} + \sqrt{\frac{88 \dmu_1 K}{T}}$ implies that
        $\frac{T \kl\del{\mu_1-\delta,\mu_1}}{K} \geq e^{3/2}$. In this case, furthermore we have $\frac{T}{K f(\delta)} = \frac{T\kl(\mu_1-\delta,\mu_1)/K}{\ln\del{T\kl(\mu_1-\delta,\mu_1)/K}} \geq \frac{2}{3} e^{3/2} > 1$. 
        
        Therefore we bound $A_1$ by
            \begin{align}
                A_1 
                \leq 
                4e^2 \cdot 2 \cdot 2 \cdot \frac{K}{T \kl\del{\mu_1-\delta, \mu_1}}
                \leq
                16e^2 \frac{K}{T \kl\del{\mu_1-\delta, \mu_1}}.
                    \label{eqn:kl-ucb++A_2}
            \end{align}
        Combining Eq~\eqref{eqn:kl-ucb++A_1} and \eqref{eqn:kl-ucb++A_2}, we derive the upper bound for $A$:
        
        \begin{align}
            A 
            &\leq
            \sum_{t=K}^{T-1} 
            \del{ 16e^2 +1 } 
            \frac{K}{ T \kl(\mu_1-\delta, \mu_1) }
            \leq
            \del{ 16e^2 +1 } 
            \frac{K}{ \kl(\mu_1-\delta, \mu_1) }
            \leq 
            O\del{ \frac{K (\dmu_1 + \delta) }{\delta^2} },
                \label{eqn:kl-ucb++-A}
        \end{align}
        where in the last inequality we use Lemma~\ref{lemma:KL-lower-bound}.

        To bound $B$, we reuse the same idea in \cite{menard17minimax} but change the definition of $n(\delta)$ to accommodate our new analysis,
        \[
            n(\delta) = \Biggl\lceil \frac{ \ln\del{\frac{T}{K} \del{1+\ln^2(\frac{T}{K})}}}{\kl\del{\mu_a+\delta, \mu_1-\delta}} \Biggr\rceil
        \]

        applying the same analysis in \cite{menard17minimax} (specifically, from their Eq. (28) to Eq.(29)), we bound $B$ by
        \begin{align}
            B 
            \leq& 
            n(\delta) - 1 + \sum_{n=n(\delta)}^T \PP\del{\kl\del{\hat{\mu}_{a, (n)}, \mu_1-\delta} \leq \kl\del{\mu_a + \delta, \mu_1 - \delta}} 
            \\
            \leq&
            n(\delta) - 1 + \sum_{n=n(\delta)}^T \PP\del{\hat{\mu}_{a,(n)} \geq \mu_a + \delta}
            \\
            \leq&
            n(\delta) - 1 + \sum_{n=1}^T \exp\del{-n \kl(\mu_a + \delta, \mu_a) }
                \tag{Lemma~\ref{lemma:maximal-inequality}}
            \\
            \leq&
            n(\delta) - 1 + \frac{1}{\exp\del{\kl(\mu_a + \delta, \mu_a) } - 1}
                \tag{Geometric sum}
            \\
            \leq &
            n(\delta) - 1 + \frac{1}{\kl(\mu_a + \delta, \mu_a)}
                \tag{$e^x \geq x + 1$ when $x \geq 0$}
            \\
            \leq&
            \frac{ \ln\del{\frac{T}{K} \del{1+\ln^2\del{\frac{T}{K}}}}}{\kl\del{\mu_a+\delta, \mu_1-\delta}} +
            \frac{4\del{\dmu_a + \delta}}{\delta^2}
                \tag{Lemma~\ref{lemma:KL-lower-bound}}
            \\
            =&
            \frac{ \ln\del{T}}{\kl\del{\mu_a+\delta, \mu_1-\delta}}
            +
            \frac{\ln\del{\frac{1}{K} \del{1+\ln^2\del{\frac{T}{K}}}}}{\kl\del{\mu_a+\delta, \mu_1-\delta}}
            +
            \frac{4\del{\dmu_1 + \Delta_a + \delta}}{\delta^2}
                \tag{By the $1$-Lipshitzness of $\mu \mapsto \dmu$}
            \\
            \leq&
            \frac{ \ln\del{T}}{\kl\del{\mu_a+\delta, \mu_1-\delta}}
            +
            O\del{ \frac{\ln\del{\frac{1}{K} \del{1+\ln^2\del{\frac{T}{K}}}}}{\delta^2 / ( \dmu_1 + \Delta_a )}
            +
            \frac{4\del{\dmu_1 + \Delta_a + \delta}}{\delta^2}
            } 
                \tag{Lemma \ref{lemma:KL-lower-bound}} 
            \\
            \leq & 
            \frac{ \ln\del{T}}{\kl\del{\mu_a+\delta, \mu_1-\delta}}
            +
            O\del{ \ln \ln T \cdot \frac{\dmu_1 + \Delta_a}{\delta^2}  }.
                \label{eqn:kl-ucb++-B}
        \end{align}
    
        Combining Eq.\eqref{eqn:kl-ucb++-A} and Eq.\eqref{eqn:kl-ucb++-B}, we get the final inequality Eq.\eqref{eqn:kl-ucb++-arm-pulling}.
    \end{proof}

Based on the above refinement and replace $\delta$ by $c\Delta_a$, we can have the following theorem.
    \begin{theorem}[KL-UCB++: refined guarantee]
    \label{thm:kl-ucb++-refined}
    For any $K$-arm bandit problem with reward distributions supported on $[0,1]$, KL-UCB++
    \begin{equation}
    \Reg(T)
    \leq 
    O\rbr{ \sqrt{ \dmu_1 T K^3 \ln T } + K^2 \ln T}.
    \label{eqn:kl-ucb++-worstcase}
    \end{equation}
\end{theorem}

\begin{proof}
    Define $S = \cbr{a \in [K]: \frac{88K}{T} + \sqrt{\frac{88 \dmu_1 K}{T}} \leq \frac{\Delta_a}{3}}$. For $a \in S$, applying Theorem \ref{thm:KL-UCB++-arm-number} with $\delta = \frac{\Delta_a}{3}$,  
and observe that by Lemma \ref{lemma:KL-lower-bound},
    \begin{align*}
        \frac{1}{\kl(\mu_a + \delta, \mu_1 - \delta)} \lesssim 
        \frac{\dmu_a + \delta}{\delta^2} \lesssim
        \frac{\dmu_1 + \Delta_a}{\Delta_a^2},
    \end{align*}
    we get:
    \begin{align}
        \EE[N_a(T)] 
        \lsim& 
        \frac{(\dmu_1 + \Delta_a)\ln T}{\Delta_a^2} + O\del{ \frac{(K + \ln\ln (T)) (\dmu_1+\Delta_a)}{\Delta_a^2} }
        \\
        \lsim&
        O\del{ \frac{(K + \ln T) (\dmu_1+\Delta_a)}{\Delta_a^2} }
    \end{align}
    Therefore, for any $\Delta > 0$, 
    the regret given a timespan of $T$ is bounded by
    \begin{align*}
        \Reg(T)
        \leq&
            \sum_{a: \Delta_a \leq \Delta} 
            \Delta_a \EE\sbr{N_a(T)} +
            \sum_{a: \Delta_a > \Delta, a \in S} 
            \Delta_a \EE\sbr{N_a(T)} +
            \sum_{a: \Delta_a > \Delta, a \notin S} 
            \Delta_a \EE\sbr{N_a(T)}
        \\
        \leq & 
            T \Delta+ 
            \sum_{a: \Delta_a > \Delta, a \in S}
                O\del{ \frac{(K + \ln T) (\dmu_1+\Delta_a)}{\Delta_a} } +
            \sum_{a: \Delta_a > \Delta, a \notin S}
            O\del{ T \del{ \frac{K}{T} + \sqrt{\frac{\dmu_1 K}{T}} } }
        \\
        \leq & 
            T \Delta+ 
            O\del{ \frac{K (K + \ln T) (\dmu_1+\Delta)}{\Delta} } +
            O\del{ K^2 + \sqrt{\dmu_1 K^{3} T} }
    \end{align*}
    Choosing $\Delta = \sqrt{ \frac{ \dmu_1 K \del{ K + \ln (T)}}{T} }$ yields Eq.~\eqref{eqn:kl-ucb++-worstcase}.
\end{proof}

\subsection{The worst-case regret bound of UCB-V}
\label{sec:ucbv}

In this section, we will show that the problem dependent regret bound presented in UCB-V\cite{audibert09exploration} can also be adaptive to $\dmu_1$ in the bandits with $[0,1]$ bounded reward setting. The starting point is that we will obtain a lemma (Lemma \ref{lemma:ucb-v-main-lemma}) to bound the arm pulling for all suboptimal arms like what we did in our paper.

\begin{lemma}\label{lemma:ucb-v-main-lemma}
    Let $N_{i}(T)$ to be the number of the arm pulling in terms of the arm $i$ until the time step $T$ (inclusively) in the algorithm UCB-V from \cite{audibert09exploration}. Then we can bound $\EE\sbr{N_i(T)}$ by the following inequality
    \begin{align}
        \EE \sbr{ N_{i}(T) } 
        \lsim
        \del{ \frac{\dmu_i^2}{\Delta_i^2} + \frac{1}{\Delta_i} } \log T
    \end{align}
\end{lemma}

\begin{proof}

    Inside the proof of Theorem $3$ in \cite{audibert09exploration}, by setting $c = 1$, for each arm $i$, we obtain the following inequality
    for any $\zeta > 0$:
    \begin{align}
        \EE \sbr{ N_{i}(T) } 
        \leq 
        1 + 
        8 \cE_T \del{ \frac{\dmu_i^2}{\Delta_i^2} + \frac{2}{\Delta_i} } +
        T e^{-\cE_T} \del{ \frac{24\dmu_i}{\Delta_i^2} + \frac{4}{\Delta_i} } +
        \sum_{t=u+1}^T \beta \del{ \cE_t, T },   
    \end{align}

where $u:= \lceil 8 \zeta \del{\frac{\dmu_k^2}{\Delta_k^2} + \frac{2}{\Delta_k}} \log T \rceil$, $\cE_T := \zeta \log T$ and $\beta \del{\cE_t, t} := \inf_{1 < \alpha \leq 3} \del{ \frac{\log t}{\log \alpha} \wedge t} e^{-\frac{\cE_t}{\alpha}}$.
We pick $\zeta = 1.1$. The last term is bounded by
\begin{align}
    \sum_{t=u+1}^T \beta \del{ \cE_t, t }
    \leq&
        \sum_{t=u+1}^T 3 \cdot \inf_{1 < \alpha \leq 3} \del{ \frac{\log t}{\log \alpha} \wedge t } e^{-\frac{\cE_t}{\alpha}}
    \leq
        \sum_{t=u+1}^T 3 \cdot \frac{\log t}{\log \del{1.1}} e^{-\frac{\cE_t}{1.1}}
    \\
    \leq&
        \frac{3}{\log (1.1)} \sum_{t=u+1}^T \frac{\log t}{t^{1.1}}
    \lsim
        \sum_{t=1}^\infty \frac{\log t}{t^{1.1}}
    \lsim
        1
\end{align}

Therefore, we have the following inequality
\begin{align}
    \EE \sbr{ N_{i}(T) } 
    \lsim
    \del{ \frac{\dmu_i^2}{\Delta_i^2} + \frac{1}{\Delta_i} } \log T
        \label{eqn:UCB-V-main-theorem}
\end{align}

\end{proof}

By using the lemma \ref{lemma:ucb-v-main-lemma} we just obtained, we can obtain the following theorem about worst-case regret bound of UCB-V.
\begin{theorem}
    The regret of the algorithm UCB-V\cite{audibert09exploration} is bounded by:
    \begin{align}
        \Reg(T) \lsim \sqrt{\dmu_1 K T \ln(T)} +  K\ln(T)
    \end{align}
    \label{thm:ucb-v-refined}
\end{theorem}

\begin{proof}

\begin{align*}
  \Reg(T) 
  & = 
    \sum_{i : \Delta_i \leq \Delta } \Delta_i \EE[N_{i}(T)] +
    \sum_{i : \Delta_i > \Delta } \Delta_i \EE[N_{i}(T)]
    \\
    &\leq
    T\Delta +
    \sum_{i : \Delta_i > \Delta } \Delta_i \EE[N_{i}(T)]
    \\
  &\lsim 
    T\Delta +
    \sum_{i: \Delta_i > \Delta} \del{\fr{\dmu_i}{\Delta_i} + 1} \log(T) \tag{By Eq.~\eqref{eqn:UCB-V-main-theorem} }
\\&\le T\Delta + \sum_{i: \Delta_i \in[\Delta, 1/4]} \del{\fr{\dmu_i}{\Delta_i} + 1} \log(T)
     + \sum_{i: \Delta_i > 1/4} \del{\fr{\dmu_i}{\Delta_i} + 1} \log(T)
\\&\lsim T\Delta + \sum_{i: \Delta_i \in[\Delta, 1/4]} \fr{\dmu_i}{\Delta_i} \log(T)
     + K\log(T)~.
\end{align*}

To bound the second term above, we consider two cases.

\textbf{Case 1}: $\mu_1 < \fr34$~.\\

In this case, one can show that $\dmu_i \lsim \dmu_1$.
Thus,
\begin{align*}
  \sum_{i: \Delta_i \in[\Delta, 1/4]} \fr{\dmu_i}{\Delta_i} \ln(T)
  \lsim  K\fr{\dmu_1}{\Delta} \ln(T) ~.
\end{align*}

\textbf{Case 2}: $\mu_1 \ge \fr34$~.\\

We observe that if $i$ satisfies $\Delta_i\in[\Delta,1/4]$, then $\dmu_i = \mu_i(1-\mu_i) \le 1-\mu_i = 1-\mu_i+\mu_1 - \mu_1 = 1 - \mu_1 + \Delta_i \lsim \dmu_1 + \Delta_i$.
Thus,
\begin{align*}
  \sum_{i: \Delta_i \in[\Delta, 1/4]} \fr{\dmu_i}{\Delta_i} \ln(T)
  \lsim \sum_{i: \Delta_i \in[\Delta, 1/4]} \del{\fr{\dmu_1}{\Delta_i}+1} \ln(T)
  \le K\fr{\dmu_1}{\Delta} \ln(T) +  K\ln(T)~.
\end{align*}

Altogether, we have
\begin{align*}
  \Reg(T) \lsim T \Delta + K\fr{\dmu_1}{\Delta} \ln(T) +  K\ln(T)~.
\end{align*}
Let us choose $\Delta = \sqrt{\fr{K\dmu_1}{T}} \wedge \fr14$.
If $T > K\dmu_1$, then we obtain the desired bound.
If $T \le K\dmu_1$, we get $\Delta = 1/4$, so
\begin{align*}
  \Reg(T) 
  \lsim n + K \dmu_1 \ln(T) +  K\ln(T)
  \le   K\dmu_1 + K \dmu_1 \ln(T) +  K\ln(T)~,
\end{align*}
which is less than the desired bound.
This concludes the proof.
\end{proof}

\section{Improved minimax analysis of the sub-Gaussian MS}
\label{sec:sgms}

We sketch how to change the proof of the sub-Gaussian MS regret bound in~\citet{bian2022maillard} so it can achieve the minimax ratio of $\sqrt{\ln(K)}$.

It suffices to show that $\forall a: \mu_a < \mu_1, \EE[N_{T,a}] \lsim \fr{\sig^2}{\varepsilon^2}\ln(\fr{T\varepsilon^2}{\sig^2} \vee e^2)$.
To bound $\EE[N_{T,a}]$, recall that there are three terms to bound: $(F1)$, $(F2)$, and $(F3)$.
Recall the symbols in~\citet{bian2022maillard}:
\begin{itemize}
  \item $\sigma^2$: the sub-Gaussian parameter.
  \item $u := \llcl \fr{2\sig^2(1+c)^2 \ln(T\Delta_a^2/(2\sig^2) \vee e^2)}{\Delta_a^2}  \rrcl$ for some $c > 0$.
  \item $\varepsilon>0$: an analysis parameter that will be chosen later to be $\Delta_a$ up to a constant factor.
\end{itemize} 

The reason why one does not obtain the minimax ratio of $\sqrt{\ln(K)}$ is that the bound obtained in~\citet{bian2022maillard} for $(F3)$ is $O(\fr{\sig^2}{\varepsilon^2}\ln(\fr{\sig^2}{\varepsilon^2} \vee e^2))$ 
rather than $O(\fr{\sig^2}{\varepsilon^2}\ln(\fr{ T \varepsilon^2 }{\sig^2} \vee e^2))$.
To achieve the latter bound for $(F3)$, first we choose the splitting threshold $\fr{\sigma^2}{\varepsilon^2}$ which takes the same role as $H$ for KL-MS in the $[0,1]$-bounded reward case and $F3$ will be separated into $F3_1$ and $F3_2$. $F3_1$ is the case where $F3$ is with the extra condition that $N_{t-1,1} \leq \fr{\sigma^2}{\varepsilon^2}$ for $1 \leq t \leq T$ and $F3_2$ the case where $F3$ is with the extra condition that $N_{t-1,1} > \fr{\sigma^2}{\varepsilon^2}$ for $1 \leq t \leq T$.
It is easy to bound $F3_2$ using a similar argument as our Claim~\ref{claim:F3-2-upper-bound} that $F3_2 \lsim \fr{\sig^2}{\varepsilon^2}$.

For $F3_1$, we define the following event
\begin{align*}
  \cE := \cbr{\forall k \in [1, \lfl \fr{\sig^2}{\varepsilon^2} \rfl], \hmu_{(k),1} \ge \mu_1 - \sqrt{\fr{4\sig^2\ln(T/k)}{k} }  }
\end{align*}
where $\hmu_{(k),1}$ is the empirical mean of arm 1 (the true best arm) after $k$ arm pulls.

We have
\begin{align*}
  F3_1 
  &= \EE \sbr{ \sum_{t=K+1}^{T} \one\cbr{I_{t}=a, N_{t-1,a} > u, \hmu_{t-1,\max} < \mu_1 - \varepsilon, N_{t-1,1} \leq \fr{\sigma^2}{\varepsilon^2} } }
\\&= \EE \sbr{ \sum_{t=K+1}^{T} \one\cbr{I_{t}=a, N_{t-1,a} > u, \hmu_{t-1,\max} < \mu_1 - \varepsilon, N_{t-1,1} \leq \fr{\sigma^2}{\varepsilon^2}, \cE} }
   + \EE \sbr{ \sum_{t=K+1}^{T-1} \one\cbr{\cE^c} }
\\&\le \EE \sbr{ \sum_{t=K+1}^{T} \one\cbr{I_{t}=a, N_{t-1,a} > u, \hmu_{t-1,\max} < \mu_1 - \varepsilon, \cE} } + T \cd \PP(\cE^c)
\end{align*}

Note that one can show that $T \cd \PP(\cE^c) \lsim \fr{\sig^2}{\varepsilon^2}$ using a similar argument to Lemma \ref{lemma:seq-estimator-deviation-bernoulli}.
One can also see that the first term above corresponds to the first term of Eq.~\eqref{eqn:f3-fix-arm-1-E} in KL-MS, and one can use a similar technique therein to bound the first term above by $\fr{\sig^2}{\varepsilon^2}\ln(\fr{T\varepsilon^2}{\sig^2} \vee e^2)$ up to a constant factor.

Adding the bounds of $F3_1$ and $F3_2$ together, we conclude that $F3 \lsim \fr{\sig^2}{\varepsilon^2}\ln(\fr{T\varepsilon^2}{\sig^2} \vee e^2)$. 

\section{Additional Experiments}
\label{sec:addl-experiments}

    \subsection{Regret comparison}
    We compare KL-MS with the Bernoulli Thompson Sampling and MS~\citep{bian2022maillard}.
    Bernoulli Thompson Sampling chooses beta distribution as the prior (Beta(0.5, 0.5)) and the posterior. 
    The reward environment is borrowed from  \cite{kaufmann12thompson}, where there are two reward environments. Both are two-arm bandit, one has the mean reward $[0,20, 0,25]$ and the other has the mean reward $[0.80, 0.90]$. From Figure~\ref{fig:regret-1} and Figure~\ref{fig:regret-2} we find that the performance of KL-MS is better than MS by a margin, although worse than Bernoulli Thompson Sampling. Nevertheless, we will see in the next section that Bernoulli Thompson Sampling tends to generate somewhat unreliable logged data for offline evaluation.

    \begin{figure}[H]
      \centering
      \text{Regret comparison with $2000$ times simulation}
      
      \begin{minipage}{0.49\textwidth}
      \includegraphics[width=\linewidth]{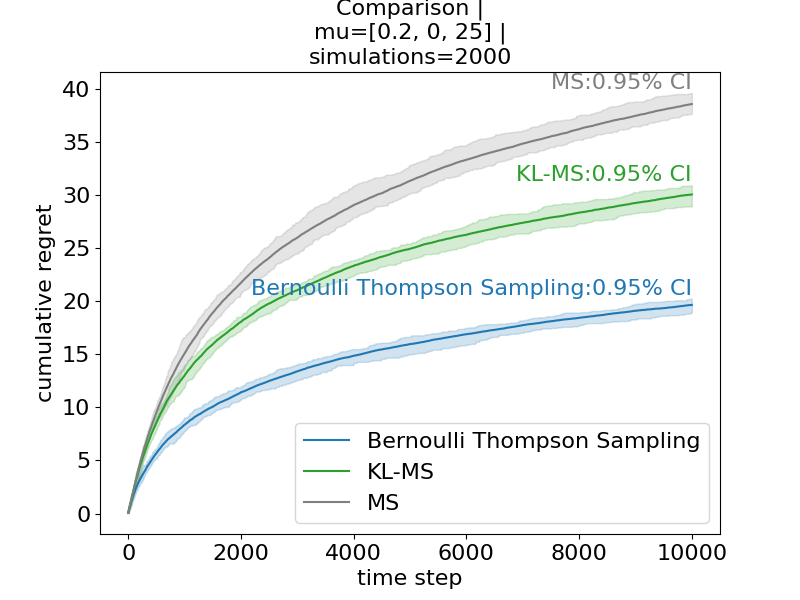}
        \caption{$\mu = [0.20, 0.25], T = 10,000$}
        \label{fig:regret-1}
      \end{minipage}
      \hfill
      \begin{minipage}{0.49 \textwidth}
      \includegraphics[width=\linewidth]{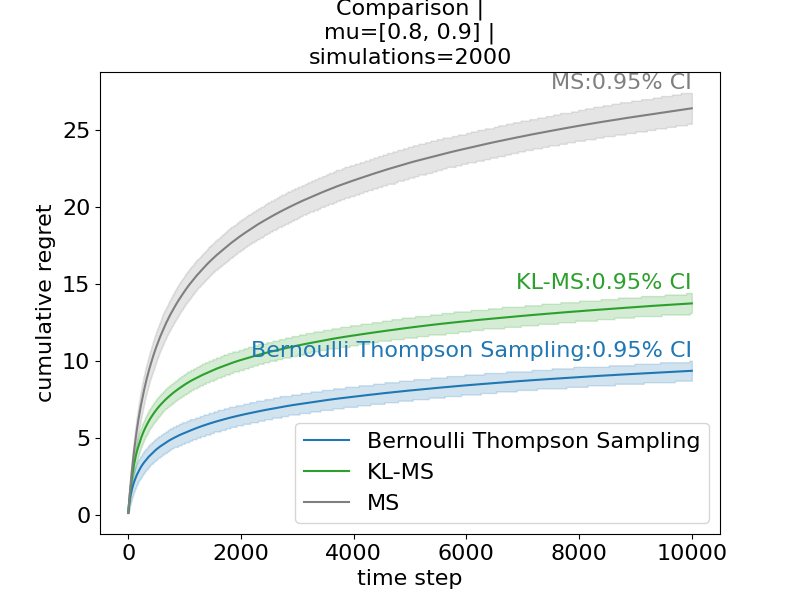}
        \caption{$\mu = [0.80, 0.90], T = 10,000$}
        \label{fig:regret-2}
      \end{minipage}
    \end{figure}
    
    \subsection{Offline evaluation}
    This section presents our simulation results on offline evaluation using logged data. We use the logged data generated by our algorithm, KL-MS, and standard Thompson Sampling, to estimate the expected reward of the policy that takes an action uniformly at random in $[K]$, which is equal to $\bar{\mu} = \frac1K \sum_{i=1}^K \mu_i$. 
    The logged data are of the form 
    $(I_t, p_{t, I_t}, r_t)_{t=1}^T$, where $I_t$ is the action taken, $p_{t,I_t}$ is the action probability (which can be exact or approximate), $r_t$ is the received reward, all at time step $t$. 
    We consider the IPW estimator~\cite{horvitz52generalization} that estimates $\mu$, defined as 
    \[
    \hat{\mu} = \frac1T \sum_{t=1}^T \frac{1/K}{p_{t, I_t}} r_t. 
    \]
    We set $T$, the time horizon of the interaction log, to be $1,000$ or $10,000$.
    For Thompson sampling, we use Monte Carlo (MC) to estimate the action probabilities; 
    we vary the number of MC samples $M$ in $\cbr{ 10^3, 10^4, 10^5}$.
    Note that MC estimation of action probabilities induces a high time cost: in our simulations, for $T=10^3$, KL-MS uses $0.43$s to generate its logged data; in contrast, BernoulliTS with $M=10^3$ uses $15.21$s to generate its logged data. This suggest that setting $M=10^4$ or $10^5$ may be impractical in applications. 

    Figures \ref{fig:exp-1-T-1K-M-1K} to \ref{fig:exp-2-T-10K-M-100K} shows the histogram of the IPW estimates of the average reward 
    induced by logged data generated by KL-MS and Bernoulli-TS with MC estimation of action probabilities, 
    based on $N=2000$ independent trials in the same reward environment used in the previous experiment. Repeatedly, We have two $2$-armed bandit problems, whose mean rewards are $[0.20, 0.25]$ and  $[0.8, 0.9]$ respectively. 
    Tables~\ref{table:mse-exp-1-T-1K} to~\ref{table:bias-exp-2-T-10K} report the MSE and the bias estimate of the respective estimator. 
    It can be seen from the figures and tables that: (1)
    the logged data induced by KL-MS consistently give more accurate estimates of $\mu$, compared to that of BernoulliTS with MC estimation of action probabilities; 
    (2)
    the offline evaluation performance of the logged data induced by BernoulliTS is sensitive to the number of MC samples $M$; while the performance of setting $M=10^4$ or $10^5$ is on par with KL-MS, the estimation error of the more-practical $M=10^3$ setting is evidently higher.
    (3) When time step $T$ is increasing, the error between the IPW estimator induced by BernoulliTS logged data and the true performance become larger while KL-MS remains the same level of error which is smaller than the BernoulliTS.

    \begin{figure}[H]
      \centering
      \textit{$\mu = [0.20, 0.25], T = 1,000$}
      
      \begin{minipage}{0.30\textwidth}
      \includegraphics[width=\linewidth]{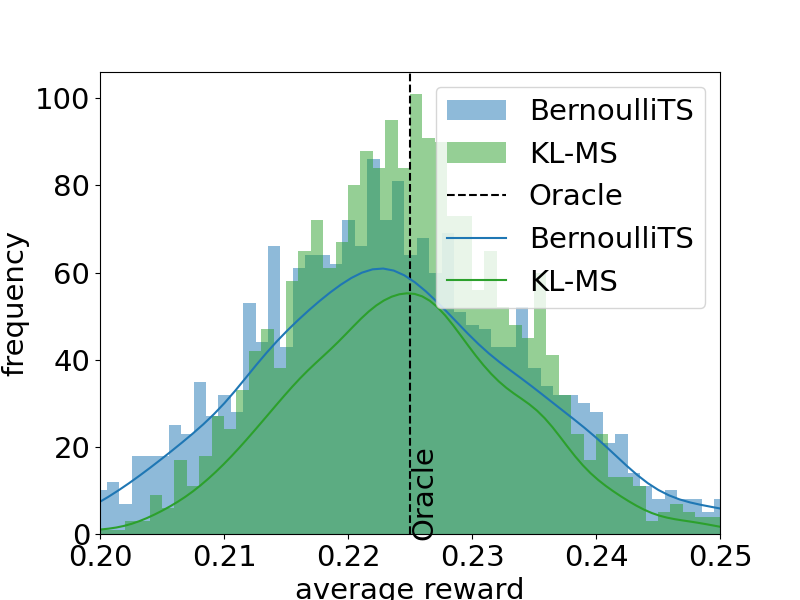}
        \caption{$M = 10^3$}
        \label{fig:exp-1-T-1K-M-1K}
      \end{minipage}
      \hfill
      \begin{minipage}{0.30 \textwidth}
      \includegraphics[width=\linewidth]{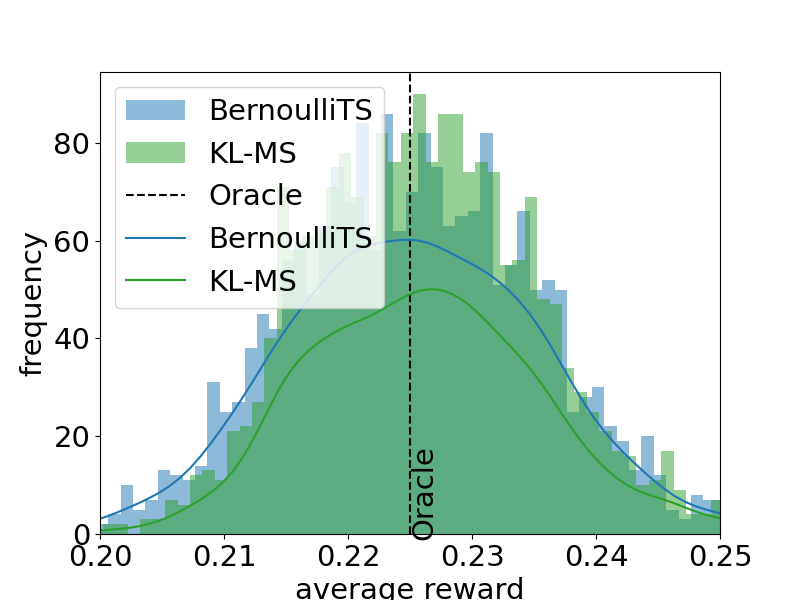}
        \caption{$M = 10^4$}
        \label{fig:exp-1-T-1K-M-10K}
      \end{minipage}
      \hfill
      \begin{minipage}{0.30 \textwidth}
      \includegraphics[width=\linewidth]{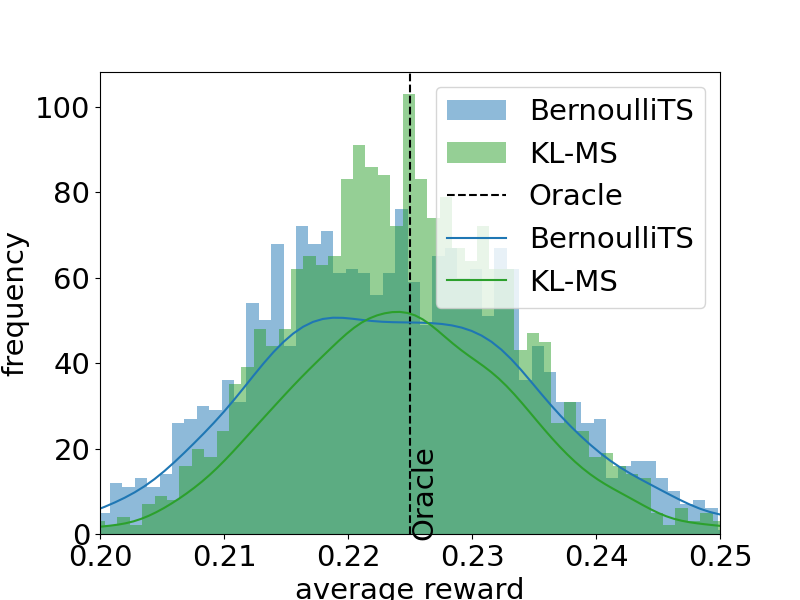}
        \caption{$M = 10^5$}
        \label{fig:exp-1-T-1K-M-100K}
      \end{minipage}
      \label{fig:exp-1-T-1K}
    \end{figure}

    \begin{figure}[H]
      \centering
      \textit{$\mu = [0.80, 0.90], T = 1,000$}

      \begin{minipage}{0.30\textwidth}
      \includegraphics[width=\linewidth]{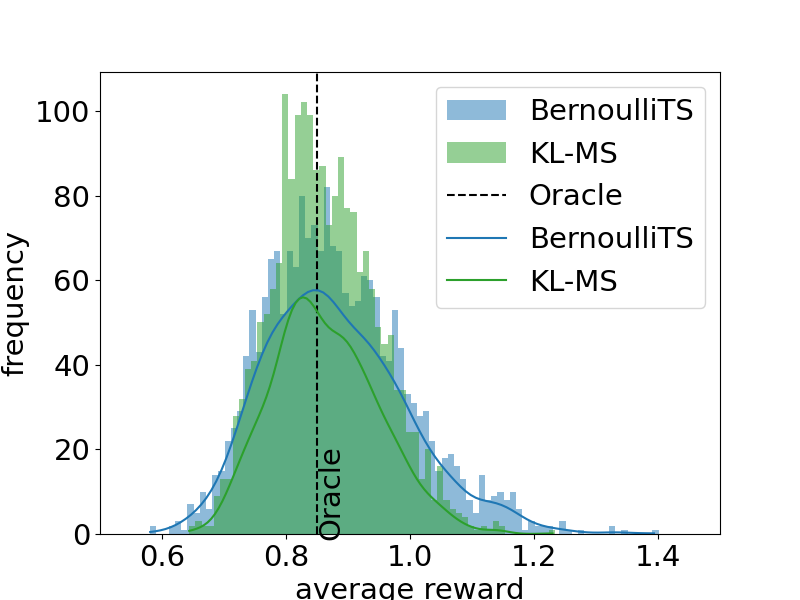}
        \caption{$M = 10^3$}
        \label{fig:exp-2-T-1K-M-1K}
      \end{minipage}
      \hfill
      \begin{minipage}{0.30 \textwidth}
      \includegraphics[width=\linewidth]{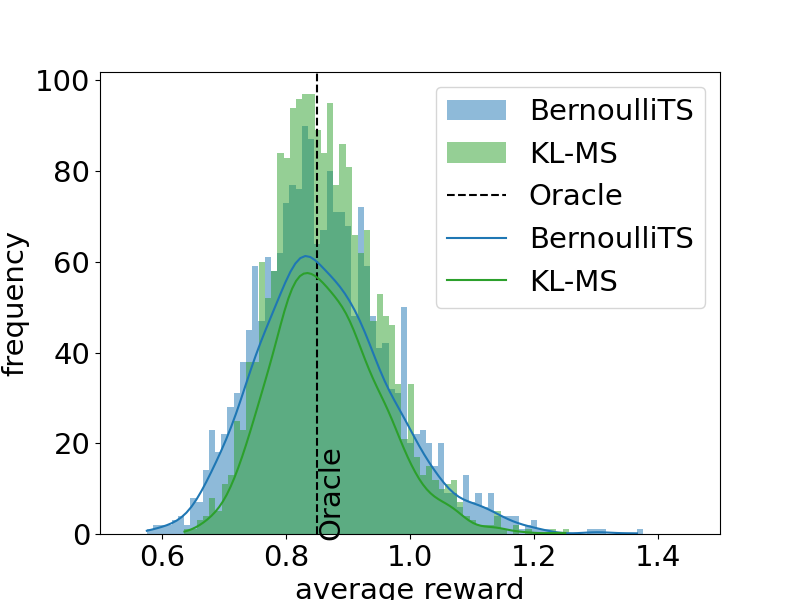}
        \caption{$M = 10^4$}
        \label{fig:exp-2-T-1K-M-10K}
      \end{minipage}
      \hfill
      \begin{minipage}{0.30 \textwidth}
      \includegraphics[width=\linewidth]{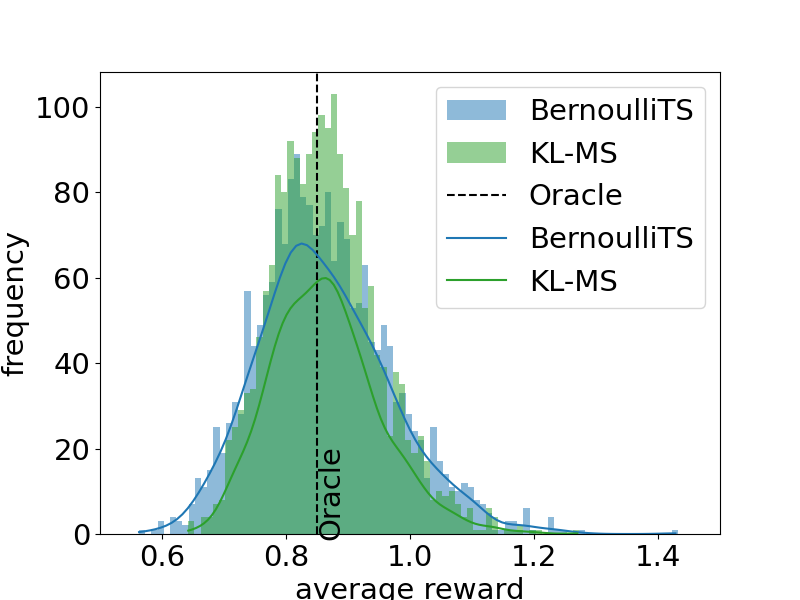}
        \caption{$M = 10^5$}
        \label{fig:exp-2-T-1K-M-100K}
      \end{minipage}
      \label{fig:exp-2-T-1K}
    \end{figure}
    
    \begin{figure}[H]
      \centering
      \text{$\mu = [0.20, 0.25], T = 10,000$}
      
      \begin{minipage}{0.30\textwidth}
      \includegraphics[width=\linewidth]{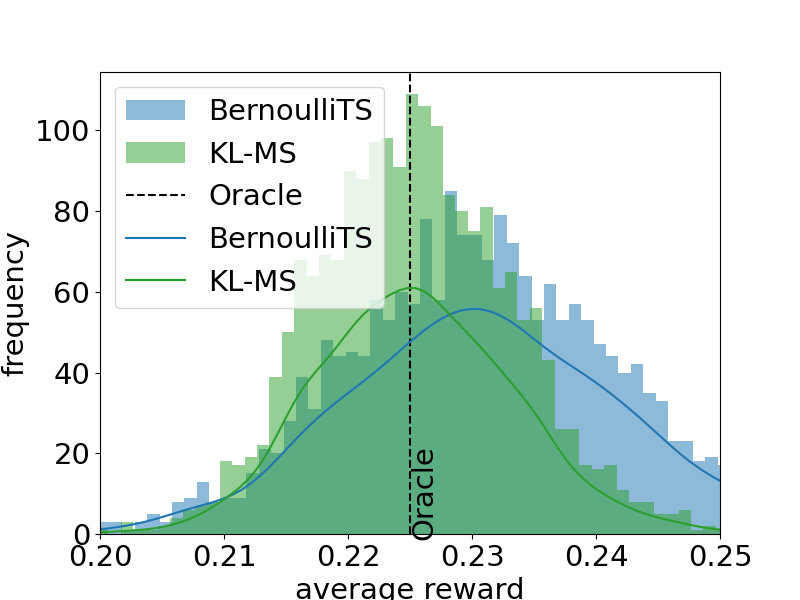}
        \caption{$M = 10^3$}
        \label{fig:exp-1-T-10K-M-1K}
      \end{minipage}
      \hfill
      \begin{minipage}{0.30 \textwidth}
      \includegraphics[width=\linewidth]{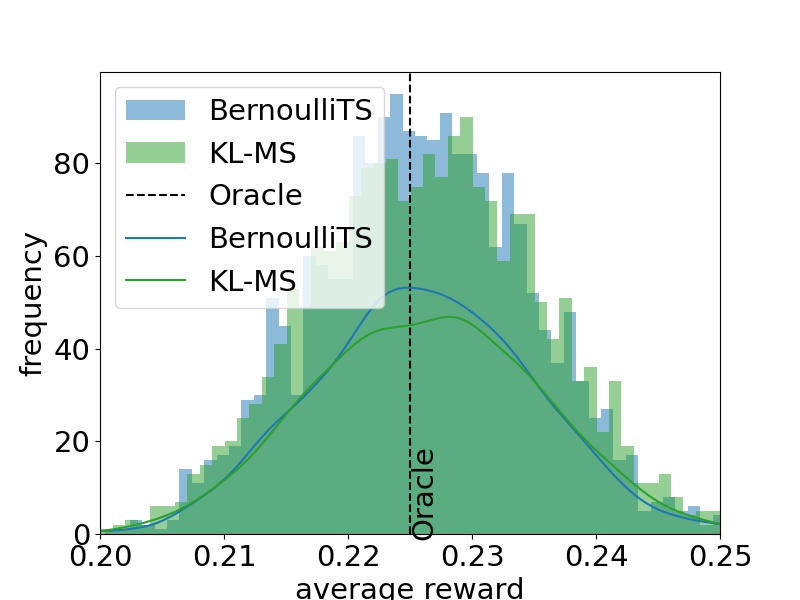}
        \caption{$M = 10^4$}
        \label{fig:exp-1-T-10K-M-10K}
      \end{minipage}
      \hfill
      \begin{minipage}{0.30 \textwidth}
      \includegraphics[width=\linewidth]{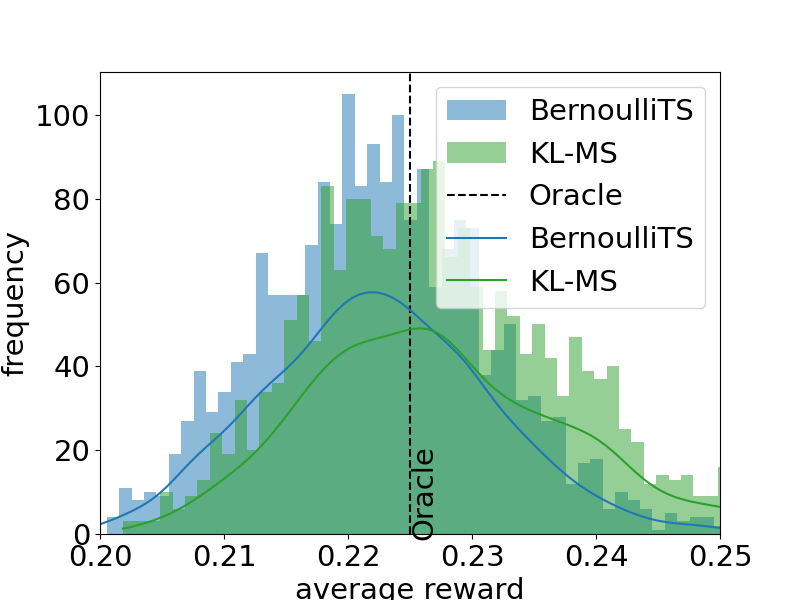}
        \caption{$M = 10^5$}
        \label{fig:exp-1-T-10K-M-100K}
      \end{minipage}
      
      \label{fig:exp-1-T-10K}
    \end{figure}

    \begin{figure}[H]
      \centering
      \textit{$\mu = [0.80, 0.90], T = 10,000$}
      
      \begin{minipage}{0.30\textwidth}
      \includegraphics[width=\linewidth]{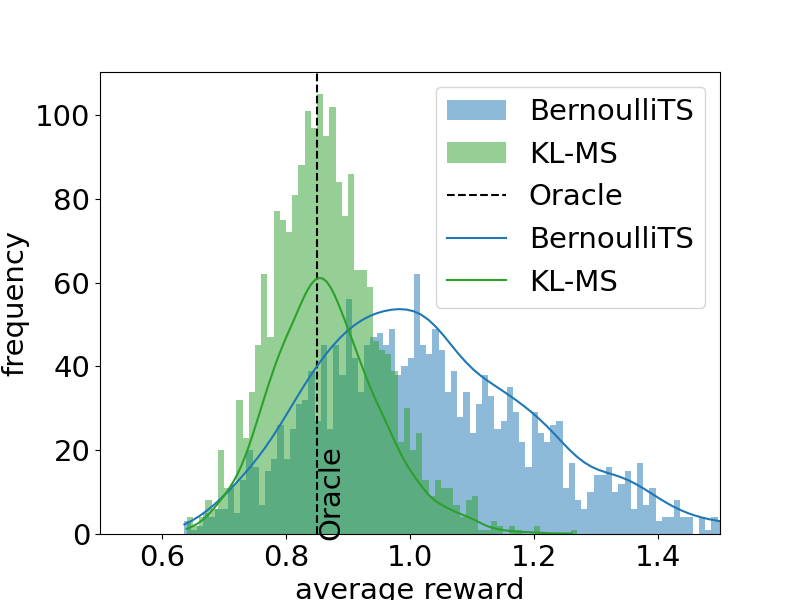}
        \caption{$M = 10^3$}
        \label{fig:exp-2-T-10K-M-1K}
      \end{minipage}
      \hfill
      \begin{minipage}{0.30 \textwidth}
      \includegraphics[width=\linewidth]{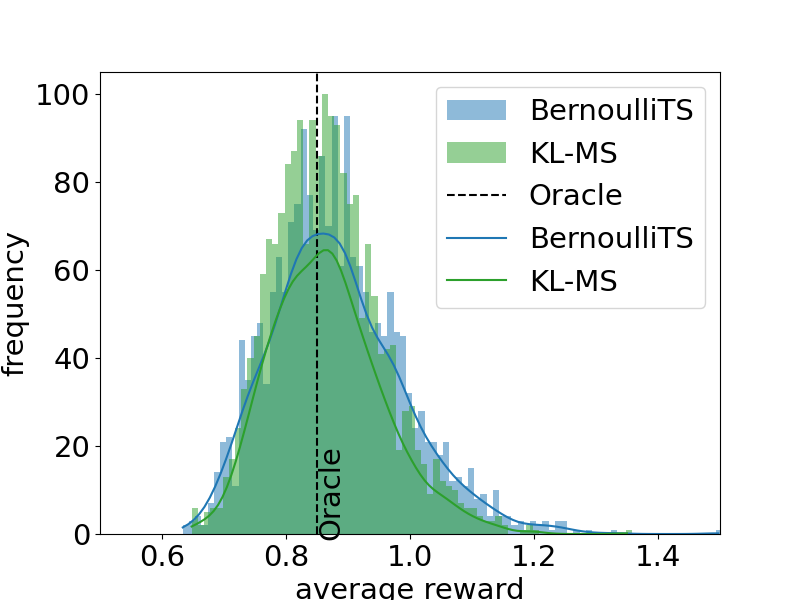}
        \label{fig:exp-2-T-10K-M-10K}
        \caption{$M = 10^4$}
      \end{minipage}
      \hfill
      \begin{minipage}{0.30 \textwidth}
      \includegraphics[width=\linewidth]{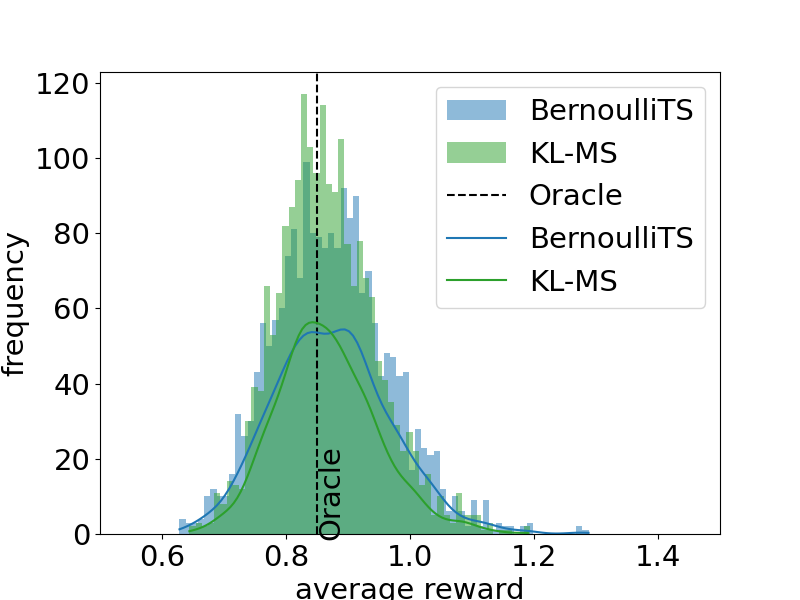}
        \caption{$M = 10^5$}
        \label{fig:exp-2-T-10K-M-100K}
      \end{minipage}
      \label{fig:exp-2-T-10K}
    \end{figure}

    \begin{table}[H]
        \begin{minipage}{0.40\textwidth}
        \hfill
            \caption{MSEs for $\mu = [0.20, 0.25]$, $T=1,000$}
            \label{table:mse-exp-1-T-1K}
        \hfill
            \begin{tabular}{llll}
            \hline
                        & \multicolumn{3}{c}{$M$}  \\ \cline{2-4} 
                        & $10^3$ & $10^4$     & $10^5$    \\ \hline
            BernoulliTS &  0.00014  & 0.00012 & 0.00014 \\
            KL-MS       &  0.00001  & 0.00001 & 0.00001 \\ \hline
            \end{tabular}
        \end{minipage}
    \hfill
        \begin{minipage}{0.40\textwidth}
        \hfill
            \caption{Bias for $\mu = [0.20, 0.25]$, $T=1,000$}
            \label{table:bias-exp-1-T-1K}
        \hfill
            \begin{tabular}{llll}
            \hline
                        & \multicolumn{3}{c}{$M$}  \\ \cline{2-4} 
                        & $10^3$ & $10^4$     & $10^5$   \\ \hline
            BernoulliTS &  -0.00059  & 0.00106 & -0.00068 \\
            KL-MS       &  -0.00096  & 0.00118 & 0.00011 \\ \hline
            \end{tabular}
        \hfill
        \end{minipage}
    \hfil
        \begin{minipage}{0.40\textwidth}
        \hfill
            \caption{MSEs for $\mu = [0.80, 0.90]$, $T=1,000$}
            \label{table:mse-exp-2-T-1K}
        \hfill
            \begin{tabular}{llll}
            \hline
                        & \multicolumn{3}{c}{$M$}  \\ \cline{2-4} 
                        & $10^3$ & $10^4$     & $10^5$    \\ \hline
            BernoulliTS &  0.01464  & 0.01143 & 0.01228 \\
            KL-MS       &  0.00733  & 0.00782 & 0.00749 \\ \hline
            \end{tabular}
        \end{minipage}
    \hfill
        \begin{minipage}{0.40\textwidth}
        \hfill
            \caption{Bias for $\mu = [0.80, 0.90]$, $T=1,000$}
            \label{table:bias-exp-2-T-1K}
        \hfill
            \begin{tabular}{llll}
            \hline
                        & \multicolumn{3}{c}{$M$}  \\ \cline{2-4} 
                        & $10^3$ & $10^4$     & $10^5$   \\ \hline
            BernoulliTS &  0.02911  & 0.01741 & 0.01636 \\
            KL-MS       &  0.01304  & 0.01412 & 0.01355 \\ \hline
            \end{tabular}
        \hfill
        \end{minipage}
    \hfill
        \begin{minipage}{0.40\textwidth}
        \hfill
            \caption{MSEs for $\mu = [0.20, 0.25]$, $T=10,000$}
            \label{table:mse-exp-1-T-10K}
        \hfill
            \begin{tabular}{llll}
            \hline
                        & \multicolumn{3}{c}{$M$}  \\ \cline{2-4} 
                        & $10^3$ & $10^4$     & $10^5$    \\ \hline
            BernoulliTS &  0.00017  & 0.00010 & 0.00009 \\
            KL-MS       &  0.00007  & 0.00006 & 0.00011 \\ \hline
            \end{tabular}
        \end{minipage}
    \hfill
        \begin{minipage}{0.40\textwidth}
        \hfill
            \caption{Bias for $\mu = [0.20, 0.25]$, $T=10,000$}
            \label{table:bias-exp-1-T-10K}
        \hfill
            \begin{tabular}{llll}
            \hline
                        & \multicolumn{3}{c}{$M$}  \\ \cline{2-4} 
                        & $10^3$ & $10^4$     & $10^5$    \\ \hline
            BernoulliTS &  0.00637  & 0.00142 & -0.00240 \\
            KL-MS       &  0.00052  & 0.00066 & 0.00220 \\ \hline
            \end{tabular}
        \end{minipage}
    \hfill
        \begin{minipage}{0.40\textwidth}
        \hfill
            \caption{MSEs for $\mu = [0.80, 0.90]$, $T=10,000$}
            \label{table:mse-exp-2-T-10K}
        \hfill
            \begin{tabular}{llll}
            \hline
                        & \multicolumn{3}{c}{$M$}  \\ \cline{2-4} 
                        & $10^3$ & $10^4$     & $10^5$    \\ \hline
            BernoulliTS &  0.06842  & 0.01276 & 0.01220 \\
            KL-MS       &  0.00898  & 0.00804 & 0.00929 \\ \hline
            \end{tabular}
        \end{minipage}
    \hfill
        \begin{minipage}{0.40\textwidth}
        \hfill
            \caption{Bias for $\mu = [0.80, 0.90]$, $T=10,000$}
            \label{table:bias-exp-2-T-10K}
        \hfill
            \begin{tabular}{llll}
            \hline
                        & \multicolumn{3}{c}{$M$}  \\ \cline{2-4} 
                        & $10^3$ & $10^4$     & $10^5$   \\ \hline
            BernoulliTS &  0.17947  & 0.03401 & 0.04313 \\
            KL-MS       &  0.02046  & 0.01731 & 0.01123 \\ \hline
            \end{tabular}
        \hfill
        \end{minipage}
    \hfill
    
    \end{table}

\end{document}